\documentclass{article}
\usepackage{ijcai22}

\usepackage{times}
\usepackage{soul}
\usepackage{url}
\usepackage[hidelinks]{hyperref}
\usepackage[utf8]{inputenc}
\usepackage[small]{caption}
\usepackage{graphicx}
\usepackage{amsmath}
\usepackage{amsthm}
\usepackage{booktabs}
\usepackage{tabularx} 
\newcommand{\tableheadline}[1]{\multicolumn{1}{c}{#1}}
\usepackage[binary-units=true,group-separator={,},group-four-digits=true]{siunitx}
 \usepackage{enumitem}
 \usepackage{microtype}
\usepackage{graphicx}
\usepackage{subfigure}
\usepackage{booktabs} 
\usepackage{tcolorbox}
 \usepackage{enumitem}

\usepackage{makeidx}  
\usepackage{amsmath,amssymb,amsthm,bm}
\usepackage{gensymb}
\usepackage{wrapfig}
\usepackage{algorithmic}
\usepackage{adjustbox,rotating,makecell,multirow}
\newlength\myindent
\providecommand{\U}[1]{\protect\rule{.1in}{.1in}}
\newtheorem{theorem}{Theorem}

\newtheorem{corollary}[theorem]{Corollary}

\newtheorem{definition}[theorem]{Definition}

\newtheorem{lemma}[theorem]{Lemma}

\newtheorem{proposition}[theorem]{Proposition}
\newtheorem{remark}[theorem]{Remark}

\DeclareMathOperator{\class}{class}
\DeclareMathOperator{\sgn}{sgn}
\DeclareMathOperator{\clamp}{clamp}

\newcommand{\argmin}[1]{\underset{#1}{\operatorname{argmin}}}
\newcommand{\supp}[1]{\operatorname{supp}_{#1}}

\title{Structural Extensions of Basis Pursuit: \\Guarantees on Adversarial Robustness}

\author{
D\'{a}vid Szeghy${}^{\ddagger,\text{\textsection},}$\footnote{Contact Author}\and
Mahmoud Aslan${}^\dagger$\and
\'{A}ron F\'{o}thi${}^\dagger$\and
Bal\'{a}zs M\'{e}sz\'{a}ros${}^\dagger$\and\\
Zolt\'{a}n \'{A}d\'{a}m Milacski\And
Andr\'{a}s L\H{o}rincz${}^\dagger$
\affiliations
${}^\dagger$Department of Artificial Intelligence, Faculty of Informatics, ELTE E\"{o}tv\"{o}s Lor\'{a}nd University
${}^\ddagger$Department of Geometry, Faculty of Natural Sciences, ELTE E\"{o}tv\"{o}s Lor\'{a}nd University\\
${}^\text{\textsection}$AImotive Inc.
\emails
david.szeghy@ttk.elte.hu, \{i7e7qi, fa2, k1wtbf\}@inf.elte.hu, srph25@gmail.com, lorincz@inf.elte.hu
}

\begin{document}

\maketitle

\begin{abstract}
While deep neural networks are sensitive to adversarial noise, sparse coding using the Basis Pursuit (BP) method is robust against such attacks, including its multi-layer extensions. We prove that the stability theorem of BP holds upon the following generalizations: (i) the regularization procedure can be separated into disjoint groups with \emph{different} weights, (ii) \emph{neurons} or \emph{full layers} may form groups, and (iii) the regularizer takes various generalized forms of the $\ell_1$ norm. This result provides the proof for the architectural generalizations of \citeauthor{cazenavette2020architectural} \shortcite{cazenavette2020architectural} including (iv) an approximation of the complete architecture as a shallow sparse coding network. Due to this approximation, we settled to experimenting with shallow networks and studied their robustness against the Iterative Fast Gradient Sign Method on a synthetic dataset and MNIST. We introduce classification based on the $\ell_2$ norms of the groups and show numerically that it can be accurate and offers considerable speedups. In this family, linear transformer shows the best performance. Based on the theoretical results and the numerical simulations, we highlight numerical matters that may improve performance further.

\end{abstract}

\section{Introduction}

Considerable effort has been devoted to overcoming the vulnerability of deep neural networks against `\emph{white box}' adversarial attacks. These attacks have access to the network structure and the loss function. They work by modifying the input towards the sign of the gradient of the loss function \cite{goodfellow2014explaining} that can spoil classification at very low levels of perturbations. Furthermore, this white box attack gives rise to successful transferable attacking samples to other networks of similar kinds \cite{liu2016delving}, called `\emph{black box attack}'. This underlines the need for network structures exhibiting robustness against white box adversarial attacks.

\begin{figure}[!ht]
\centering     
\includegraphics[width=85mm]{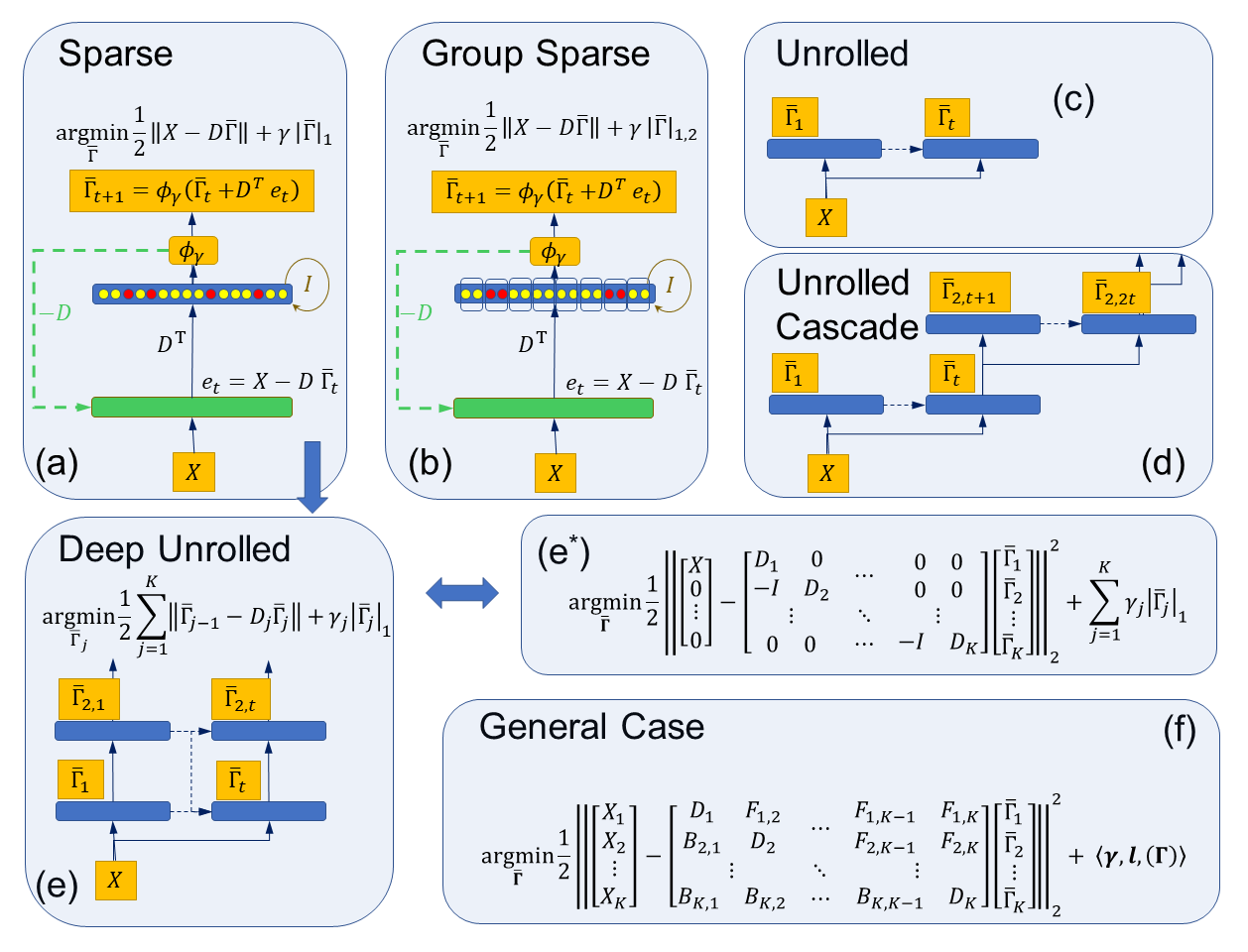}
\caption{Steps of Basis Pursuit (BP) generalizations. 
Equations with argmin: the minimization tasks. \textbf{(a):} Recurrent BP with sparse representation. 
Blue (light green) rectangle: representation (input) layer. 
Blue (dashed light green) arrows: channels that deliver quantities in the actual (in the previous) time step. 
Red (light yellow) circles: active (non-active) units of the sparse representation. 
$X$: input. $\bar{\Gamma}_t$ and $e_t$: representation and error at the $t^{th}$ iteration. $\bar{\Gamma}_{j,t}$: same at the $j^{th}$ layer of the deep unrolled network.
Matrices $D$, $D_j$: dictionaries, $I$: identity matrix, $\phi_{\gamma}$ softmax with $\gamma$ bias. 
\textbf{(b):} group sparse case: $\ell_1$ norm is replaced with the $\ell_{1,2}$ norm. \mbox{\textbf{(c):} Unrolled} feedforward network with finite number of iterations. \textbf{(d):} Cascaded unrolled deep network. 
\textbf{(e):} Non-cascaded modification of the unrolled deep \emph{sparse} cascade. \textbf{(e$^*$):} The minimization task of \textbf{(e)}. \textbf{(f):} The general case still having warranties against adversarial attacks. Within layer groups are not shown. More details: text and supplementary material.}
\label{fig:experiments}
\end{figure}

Sparse methods exploiting $\ell_1$ norm regularization and the Basis Pursuit (BP) algorithm (Figs.~\ref{fig:experiments}(a) and \ref{fig:experiments}(c)) exhibit robustness against such attacks, including their multilayer Layered Basis Pursuit (LBP)  extensions \cite{romano2020adversarial} (Fig.~\ref{fig:experiments}(d)). \citeauthor{cazenavette2020architectural} \shortcite{cazenavette2020architectural} found a solution to the LBP's drawback that layered basis pursuit accumulates errors: they put forth an architectural generalization of LBP to modify the cascade of layered basis pursuit steps of the deep neural network in such a way that the entire network becomes an approximation to a single structured sparse coding problem that they call deep pursuit (Figs.~\ref{fig:experiments}(e) and \ref{fig:experiments}(e$^*$)). Note that their generalization goes beyond the structure depicted in Fig.~\ref{fig:experiments}(e). This architectural generalization points to the relevance of a single sparse layer BP that we study here.

A long-standing problem is that sparse coding is slow. An early effort utilized an associative correlation matrix  \cite{gregor2010learning}. Recent efforts, put forth the first approximation of BP combined with specific loss terms during training (see \cite{murdock2021reframing} and the references therein). Although the approach is attractive, theoretical stability warranties are missing. 

We propose group sparse coding as an additional means for the resolution. Sparse coding that exploits $\ell_1$ norm regularization to optimize the hidden representation can be generalized to group sparse coding that uses the $\ell_{1,2}$ norm or the elastic $\ell_{\beta,1,2}$ norm instead.

We present theoretical results
on the stability of a family of group sparse coding that alike to its sparse variant can robustly recover the underlying representations under adversarial attacks. Yet, group sparse coding offers fast and efficient feedforward estimations of the groups either by traditional networks or by transformers that the classification step can follow. Previous work \cite{lorincz2016columnar} suggested the feedforward estimation of the groups to be followed by the pseudoinverse estimation of the group activities for learning and finding a group sparse code but without targeting classification or adversarial considerations. 

Our feedforward method estimates the $\ell_2$ norms of the active groups followed by the classification step, achieving further computational gains by eliminating the pseudoinverse computations. We consider how to combine the fast estimation with the robust BP computations based on our theoretical and numerical results. However, the speed considerations and test will be presented in a separate paper, now will focus on the robustness results.

Our contributions are as follows:
\begin{itemize}
    \item we extend the theory of adversarial robustness of Basis Pursuit to a family of networks, including groups, layers, and skip connections between the layers both to deeper and to more superficial layers,
    \item we introduce group norm based classification and its group pooled variant,
    \item suggest and study gap regularization, 
    \item execute numerical computations and test feedforward shallow, deep, transformer networks trained on sparse and group sparse layers with a synthetic and the MNIST dataset,study the performance of these fast algorithms, and
    \item we point to bottlenecks in the training procedures.
\end{itemize}

We present our theoretical results in Sect.~\ref{s:theo}. It is followed by the experimental studies (Sect.~\ref{s:exps}). We examine the properties of the group sparse structures outside of the scope of the theory to foster further works. Section~\ref{s:disc} contains the discussions of our results. We conclude in Section~\ref{s:conc}.  Details of the theoretical derivations are in Appendix~\ref{AppA}.

\section{Theory}\label{s:theo}

We start with the background of the theory including the notations. It is followed by our theoretical results.

\subsection{Background and Notation}\label{SecBandN}

We denote the Sparse Coding (SC) problem by $\bm{X}=\bm{D}\bm{\Gamma}$, where given the \textit{signal} $\bm{X}\in\mathbb{R}^{N}$ and the unit-normed \textit{dictionary} $\bm{D}\in\mathbb{R}^{N\times M}$, the task is to recover the \textit{sparse vector representation} $\bm{\Gamma}\in\mathbb{R}^{M}$.
\begin{equation}
\min{\left \Vert \bm{\Gamma}\right\Vert _{0} \,\,\, \mathrm{subject\,\, to} \,\, \bm{X}=\bm{D}\bm{\Gamma,} }
\tag{$P_0$}\label{ell0}%
\end{equation}
where $\left \Vert . \right\Vert _{0}$ denotes the $\ell_0$ norm. For an excellent book on the topic, see \cite{elad2010sparse} and the references therein. 

One may try to approximate the solution of Eq.~\eqref{ell0} via the unconstrained version of the Basis Pursuit (BP, or LASSO) method \cite{tibshirani1996regression,chen2001atomic,donoho2003optimally}:
\begin{equation}
\argmin{\bar{\bm{\Gamma}}}\,{L\left(\bar{\bm{\Gamma}}\right)}\overset{def}{=}\argmin{\bar{\bm{\Gamma}}}{\frac{1}{2}\left\Vert \bm{X}-\bm{D}\bar{\bm{\Gamma}}
\right\Vert _{2}^{2}+\gamma\cdot \left\Vert \bar{\bm{\Gamma}}\right\Vert _{1}},
\tag{BP}\label{BP}%
\end{equation}
where $\gamma>0$.

Given $\bm{X}=\bm{D}\bm{\Gamma}$,
we may assume that $\bm{\Gamma}$ can be further decomposed in a way similar to $\bm{X}$:
\begin{align}
\bm{X} &=\bm{D}_{1}\bm{\Gamma}_{1},\label{ML}\\
\bm{\Gamma}_{1} &=\bm{D}_{2}\bm{\Gamma}_{2},\nonumber\\
&\vdots\nonumber\\
\bm{\Gamma}_{K-1} &=\bm{D}_{K}\bm{\Gamma}_{K}.\nonumber
\end{align}
The layered problem then tries to recover $\bm{\Gamma}_1,\dots,\bm{\Gamma}_K$.

\begin{definition}
The Layered Basis Pursuit (LBP) \cite{papyan2017convolutional} first solves the Sparse Coding problem $\bm{X}=\bm{D}_{1}\bm{\Gamma}_{1}$ via Eq.~\eqref{BP} with parameter
$\gamma_{1}$, obtaining $\hat{\bm{\Gamma}}_{1}$.
Next, it solves another Sparse Coding problem $\hat{\bm{\Gamma}}_{1}=\bm{D}_{2}\bm{\Gamma}_{2}$ again by Eq.~\eqref{BP} with parameter
$\gamma_{2}$, denoting the result by $\hat{\bm{\Gamma}}_{2}$, and so on.
The final vector $\hat{\bm{\Gamma}}_{K}$ is the solution of LBP.
The vector $\bm{\gamma}^{LBP}$ contains the weights $\gamma_i$ in Eq.~\eqref{BP} for each layer $i$.
\end{definition}

It was shown in \cite{papyan2016working} and  \cite{papyan2017working} that LBP suffers from error accumulation. To alleviate this obstacle, \cite{cazenavette2020architectural} rewrote LBP into a single joint Eq.~\eqref{BP}-like minimization scheme (i.e., all layers are processed simultaneously) that can be equipped with skip connections. However, the solutions of the two programs differ, and the stability has not been proven for the latter that we do here. (See Appendix~\ref{AppA}, Figs.~\ref{fig:experiments}(e*), and (f)).

We want to extend these methods to allow different norms on different parts of $\Gamma$ with different $\gamma$ weights (as in the layered case) and prove a stability result for this more general case. This will also allow to relieve the condition on the dictionary $\bm{D}$ that its columns have unit length in the $\ell_2$ norm.

Let us introduce a slightly modified version of
the notation used by \cite{papyan2016working} and  \cite{papyan2017working}.
Let $\Lambda$ be a subset of
$\left\{  1,\dots,M\right\}  $ which is called a \textit{subdomain}, and the components, or \emph{atoms}
corresponding to $\Lambda$ form the \textit{subdictionary }$\bm{D}_{\Lambda}$.
Let
$\bm{d}_{\omega}$, $\omega\in\left\{  1,\dots,M\right\}  $ denote the atom
corresponding to the index $\omega.$

If $\Lambda_{i}\left(  \bm{D}\right)  \overset{def}{=}\left\{  \omega
~|~\left\langle \bm{d}_{\omega},\bm{d}_{i}\right\rangle \neq0\right\}  $ and $\left\vert \Lambda_{i}\left(  \bm{D}\right)\right\vert$ is its cardinality, then the restriction $\bm{\Gamma}_{\Lambda_{i}\left(  \bm{D}\right)}\in
\mathbb{R}^{\left\vert \Lambda_{i}\left(  \bm{D}\right)\right\vert}$ of $\bm{\Gamma}\in\mathbb{R}^{M}$ to the indices in $\Lambda
_{i}\left(  \bm{D}\right)  $ is given by, 
\[
\left(\bm{\Gamma}_{\Lambda_{i}\left(  \bm{D}\right)}\right)_{\theta}
\overset{def}{=}
\begin{cases}
\bm{\Gamma}_{\theta}, &\text{ if }\theta\in\Lambda_{i}\left(  \bm{D}\right),\\
$0$, &\text{otherwise}.
\end{cases}
\]
Now let
\[
\left\Vert \bm{\Gamma}\right\Vert _{0,st,\bm{D}}\overset{def}{=}\max_{i}\left\Vert
\bm{\Gamma}_{\Lambda_{i}\left(  \bm{D}\right)  }\right\Vert _{0}%
\]
be the \textit{stripe norm with respect to} $\bm{D}$, a generalization of the definition in \cite{papyan2017working}.

If $\bm{D}$ is fixed, then we will use the shorter form $\left\Vert \bm{\Gamma}\right\Vert
_{0,st}=\left\Vert \bm{\Gamma}\right\Vert _{0,st,\bm{D}}$.
Further, let
$\mu\left(  \bm{D}\right)  =\max_{i\neq j}\left\langle \bm{d}_{i},\bm{d}_{j}\right\rangle $
be the \textit{mutual coherence} of the dictionary (since $\bm{D}$ is unit-normed the division by $\left\Vert \bm{d}_i \right\Vert_{2}\cdot\left\Vert \bm{d}_j \right\Vert_{2}$ is dropped).

We want to use 4 different norms the $\ell_{1},\,\ell_{2}$ and the elastic $\ell_{\beta,1,2}$  norm defined as $\left\Vert \bm{Z}\right\Vert
_{\beta,1,2}\overset{def}{=}\beta\cdot\left\Vert \bm{Z}\right\Vert _{1}+\left(
1-\beta\right)  \left\Vert \bm{Z}\right\Vert _{2}$, i.e., it is the convex combination of the $\ell_{1}$ and $\ell_{2}$ norms, and finally, the $\ell_{1,2}$ group-norm, sometimes referred to as the Group LASSO \cite{yuan2006model,bach2011optimization}. To define this we need a group partition of the index set.

If the index set $\left\{  1,\dots,M\right\}$ is partitioned into groups $\mathcal{G}_{i},~i\in\left\{  1,\dots,k\right\}  $ (i.e., $\bigcup_{i=1}^k{\mathcal{G}_i}=\left\{  1,\dots,M\right\}$ and $\mathcal{G}_i \cap \mathcal{G}_j=\emptyset$ for $i\neq j$), then the $\ell_{1,2}$ norm ( see, e.g., \cite{bach2011optimization} and the references therein) is 
\[
\left\Vert \bm{Z}\right\Vert _{1,2}\overset{def}{=}\sum_{i=1}^{k}\left\Vert
\bm{Z}_{\mathcal{G}_{i}}\right\Vert _{2},
\]
where $\bm{Z}_{\mathcal{G}_{i}}=\sum
_{j\in\mathcal{G}_{i}}z_{j}\cdot\mathbf{e}_{j}$ with the standard basis
vectors $\mathbf{e}_{j}\in\mathbb{R}^{M}$.

To extend the regularizer of Eq.~\eqref{BP}, if $\mathcal{G}_{i},~i\in\left\{  1,\dots,k\right\}  $ is a partition of the index set $\left\{  1,\dots,M\right\}$ then let 
\[
\bm{l}:\mathbb{R}^{M}\to\mathbb{R}^{k},\,\bm{l}\left(\bm{\Gamma}\right)\overset{def}{=}\left(l_{\alpha_{1}}\left(\bm{\Gamma_{\mathcal{G}_{1}}} \right),\dots, l_{\alpha_{k}}\left(\bm{\Gamma_{\mathcal{G}_{k}}} \right) \right),
\]
where $l_{\alpha_{i}}$ is one of the $\ell_1,\,\ell_{2},\,\ell_{\beta,1,2}$ norm. For different groups the parameter $\beta$ can be different as well. So this is a vector which elements are norms evaluated on different parts of $\bm{\Gamma}$ corresponding to the different groups and for each group, we can individually decide which norm to use. Let $\bm{\gamma}\overset{def}{=}\left( \gamma_{1},\dots,\gamma_{k}\right)$ be a weight vector for the different groups (more precisely for the norms of the different groups), where $\gamma_{i}>0,\,\forall i$. We want use the regulariser
\[
\left< \bm{\gamma},\bm{l}\left( \bm{\Gamma} \right) \right> = \sum_{i=1}^{k}
\gamma_{i}\bm{l}_{\alpha_{i}} \left( \bm{\Gamma_{\mathcal{G}_{i}}}\right).
\]

Note that if for some groups we use the $\ell_{2}$ norm with the same weight $\gamma$, then we think of this as using the $\ell_{1,2}$ group norm for this group of groups with the weight $\gamma$ being a special case.

Now if we fix a partition $\mathcal{G}_{i}$ and a regularizer $\bm{l}$ (i.e. norms for the groups), then let $\bm{\chi}_{\bm{\Gamma},\mathcal{G}}\in\mathbb{R}^{M}$ be the \textit{$2$-norm group characteristic vector} of $\bm{\Gamma}$, i.e.,
\[
\left(\bm{\chi}_{\bm{\Gamma},\mathcal{G}}\right)_j
\overset{def}{=}
\begin{cases}
1, &\text{if }j\in\supp{}{\bm{\Gamma}},\\
&\text{or }j\in\mathcal{G}_i, \mathcal{G}_{i}\cap\supp{}{\bm{\Gamma}}\neq\emptyset\text{ and }l_{\alpha_{i}}=\ell_{2}, \\
$0$, &\text{otherwise},
\end{cases}
\]
where $\supp{}{\bm{\Gamma}}
\overset{def}{=}\left\{  \omega~|~\bm{\Gamma}_{\omega}\neq0\right\}  $ is the
support of $\bm{\Gamma}$.

For $\bm{Z}\in
\mathbb{R}^{N},$ we define
\[
\left(\bm{Z}_{\supp{}{\bm{d}_{i}}}\right)_\theta
\overset{def}{=}
\begin{cases}
z_\theta, & \text{ if }\theta\in\supp{}{\bm{d}_i},\\
0, & \text{ otherwise}.
\end{cases}
\]
We call
\[
\left\Vert \bm{Z}\right\Vert _{L,\bm{D}}\overset{def}{=}\max_{i}\left\Vert
\bm{Z}_{\supp{}{\bm{d}_{i}}}\right\Vert _{2}%
\]
the \textit{local amplitude of }$\bm{Z}$\textit{ with respect to the dictionary }$\bm{D}.$

For a fixed $D$, we use the shorthand $\left\Vert
\bm{Z}\right\Vert _{L}=\left\Vert \bm{Z}\right\Vert _{L,\bm{D}}$.

Both the stripe norm defined previously, and the local amplitude seem difficult to calculate. However, as in \cite{papyan2017working} if $\bm{D}$ corresponds to a CNN architecture, then both become quite natural and the calculation is easy. Moreover, it is easier to keep mutual coherence of the dictionary low.

\subsection{Theoretical Results}\label{ss:T_results}
The proofs of the results can be found in Appendix~\ref{AppA}.

%
%

Here, we will investigate the stability of Eq.~\eqref{BP} and two closely related algorithms.
To unify the several different cases, we introduce the following definition.
\begin{definition}
First, fix a partition $\mathcal{G}_{i},~i\in\left\{  1,\dots,k\right\}$, norms for this partition $\bm{l}\left(\bm{\Gamma}\right)$ and the weights $\bm{\gamma}$ for the norms. The unconstrained Group Basis Pursuit (GBP) is the solution of the problem:
\begin{equation}
\argmin{\bar{\bm{\Gamma}}}\,{L\left(  \bar{\bm{\Gamma}}\right)} \overset{def}{=}\argmin{\bar{\bm{\Gamma}}}{\frac{1}{2}\left\Vert \bm{X}-\bm{D}\bar{\bm{\Gamma}}
\right\Vert _{2}^{2}+\left< \bm{\gamma},  \bm{l}\left( \bar{\bm{\Gamma}}\right) \right>},
\tag{GBP}\label{GBP}%
\end{equation}

\end{definition}

\begin{theorem}
\label{T2}Let $\bm{X}=\bm{D}\bm{\Gamma}$ be a clean signal and $\bm{Y}=\bm{X}+\bm{E}$ be its perturbed variant.
Let $\bm{\Gamma}_{GBP}$ be the minimizer of Eq.~\ref{GBP} where $\bm{\gamma}$ is the weight vector. If among the norms of $\bm{l}$ we used the elastic norm, let $\left\{\beta_{1},\dots,\beta_{r}\right\}$ be the set of the parameters used in the elastic norms and $\lambda\overset{def}{=}\min\left\{ 1,\beta_{1},\dots,\beta_{r}\right\}$. Moreover, let $\gamma_{\max}\overset{def}{=}\max\left\{\gamma_{1},\dots,\gamma_{k}\right\}$ and $\gamma_{\min}\overset{def}{=}\min\left\{\gamma_{1},\dots,\gamma_{k}\right\}$ for the weight vector $\bm{\gamma}$ and $\theta\overset{def}{=}\frac{\lambda\gamma_{\min}}{\gamma_{\max}}$.
Assume that
\begin{enumerate}
\item[a)]
$\left\Vert \bm{\chi}_{\bm{\Gamma},\mathcal{G}} \right\Vert _{0,st}\leq c\frac{\theta}{1+\theta} \left(1+\frac{1}{\mu\left(\bm{D}\right)}\right)$,

\item[b)] $\frac{1}{\lambda\left(  1-c\right)  }\left\Vert \bm{E}\right\Vert
_{L}\leq\gamma_{\min}$,
\end{enumerate}
where $0<c<1$.
If $\bm{D}_{\supp{}{\bm{\chi}_{\bm{\Gamma},\mathcal{G}}}}$ has full column rank, then
\begin{enumerate}
\item[1)] $\supp{}{\bm{\Gamma}_{GBP}}\subseteq\supp{}{\bm{\chi}_{\bm{\Gamma},\mathcal{G}}}$,
\item[2)] the minimizer of Eq.~\ref{GBP} is unique.
\end{enumerate}
If we set $\gamma_{\min}=\frac{1}{\lambda\left(  1-c\right)  }\left\Vert \bm{E}\right\Vert
_{L}$, then
\begin{enumerate}
\item[3)] $\left\Vert \bm{\Gamma}_{GBP}-\bm{\Gamma}\right\Vert _{\infty}%
<\frac{1+\theta}{\left(  1+\mu\left(  \bm{D}\right)  \right)  \theta\left(
1-c\right)  }\left\Vert \bm{E}\right\Vert _{L}$,
\item[4)] $\left\{i~\Big|~\left\vert \bm{\Gamma}_{i}\right\vert
>\frac{1+\theta}{\left(  1+\mu\left(  \bm{D}\right)  \right)  \theta\left(
1-c\right)  }\left\Vert \bm{E}\right\Vert _{L}\right\}\subseteq\supp{}{\bm{\Gamma}_{GBP}}$,
\end{enumerate}
where
$\frac{1+\theta}{\left(  1+\mu\left(  \bm{D}\right)  \right)  \theta\left(
1-c\right)  }\left\Vert \bm{E}\right\Vert _{L}\leq\frac{1+\theta}{\theta\left(
1-c\right)  }\left\Vert \bm{E}\right\Vert _{L}$ yields a weaker bound in 3) and 4)
without the mutual coherence.
\end{theorem}

Roughly speaking, if the perturbation is not too large, the support of the noisy representation stays within its clean equivalent, and the indices that are above the threshold level in 4) are recovered.
Moreover, we can compare our result to the original Eq.~\ref{BP}, Theorem 6 in \cite{papyan2016working}, as in the pure $\ell_1$ norm case
$\lambda=1$ and if we set $c=\frac{2}{3}$, we get the same bound $\left\Vert \bm{\Gamma}\right\Vert
_{0,st}<\frac{1}{3}\left(  1+\frac{1}{\mu\left(  \bm{D}\right)  }\right)  $, but we have
$3\left\Vert \bm{E}\right\Vert _{L}\leq\gamma$ instead of the original $4\left\Vert \bm{E}\right\Vert
_{L}$ in b).
Similarly, our weaker bound in 3) and 4) is $6\left\Vert
\bm{E}\right\Vert _{L}$ instead of their $7.5\left\Vert \bm{E}\right\Vert _{L}$.

Interestingly, this single sparse layer theorem for Eq.~\ref{GBP} extends to multiple layers, where on each layer we can add group partitioning, can choose norms and weights. The precise convergence theorem can be found in Appendix~\ref{AppA}. It is a generalized version of Theorem 12 in \cite{papyan2017convolutional}, but that suffers from error accumulation \cite{romano2020adversarial}.

As mentioned earlier, we can rewrite a layered GBP into a single sparse layer GBP. The solution will differ a bit, but the error accumulation is not present, see Appendix~\ref{AppA} for the details. However, the new dictionary describing all the layers won't have unit normalization being a problem in the `classical' case but not in ours. This is because if the dictionary $\bm{D}$ is not unit-normed, but the columns belonging to a group $\mathcal{G}_{i}$ (where we choose the $\ell_{2}$ or the $\ell_{\beta,1,2}$ norm) have the same $\ell_2$ norm, then we can push the "normalization weights" of the columns of $\bm{D}$ to the weight $\gamma_{i}$ in $\bm{\gamma}$ through the solutions of the \eqref{GBP}. The problem and the solution change, but the solution will be equivalent to the original problem, see Appendix~\ref{AppA} for further details. This allows us to extend our result for more general sparse coding problems, see Fig.~\ref{fig:experiments}f and the Appendix~\ref{AppA}.

Now, if we stack a linear classifier onto the top of GBP (or onto a layered GBP) as it was done in \cite{romano2020adversarial}, we have several classification stability results, see in Appendix~\ref{AppA}. 

Also if we solve Eq.\eqref{GBP} with positive coding, i.e. restrict the problem to non-negative $\bar{\bm{\Gamma}}$ vectors, and the solution $\bm{\Gamma}_{+GBP}$ is group-full (i.e. $\supp{}{\bm{\Gamma}_{+GBP}}=\supp{}{\bm{\chi}_{\bm{\Gamma}_{+GBP},\bm{\mathcal{G}}}}$) then a weak stability theorem holds for $\bm{\Gamma}_{+GBP}$, more in Appendix~\ref{AppA}.

%

\section{Experimental Studies}\label{s:exps}
We turn to the description of our numerical studies. We want to explore the limitations of Group Basis Pursuit (GBP) methods and our experiments are outside of
the scope of the present theory. We first review the methods. It is followed by the description of the 
datasets and the experimental results. Throughout these studies we used fully connected (dense) networks implemented in PyTorch \cite{paszke2019pytorch}.

\subsection{Methods}\label{ss:meth}
\subsubsection{Architectures}

To evaluate the empirical robustness of our GBP with $\ell_2$ norm regularization, we compared two variants of it with Basis Pursuit (BP) and $3$ Feedforward networks.

For our BP experiments, we used a single BP layer to compute the hidden representation $\bm{\Gamma}_{BP}$, then stacked a classifier $\bm{w}$ on top.

Next, for GBP, we considered two scenarios.
First, we applied GBP on its own
to compute a full $\bm{\Gamma}_{GBP}$ code.
Second, we introduced \emph{Pooled GBP (PGBP)}: after computing $\bm{\Gamma}_{GBP}$ with GBP, we compressed it with a per group $\ell_2$ norm calculation into $\bm{\Gamma}_{PGBP}$, and used this smaller code as input to a smaller classifier $\bm{w}_{PGBP}$.

Finally, we employed $3$ feedforward neural networks trained for approximating $\bm{\Gamma}_{PGBP}$: a Linear Transformer \cite{katharopoulos2020transformers}, a single dense layer, and a dense deep network having parameter count similar to the Transformer.
Network structure details can be  found in Sect.~\ref{ss:nets}.
For the nonnegative norm values, we used Rectified Linear Unit (ReLU) activation at the top of these networks.
To migitate vanishing gradients, we also added a batch normalization layer in some cases.
After obtaining the approximate pooled $\hat{\bm{\Gamma}}_{PGBP}$, we applied the smaller $\bm{w}_{PGBP}$ as the classifier.


\subsubsection{Loss functions}\label{sss:loss}
Whenever training was necessary for classification (see Sect.~\ref{ss:mnist}), we pretrained our methods to minimize the unsupervised reconstruction loss $\|\bm{X}-\bm{D}\bm{\Gamma}_{(G)BP}\|_2^2$.

During classification and attack phase, we used a total loss function $J\left(\bm{D},\bm{w},\bm{b},\bm{X},\class{(\bm{X})}\right)$ consisting of a common classification loss term with an optional regularization term.

For the classification loss, we made our choice depending on the number of classes.
For the $2$ class (binary classification) case we used hinge loss, whereas for the multiclass case we applied the categorical cross-entropy loss.

The regularization loss was specifically employed to test whether it can further improve the adversarial robustness.
For this, we introduced a \emph{gap regularization} term to encourage a better separation between active and inactive groups.
We intended to increase the smallest difference of preactivations between the smallest active and the largest inactive group norm within a mini-batch of $\bm{\Gamma}_{(G)BP}$ samples:
\begin{equation}
    \begin{aligned}
    J_{\mathrm{gap}} &=  -\min_{i=1,\dots,N}\Bigl(\min_{j\colon\phi_\gamma\bigl(||\bm{\Gamma}_{(G)BP,G_j}^{(i)}||_2\bigr) \neq 0} ||\bm{\Gamma}_{(G)BP,G_j}^{(i)}||_2 \\
    &\quad\quad\quad\quad\quad- \max_{j\colon\phi_\gamma\bigl(||\bm{\Gamma}_{(G)BP,G_j}^{(i)}||_2\bigr)=0} ||\bm{\Gamma}_{(G)BP,G_j}^{(i)}||_2\Bigr),
    \end{aligned}
\end{equation}\label{eq:gap}
where $i$ is the sample index, $||\bm{\Gamma}_{(G)BP,G_j}^{(i)}||_2$ is the $\ell_2$ norm of group $j$ within $\bm{\Gamma}_{(G)BP}^{(i)}$ (i.e., an element of $\bm{\Gamma}_{PGBP}^{(i)}$) and $\phi_\gamma$ is an appropriate proximal operator.
For the BP case we applied group size $1$.

For the training of the  feedfoward networks, we applied mean squared error against $\bm{\Gamma}_{PGBP}$.


\subsubsection{Adversarial Attacks}
To generate the perturbed input $\bm{Y}=\bm{X}+\bm{E}$, we used the Iterative Fast Gradient Sign Method (IFGSM) \cite{kurakin2016adversarial}.
Specifically, this starts from $\bm{X}$ and takes $T$ bounded steps wrt. $\ell_\infty$ and $\ell_2$ norms according to the sign of gradient of the total loss $J$ to get $\bm{Y}=\bm{Y}_T$:
\begin{equation}
\begin{aligned}
\bm{Y}_0&=\bm{X},\\
\bm{G}_{t-1}&=\nabla_{\bm{Y}_{t-1}} J\left(\bm{D},\bm{w},\bm{b},\bm{Y}_{t-1},\class{(\bm{X})}\right)\\
\bm{Y}_t&=\clamp\left(\bm{Y}_{t-1}+a \cdot \sgn{\left(\bm{G}_{t-1}\right)}\right).
\end{aligned}
\end{equation}
where for the learning rate we set $a=\frac{\epsilon}{T}$ and $\clamp$ is a clipping function.
For most cases, the attack was white box and if applicable, the total loss $J$ included the optional gap regularization term.
However, for the $3$ Feedforward networks we computed $\bm{Y}$ using PGBP, resulting in a black box attack.

\subsection{Datasets}\label{ss:data}
We used three datasets; two synthetic ones and MNIST.

\subsubsection{Synthetic Data}
We generated two synthetic datasets, one without and another with group pooling, according to the following procedure.
First, we built a \emph{dictionary} $\bm{D}\in\mathbb{R}^{100 \times 300}$ using normalized Grassmannian packing with $75$ groups of size $4$ \cite{dhillon2008constructing}.
We generated two normalized random classifiers $\bm{w}\in\mathbb{R}^{300}$ and $\bm{w}_{PGBP}\in\mathbb{R}^{75}$ with components drawn from the normal distribution $\mathcal{N}(0,1)$ and set the bias term to zero ($\bm{b}=0$).
Next, we created the respective input sets.
We kept randomly generating $\bm{\Gamma}\in\mathbb{R}^{300}$ vectors having $8$ nonzero groups of size $4$ with activations drawn uniformly from
$[1,2]$ and computed $\bm{X}=\bm{D}\bm{\Gamma}$.
We collected two sets of $\SI{10000}{}$ $\bm{X}$ vectors that satisfied classification margin 
$\mathcal{O}(\bm{X})\geq\eta\in \{0.03, 0.1, 0.3\}$ 
in terms of the classifiers $\bm{w}$ and $\bm{w}_{PGBP}$ acting on top of $\bm{\Gamma}$ (no pooling) and the $\ell_2$ norms of the groups of $\bm{\Gamma}$ (pooled), respectively.
While running our methods, we used a single dense layer and a linear classifier layer with the true parameters ($\bm{D}$, $\bm{w}_{(PGBP)}$).

\subsubsection{MNIST Data}\label{ss:mnist}
We employed image classification on the real MNIST dataset.
The images were vectorized and we preprocessed to zero mean and unit variance.
We used a fully connected (dense) dictionary $\bm{D}\in\mathbb{R}^{784 \times 256}$, hidden representation $\bm{\Gamma}_{(G)BP} \in \mathbb{R}^{256}$ with optionally $32$ groups of size $8$ for our grouped methods, and a fully connected softmax classifier $\bm{w}$ mapping to the $10$ class probabilities acting either on top of the full $\bm{\Gamma}_{(G)BP}$ (i.e., $\bm{w}_i\in\mathbb{R}^{256}$, $i=1,\dots,10$) or the compressed
$\bm{\Gamma}_{PGBP}$ (i.e., $\bm{w}_{PGBP,i}\in\mathbb{R}^{32}$, $i=1,\dots,10$).
Since in this case the true parameters ($\bm{D}$, $\bm{w}$, $\bm{b}$) were not available for our single layer methods, we tried to learn these via backpropagation over the training set.
For this, we applied Stochastic Gradient Descent (SGD) \cite{bottou2018optimization} over $500$ epochs with early stopping patience $10$.
To prevent dead units in $\bm{D}$, we increased $\gamma$ linearly between $0$ and its final value over the initial $4$ epochs.

In agreement with the sparse case \cite{sulam2020adversarial}, we found that pretraining the dictionary using reconstruction loss (see Sect.~\ref{sss:loss}) is beneficial in the group case, too. 




\subsection{Experimental Results}\label{ss:exp_res}
We note that our numerical studies are outside of the scope of the theory as shown by Table~\ref{tab:sparsity} in Appendix~\ref{AppB} since (i) only about 50\% of the perfect group combinations could be found in the synthetic case and (ii) the group assumption is not warranted for the MNIST dataset. 

\begin{figure}[!h]
\centering     
\includegraphics[width=67.85mm]{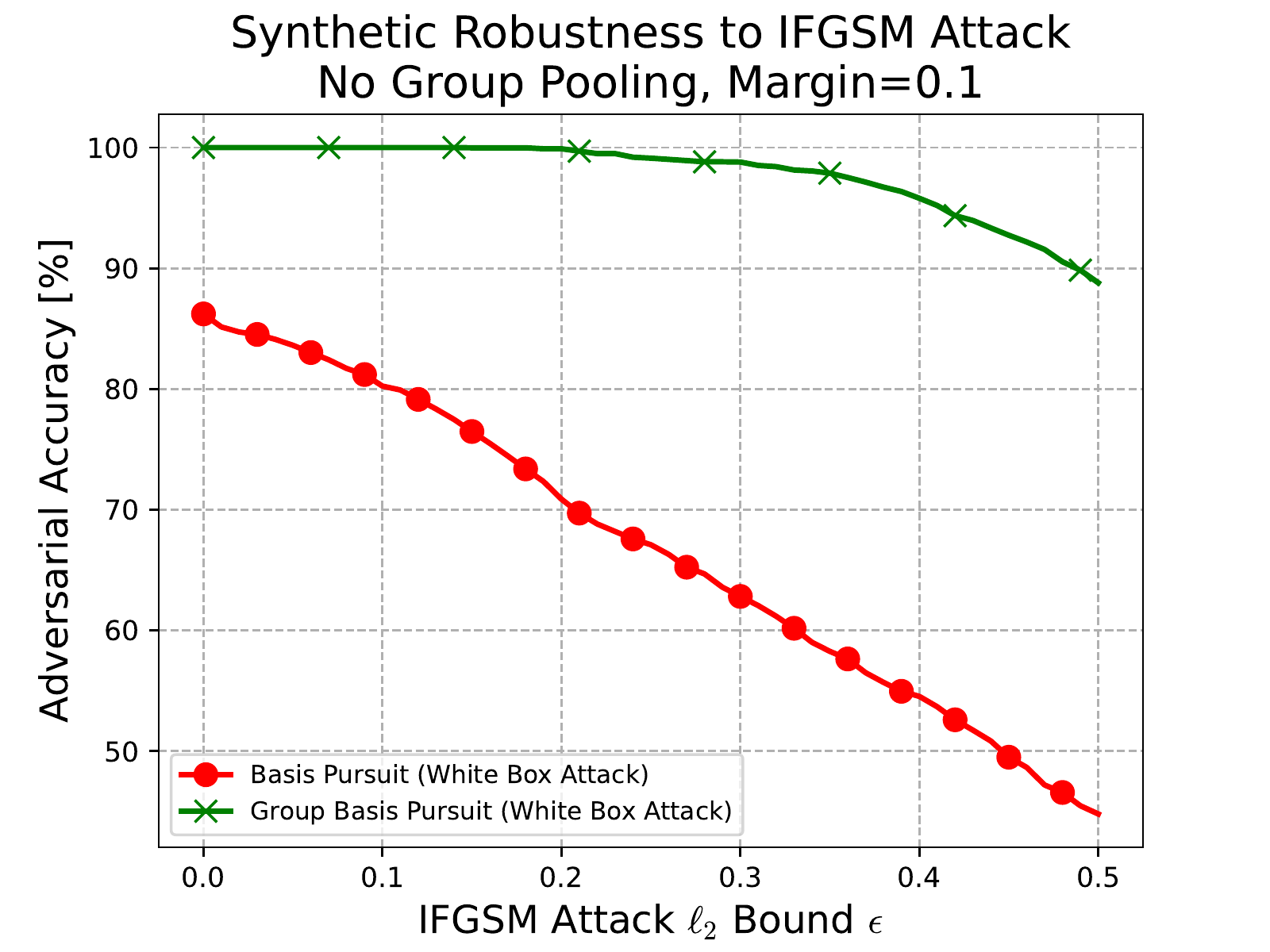}\\
\textbf{(a)}\\
\includegraphics[width=67.85mm]{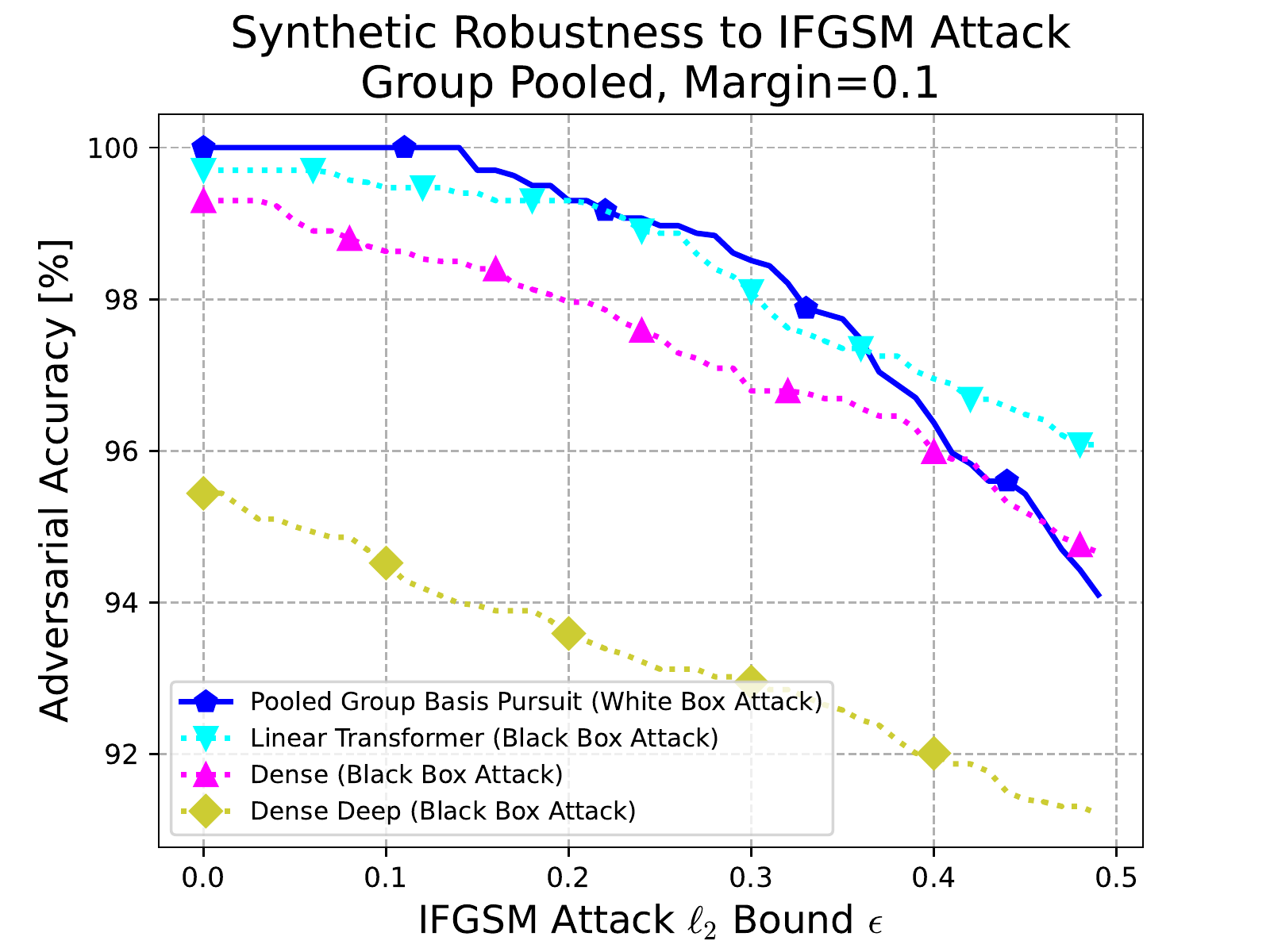}\\
\textbf{(b)}\\
\includegraphics[width=67.85mm]{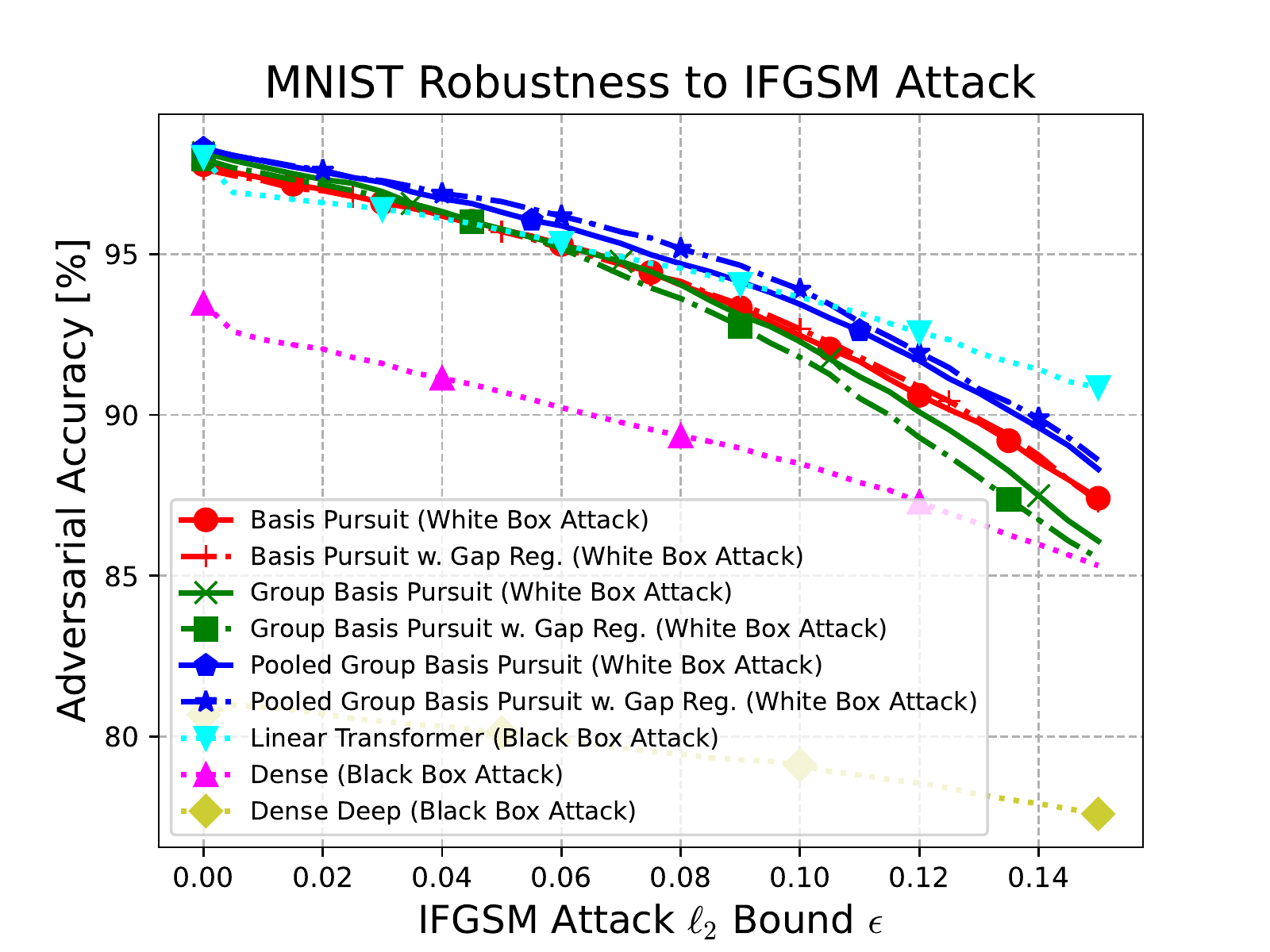}\\
\textbf{(c)}
\caption{Results for adversarial robustness against Iterative Fast Gradient Sign Method (IFGSM) attack. Datasets differ for all subfigures. Best viewed zoomed in. 
\textbf{(a):} Synthetic dataset, no group pooling: our Group Basis Pursuit (GBP, green) obtains $100\%$ accuracy for small $\epsilon$ and considerably outperforms Basis Pursuit (BP, red) as it can exploit the given group structure.
\textbf{(b):} Synthetic dataset, group pooling: Pooled Group Basis Pursuit (PGBP, blue) achieves perfect scores for small $\epsilon$. Break down is faster than for the Linear Transformer (LT, cyan) and the Dense (magenta) networks due to the difference between white box and black box attacks. Deep network (yellow) having parameter count similar to LT is overfitting.
\textbf{(c)}: MNIST dataset: PGBP is the best for small $\epsilon$, and it also consistently outperforms all BP and GBP variants for large $\epsilon$. For some methods, gap regularization (dash-dotted) increases performance. For large $\epsilon$, black box attacked LT scores the highest. Deep network overfits.}
\label{fig:net-exps}
\end{figure}

\subsubsection{Synthetic experiments}
We used three margins, $0.03$, $0.1$, and $0.3$ on the synthetic data. 
Results for margin $0.1$ of the no group pooling and group pooled synthetic experiments are shown in Fig.~\ref{fig:net-exps}~a) and b), respectively.
See Appendix~\ref{AppB} for the rest. 

For the no group pooling experiment, we found that BP achieves low accuracy even without attacks, and it breaks down rapidly for increasing $\epsilon$.
In contrast, our GBP achieves perfect scores for low $\epsilon$, since it has access to the ground truth group structure of the data, and it is able to leverage it.
For large $\epsilon$ values, it still breaks down and is slower than BP in the studies domain.
Note, however, that the search space is much larger for BP than for GBP.


For the group pooled experiment, the dense, deep dense and transformer networks were trained to approximate PGBP instead of the ground truth, hence they score worse for zero attack.
Up to $ \epsilon \approx 0.14$ values, PGBP reaches perfect accuracy.
Beyond that and due to the different nature of the attack (white box for PGBP and black box for the others), the breakdown is faster for PGBP than for the other methods. The effect is more pronounced for smaller margins (see \ref{ss:Synth_attack}).  
Out of the three feedforward estimations, the transformer performed the best.

\subsubsection{MNIST experiment}
On MNIST, we compared BP, GBP, PGBP, their respective gap regularized variants and the 3 feedfoward networks.

Among the white box attacked pursuit methods, PGBP gave the best results for both the non-attacked and for the attacked case, indicating the benefits of the pooled representation, i.e., it is more difficult to attack group norms than the elements within groups.
We think that this result deserves further investigation.
BP and GBP were worse and their curves crossed each other.

Gap regularization (Eq.~\eqref{eq:gap}) slightly increased performance for BP and PGBP, but it impairs GBP.
We believe that this technique may be improved by making it less restrictive, similarly to the modifications for mutual coherence in \cite{murdock2020dataless}, e.g., by averaging the terms.

Feedforward nets were attacked by the black box method. The Linear Transformer obtained the best results. Deep Network was difficult to teach; it was overfitting.

\section{Discussion}\label{s:disc}

We have dealt with the structural extensions of basis pursuit methods. We have extended the stability theory of sparse networks and their cascaded versions as follows: 
\begin{enumerate}[]
\setlength\itemsep{-0.1em}
    \item  The non-cascaded extension \cite{cazenavette2020architectural} that includes skip connections beyond the off-diagonal identity blocks of the matrix depicted in Fig.~\ref{fig:net-exps} that is the lower triangular part of the matrix can be filled by general blocks has stability proof.
    \item Stability proof holds if non-zero general block matrices occur in the upper triangular matrix representing unrolled feedback connections.
    \item Stability proof holds if representation elements within any layers are grouped.
    \item Different layers and groups can have different biases, diverse norms, such as $\ell_1$, $\ell_{1,2}$, and the elastic norm.
    \item The theorem is valid for Convolutional Neural Networks.
    \item Proofs are valid for positive coding for the sparse case and under certain conditions, for the group case, too.
\end{enumerate}


Feedforward estimations are fast and our experiments indicate that they are relatively accurate especially for the Linear Transformer for the group structures when there is no attack.
In case of attacks, the transformer shows reasonable robustness against black box attacks. However, it seems that transformers are also fragile for white box attacks \cite{bai2021transformers}.
Attacks can be detected as shown by the vast literature on this subject. For recent reviews, see \cite{akhtar2021advances,salehi2021unified} and the references therein. Detection of the attacks can optimize the speed if all (P)GBP and feedforward estimating networks are run in parallel and the detection is fast so it can make the choice in time. 

Performances could be improved by introducing additional regularization loss terms \cite{murdock2021reframing}.
We could improve our results by adding a loss term aiming to increase the gap between the groups that will become active and the groups that will be inactive after soft thresholding.
Our results are promising and the present loss term (Eq.~\eqref{eq:gap}) may be too strict. 
Another interesting loss term could be the minimization of the mutual coherence of $\bm{D}$ \cite{murdock2020dataless} and we leave this examination for future works.

Our experimental studies can be generalized in several ways. Firstly, a single layer can not be perfect for all problems. The hierarchy of layers is most promising for searching for groups of different sizes. As an example, edge detectors can be built hierarchically using CNNs, see, e.g., \cite{poma2020dense}.

Further, we restricted the investigations to groups of the same size and the same bias, even though that inputs may be best fit by  groups of different sizes, or even by including a subset of single elements, and the bias may also differ. This is an architecture optimization problem, where the solution is unknown. Learning of the sparse representation is however, promising since under rather strict conditions, high-quality sparse dictionaries can be found \cite{arora2015simple}. The step to search for groups is still desired since (a) the search space may become smaller by the groups and (b) the presence of the active groups may be estimated quickly and accurately using feedforward methods, especially transformers (in the absence of attacks). In turn, feedforward estimation of the groups followed by (P)GBP with different group sizes including single atoms seems worth studying.

\section{Conclusions}\label{s:conc}
We studied the adversarial robustness of sparse coding.
We proved theorems for a large variety of structural generalizations, including: groups within layers, diverse connectivities between the layers and versions of optimization costs related to the $\ell_1$ norm.
We also studied group sparse networks experimentally. We demonstrated that our GBP can outperform BP, and that our PGBP works better than both using 8 times smaller representation.
We found that PGBP offers fast feedforward estimations and the transformer version shows considerable robustness for the datasets we studied.
Finally, we showed that gap regularization can improve robustness even further, as suggested by condition \textit{4)} of Theorem~\ref{T2}.

Yet, the scope of our studies are limited from multiple perspectives.
First, the suprisingly great performance of our PGBP despite its small representation calls for further investigations and if possible, for theorems.
We believe that extensions for varying group sizes and other loss functions may provide performance improvements.

Defenses against noise, novelties, anomalies and, in particular, against adversarial attacks may be solved by combining our robust, structured sparse networks with out-of-distribution detection methods. 

%


\section*{ACKNOWLEDGEMENTS}
The research was supported by (a) the Ministry of Innovation and Technology NRDI Office within the framework of the Artificial Intelligence National Laboratory Program and (b) Application Domain Specific Highly Reliable IT Solutions project of the National Research, Development and Innovation Fund of Hungary, financed under the Thematic Excellence Programme no.\ 2020-4.1.1.-TKP2020 (National Challenges Subprogramme) funding scheme.

\clearpage\newpage

\bibliographystyle{named}
\bibliography{ijcai22}

\clearpage\newpage

\appendix
\section{Proofs of our Theorems}\label{AppA}

Most of the results are proved along the lines of the original proofs, however we will need some extensions. To be more self-contained we present the proofs which differ from the original ones. First we define our objects in \ref{SSnot}

 In \ref{SSerc} we begin with the result of J. A. Tropp \cite{tropp2006just} to describe when the solution of a \eqref{GBP} problem restricted to a subspace is the same as on the whole space. This will require subgradients and the extended version of the exact recovery condition (ERC). This will help us to prove Theorem \ref{T1} which is the main pillar of the whole theory.

Next in \ref{SSmutual} we follow the papers \cite{papyan2017working}, \cite{papyan2016working} to use the mutual coherence and the stripe norm to give guarantees that the assumptions of Theorem \ref{T1} are fulfilled to prove Theorem \ref{T2} and in \ref{SSmulti} it's multi layered extension Theorem \ref{T3}.

Then in \ref{SSSingle} we show that the trick proposed by \cite{cazenavette2020architectural} to rewrite a multi layered (BP) into a single layer (BP)-like system, for which the (BP) math was not working, works in our case for the \eqref{GBP}. We show how this helps to alleviate the condition on the dictionary that it's atoms are unit normed in the $\ell_{2}$ norm.

At the end in \ref{SSlinclas}, similarly to \cite{romano2020adversarial}, we show how linear classifiers behave under these theorems.

\subsection{Notations}\label{SSnot}

We will use the following notations and definitions. Let $\bm{D}\in\mathbb{R}%
^{N\times M}$ be our \textit{dictionary} and our goal is to solve the equation%
\begin{equation}
\bm{X}=\bm{D}\bm{\Gamma}\nonumber
\end{equation}
where $\bm{X}\in\mathbb{R}^{N}$ is the \textit{signal} and $\bm{\Gamma}\in\mathbb{R}%
^{M}$ is the \textit{sparse vector} coding our signal. We will assume here 
that $\bm{D}$ is normalized, i.e. every column in the $\ell_{2}$ norm has a length of
$1$. This condition can be weakened (see later) but we keep this condition for the sake simplicity. Let $\left\Vert \bm{Z}\right\Vert _{p}=\left[  \sum_{i=1}^{n}z_{i}^{p}\right]
^{1/p},~\bm{Z}\in\mathbb{R}^{n}$ be the usual $\ell_{p}$ norm and $\left\Vert
\bm{Z}\right\Vert _{\infty}=\max_{i\in\left\{  1,\dots,n\right\}  }\left\vert
z_{i}\right\vert $ the usual $\ell_{\infty}$ norm. We will also need the $\left(
p,q\right)  $ \textit{operator norm} of a matrix $\bm{A}$ which is defined by
$\left\Vert \bm{A}\right\Vert _{p,q}\overset{def}{=}\max_{\bm{x}\neq0}\frac{\left\Vert
\bm{A}\bm{x}\right\Vert _{q}}{\left\Vert \bm{x}\right\Vert _{p}}.$ This yields $\left\Vert
\bm{A}\bm{x}\right\Vert _{q}\leq\left\Vert \bm{A}\right\Vert _{p,q}\left\Vert \bm{x}\right\Vert
_{p}$. Moreover $\left\Vert \bm{A}\right\Vert _{p}\overset{def}{=}\left\Vert
\bm{A}\right\Vert _{p,p}$ is the common shorthand we may use sometimes. We will
consider the $\ell_{1},~\ell_{\lambda,1,2},~\ell_{1,2}$ norms at the minimization
problem of Eq.~\eqref{GBP}.

We \ will also use the standard definition of the subgradient of a function
$l\in C\left(  \mathbb{R}^{n}\right)  $ which is $\partial l\left(  \bm{x}\right)
=\left\{  g\in\mathbb{R}^{n}~|~l\left(  \bm{y}\right)  \geq l\left(  \bm{x}\right)
+\left\langle \bm{y}-\bm{x},g\right\rangle ,~\forall \bm{y}\in\mathbb{R}^{n}\right\}  .$ From
convex analysis we know that $\partial\left(  f_{1}+f_{2}\right)  \left(
\bm{x}\right)  =\partial f_{1}\left(  \bm{x}\right)  +\partial f_{2}\left(  \bm{x}\right)  $
and if $f$ is differentiable at $\bm{x}$, then $\partial f\left(  \bm{x}\right)
=\left\{  \nabla f\left(  \bm{x}\right)  \right\}  .$

Let us consider the dictionary $\bm{D}$ and let $\Lambda\subset\left\{
1,\dots,N\right\}  $ be a subset of the domain (called \textit{subdomain}).
Let us take the corresponding columns (called \textit{atoms}) of the
dictionary and their matrix will be denoted by $\bm{D}_{\Lambda}$ and called
\textit{subdictionary}. If $\bm{D}_{\Lambda}$ has full column rank, then
$\bm{D}_{\Lambda}^{\dag}\overset{def}{=}\left(  \bm{D}_{\Lambda}^{\ast}\bm{D}_{\Lambda
}\right)  ^{-1}\bm{D}_{\Lambda}^{\ast}$ is the Moore-Penrose pseudo-inverse of
$\bm{D}_{\Lambda}$ and $\bm{P}_{\Lambda}=\bm{D}_{\Lambda}\bm{D}_{\Lambda}^{\dagger}$ is the
orthogonal projection matrix to the subspace spanned by the columns of
$\bm{D}_{\Lambda}$. If $\bm{X}$ is a signal then let us denote $\bm{X}_{p,\Lambda
}\overset{def}{=}\bm{P}_{\Lambda}\bm{X}$ its projection, which is the best $\ell_{2}$
approximation of $\bm{X}$ with the subdictionary $\bm{D}_{\Lambda}$. By the assumption
that $\bm{D}_{\Lambda}$ has full column rank, we can explicit solve
\[
\bm{X}_{p,\Lambda}=\bm{D}_{\Lambda}\bm{C}_{\Lambda}\text{ with } \bm{C}_{\Lambda}=\bm{D}_{\Lambda}^{\dagger}\bm{X}_{p,\Lambda
}.
\]

Let $\left\vert \Lambda\right\vert $ denote the number of indices in $\Lambda
$. Consider the subspace $V_{\Lambda}\overset{def}{=}\left\{  \bm{v}\in
\mathbb{R}^{N}~|~\supp{}{\bm{v}}\subseteq\Lambda\right\}  $ which can be
identified with $\mathbb{R}^{\left\vert \Lambda\right\vert }$ and a function
$l\in C\left(  \mathbb{R}^{N}\right)  $. This function can be restricted to
$V_{\Lambda}$ and by the identification $V_{\Lambda}=\mathbb{R}^{\left\vert
\Lambda\right\vert }$ we can consider the restricted $l$ as a function on
$\mathbb{R}^{\left\vert \Lambda\right\vert }.$ Let $l_{\Lambda}\in C\left(
\mathbb{R}^{\left\vert \Lambda\right\vert }\right)  $ denote the function we
got after the restriction and identification which will be called\textit{ the
induced function of }$l$\textit{ on }$\Lambda$.

If we solve Eq.~\eqref{GBP} over the subdomain $\Lambda$ (i.e. on $V_{\Lambda}$)
and get a solution $\widetilde{\bm{\Gamma}}_{\Lambda}\in\mathbb{R}^{N}$ then by the
identification $V_{\Lambda}=\mathbb{R}^{\left\vert \Lambda\right\vert }$ let
$\hat{\bm{\Gamma}}_{\Lambda}$ be the vector corresponding to the solution
$\widetilde{\bm{\Gamma}}_{\Lambda}$. The vector $\hat{\bm{\Gamma}}_{\Lambda}%
\in\mathbb{R}^{\left\vert \Lambda\right\vert }$ will be called \textit{the
compression of the solution} $\widetilde{\bm{\Gamma}}_{\Lambda}$.

We will use a modified version of the fundamental lemmata found in Fuchs \cite{fuchs2004sparse} but first we need some notations.

Let $\mathcal{G}_{i},~i\in\left\{  1,\dots,k\right\}  $ be a partition of the index set $\left\{  1,\dots,M\right\}$ and $\bm{\gamma}\in\mathbb{R}^k$ a weight vector. Let $\Lambda$ be a subdomain and $\bm{g}_i\in\mathbb{R}^M$ be vectors with supports in $\mathcal{G}_{i}\cap\Lambda$ for every $i\in\left\{  1,\dots,k\right\}$. We can take the compressions $\left(\bm{g}_i\right)_{\Lambda}$ of the $\bm{g}_i$ vectors which will be vectors in $\mathbb{R}^{\Lambda}$. Now if we consider the function $\left< \bm{\gamma},\bm{l}\left( \bm{\Gamma} \right) \right> = \sum_{i=1}^{k}
\gamma_{i}\bm{l}_{\alpha_{i}} \left( \bm{\Gamma_{\mathcal{G}_{i}}}\right)$ then we can restrict it to vectors with support in $\Lambda$. Thus the restricted functions $\left(\bm{l}_{\alpha_{i}}\right)_{\Lambda}$ depend only on coordinates in $\mathcal{G}_{i}\cap\Lambda$ (and this function can be simply considered as the $\bm{l}_{\alpha_{i}}$ norm on $\mathbb{R}^{|\mathcal{G}_{i}\cap\Lambda|}\subset\mathbb{R}^{|\Lambda|}$). Therefore it's subgradient $\partial\left(\bm{l}_{\alpha_{i}}\right)_{\Lambda}$ contains only vectors with support in $\mathcal{G}_{i}\cap\Lambda$ i.e. the subgradients can be compressed into $\Lambda$. Now if we restrict our regularizer to the support $\Lambda$ we have
$\partial\left< \bm{\gamma}_{\Lambda},\bm{l}\left( \bm{\Gamma}_{\Lambda} \right) \right> = \sum_{i=1}^{k}
\gamma_{i}\partial\left(\bm{l}_{\alpha_{i}}\right)_{\Lambda} \left( \bm{\Gamma_{\mathcal{G}_{i}\cap\Lambda}}\right).$ Thus every element in the subgradient $\bm{g}\in\partial\left< \bm{\gamma}_{\Lambda},\bm{l}\left( \bm{\Gamma}_{\Lambda} \right) \right>$ can be expressed as $\bm{g}=\sum_i \gamma_i \bm{g}_i$ for some $\bm{g}_i\in \partial\left(\bm{l}_{\alpha_{i}}\right)_{\Lambda} \left( \bm{\Gamma_{\mathcal{G}_{i}\cap\Lambda}}\right)$. Also the compressed vector is just the sum of the compressions $\left(\bm{g}\right)_{\Lambda}=\sum_i \gamma_i \left(\bm{g}_i\right)_{\Lambda}$. We want to introduce a notation for this vector. Let 
\[
\partial\bm{l}_{\Lambda}\left(\hat{\bm{\Gamma}}_{\Lambda}\right)\overset{def}{=}
\left( \partial\left(\bm{l}_{\alpha_{i}}\right)_{\Lambda} \left( \bm{\Gamma_{\mathcal{G}_{1}\cap\Lambda}}\right),\dots,\partial\left(\bm{l}_{\alpha_{i}}\right)_{\Lambda} \left( \bm{\Gamma_{\mathcal{G}_{k}\cap\Lambda}}\right)\right)
\]
be the "vector of sets" and $\bm{g}\in \partial\bm{l}_{\Lambda}\left(\hat{\bm{\Gamma}}_{\Lambda}\right)$ be a "vector of vectors" such that
\[
\bm{g}=\left(\bm{g}_1,\dots,\bm{g}_k \right),\text{ for some } \bm{g}_i\in
\partial\left(\bm{l}_{\alpha_{i}}\right)_{\Lambda} \left(\bm{\Gamma_{\mathcal{G}_{i}\cap\Lambda}}\right).
\]
Now let
\[
\left(\bm{g}\star\bm{\gamma}\right)\overset{def}{=}\sum_i \gamma_i \left(\bm{g}_i\right)_{\Lambda}
\]
which is a vector in $\mathbb{R}^{|\Lambda|}$. 

\subsection{(ERC) condition and it's stability consequences}\label{SSerc}

\begin{lemma}
[Fundamental lemmata]\label{Lemmata}Let $\Lambda$ be a subdomain such that
$\bm{D}_{\Lambda}$ has maximal column rank. Assume that $\hat{\bm{\Gamma}}_{\Lambda
}$ is the compression of the solution $\widetilde{\bm{\Gamma}}_{\Lambda}$ which is
a minimizer of Eq.~\eqref{GBP} over all the vectors with support $\Lambda$. A
necessary and sufficient condition on such a minimizer is that%
\[
\bm{C}_{\Lambda}-\hat{\bm{\Gamma}}_{\Lambda}=\left(  \bm{D}_{\Lambda}^{\ast
}\bm{D}_{\Lambda}\right)  ^{-1}\left(\bm{g}\star\bm{\gamma}\right)\text{, for some }\bm{g}\in\partial \bm{l}_{\Lambda}\left(
\hat{\bm{\Gamma}}_{\Lambda}\right).
\]
where $\bm{l}_{\Lambda}$ is the compression of the vector valued function $\bm{l}$ to $\Lambda$. Moreover
the minimizer $\hat{\bm{\Gamma}}_{\Lambda}$ is unique.
\end{lemma}

\begin{proof}
If the support of $\bm{\Gamma}$ is in $\Lambda$ then $\bm{D}\bm{\Gamma}=\bm{D}_{\Lambda}\bm{\Gamma}_{\Lambda}$ and as $\bm{X}_{p,\Lambda}$ is the projection of $\bm{X}$ to the subspace spanned by the atoms of $\bm{D}_{\Lambda}$, we have
\begin{align*}
&\frac{1}{2}\left\Vert \bm{X}-\bm{D}\bar{\bm{\Gamma}}
\right\Vert _{2}^{2}+\left< \bm{\gamma},  \bm{l}\left( \bm{\Gamma}\right) \right> =\\
& \frac{1}{2}\left\Vert \bm{X}-\bm{X}_{p,\Lambda}
\right\Vert _{2}^{2}+\frac{1}{2}\left\Vert \bm{X}_{p,\Lambda}-\bm{D}_{\Lambda}\bm{\Gamma}_{\Lambda}
\right\Vert _{2}^{2}+\left< \bm{\gamma},  \bm{l}\left( \bm{\Gamma}_{\Lambda}\right) \right>.
\end{align*}
Here the first part is independent of $\Gamma$ and the subgradient (derivative) of this function contains the $0$ vector at the minimizer $\widetilde{\bm{\Gamma}}_{\Lambda}$. Thus
\[
-\bm{D}_{\Lambda}^{T}\left(\bm{X}_{p,\Lambda}-\bm{D}_{\Lambda}\bm{\Gamma}_{\Lambda} \right)+\left(\bm{g}\star\bm{\gamma}\right)=0
\]
for some $\bm{g}\in \partial\bm{l}_{\Lambda}\left(\hat{\bm{\Gamma}}_{\Lambda}\right)$. Using that $\bm{X}_{p,\Lambda}=\bm{D}_{\Lambda}\bm{C}_{\Lambda}$ with $\bm{C}_{\Lambda}=\bm{D}_{\Lambda}^{\dagger}\bm{X}_{p,\Lambda}$ and arranging the above equality we conclude the proof
\end{proof}

We recall the subgradient of the different norms. For the $\ell_{1}$ norm on
$\mathbb{R}^{n}$%

\[
\partial\left\Vert \bm{x}\right\Vert _{1}=\left\{
\underset{x_{i}\neq\underline{0}}{\sum}\sgn\left(  x_{i}\right)  \cdot\mathbf{e}_{i} + \underset{x_{j}\neq\underline{0}}{\sum} v_{j} \cdot \mathbf{e}_{j}\,\middle\vert\, v_{j}\in \left[ -1,1\right]
 \right\}.
\]

For the $\ell_{2}$ norm on on $\mathbb{R}^{n}$%

\[
\partial\left\Vert \bm{x}\right\Vert _{2}=
\begin{cases}
\frac{\bm{x}}{\left\Vert \bm{x}\right\Vert _{2}} & \text{if }\bm{x}\neq\underline{0}\\
\bm{v}\text{ for any }\left\Vert \bm{v}\right\Vert _{2}\leq1 & \text{if }\bm{x}=\underline{0}
\end{cases}
.
\]

For the $\ell_{1,2}$ group norm on $\mathbb{R}^{n}$ with groups $\mathcal{G}%
_{j},~j\in\left\{  1,\dots,k\right\}  $

\[
\partial\left\Vert \bm{x}\right\Vert _{1,2}=
\left\{
\underset{\bm{x}_{\mathcal{G}_{j}}\neq\underline{0}}{\sum}\frac{\bm{x}_{\mathcal{G}_{j}%
}}{\left\Vert \bm{x}_{\mathcal{G}_{j}}\right\Vert _{2}}+\underset{\bm{x}_{\mathcal{G}%
_{j}}=\underline{0}}{\sum}\bm{v}_{\mathcal{G}_{j}} 
\right\}.
\]
where $\bm{v}_{\mathcal{G}_{j}}$ are arbitrary vectors with $\supp{}{\left( \bm{v}_{\mathcal{G}_{j}} \right) } \subseteq\mathcal{G}_{j}\text{ and
}\left\Vert \bm{v}_{\mathcal{G}_{j}}\right\Vert _{2}\leq1$

For the $\ell_{\lambda,1,2}$ norm on $\mathbb{R}^{n}$ we have $\partial\left\Vert
\bm{x}\right\Vert _{\lambda,1,2}$ $=\lambda\partial\left\Vert \bm{x}\right\Vert
_{1}+\left(  1-\lambda\right)  \partial\left\Vert \bm{x}\right\Vert _{2}$.

\begin{remark}
\label{R1}For any $\bm{x}\in\mathbb{R}^{n}$ and for any of the norms $\ell_{1}%
,~\ell_{2},~\ell_{1,2},~\ell_{\beta,1,2}$ if $\bm{v}\in\partial\left\Vert \bm{x}\right\Vert
_{\ast}$ is a vector in the subgradient of the norm at $\bm{x}$ then we have
\[
\left\Vert \bm{v}\right\Vert _{\infty}\leq1.
\]
Moreover if $\Lambda$ is a subdomain then for any of the above norms $\ell$ the
restricted norm $\ell_{\Lambda}$ has the same property, i.e. if $\bm{v}\in
\partial\left\Vert \bm{x}\right\Vert _{\ast}$ where $\bm{v},\bm{x}\in\mathbb{R}^{\left\vert
\Lambda\right\vert },$ then $\left\Vert \bm{v}\right\Vert _{\infty}\leq1$.
\end{remark}

Moreover if $\bm{g}\in \partial\bm{l}_{\Lambda}\left(\hat{\bm{\Gamma}}_{\Lambda}\right)$ as before (where $\Lambda$ can be the whole index set), then as the supports of every $\bm{g}_i\in\partial\left(\bm{l}_{\alpha_{i}}\right)_{\Lambda} \left(\bm{\Gamma_{\mathcal{G}_{i}\cap\Lambda}}\right)$ is disjoint (where  $\bm{g}=\left(\bm{g}_1,\dots,\bm{g}_k \right)$). Thus
\begin{align*}
    \left\Vert\left(\bm{g}\star\bm{\gamma}\right)\right\Vert_{\infty}= & \left\Vert\sum_i \gamma_i \left(\bm{g}_i\right)_{\Lambda}\right\Vert_{\infty}\\
    = & \gamma_{\max} \max_{i}\left\Vert \bm{g}_i \right\Vert_{\infty}\leq \gamma_{\max}
\end{align*}

With this remark we can prove the analogue of \textit{Corollary 5 }in \cite{tropp2006just} for \eqref{GBP}.
Recall from convex analysis that if $\frac{1}{p}+\frac{1}{q}=1$ and $\frac
{1}{p^{\prime}}+\frac{1}{q^{\prime}}=1$ then%
\begin{equation}
\left\Vert \bm{A}^{\ast}\right\Vert _{p,p^{\prime}}=\left\Vert \bm{A}\right\Vert
_{q,q^{\prime}} \label{pq}%
\end{equation}

\begin{corollary}
\label{C1}Let $\Lambda$ be a subdomain and $\hat{\bm{\Gamma}}_{\Lambda}$ be the
compression of the minimizer of Eq.~\eqref{GBP} over the subdomain $\Lambda$ (i.e. we solve Eq.~\eqref{GBP} over $\Lambda$). Then%
\begin{align*}
\left\Vert \bm{C}_{\Lambda}-\hat{\bm{\Gamma}}_{\Lambda}\right\Vert _{\infty} &
\leq\gamma_{\max}\left\Vert \left(  \bm{D}_{\Lambda}^{\ast}\bm{D}_{\Lambda}\right)
^{-1}\right\Vert _{\infty,\infty}\\
\left\Vert \bm{D}_{\Lambda}\left(  \bm{C}_{\Lambda}-\hat{\bm{\Gamma}}_{\Lambda}\right)
\right\Vert _{2} &  \leq\gamma_{\max}\left\Vert \bm{D}_{\Lambda}^{\dagger}\right\Vert
_{2,1}%
\end{align*}

\end{corollary}

\begin{proof}
Similarly to \cite{tropp2006just}, by
our modified lemmata \ref{Lemmata} as $\bm{C}_{\Lambda}-\hat{\bm{\Gamma}}_{\Lambda}%
=\left(  \bm{D}_{\Lambda}^{\ast}\bm{D}_{\Lambda}\right)  ^{-1}\left(\bm{g}\star\bm{\gamma}\right)$, for some
$\bm{g}\in\partial \bm{l}_{\Lambda}\left(  \hat{\bm{\Gamma}}_{\Lambda}\right)  $ we can
take the $\ell_{\infty}$ norm to get%
\begin{align*}
& \left\Vert \bm{C}_{\Lambda}-\hat{\bm{\Gamma}}_{\Lambda}\right\Vert _{\infty}
= \left\Vert \left(  \bm{D}_{\Lambda}^{\ast}\bm{D}_{\Lambda
}\right)  ^{-1}\left(\bm{g}\star\bm{\gamma}\right)\right\Vert _{\infty}\leq \\
& \left\Vert \left(  \bm{D}_{\Lambda}^{\ast}\bm{D}_{\Lambda}\right)  ^{-1}\right\Vert
_{\infty,\infty}\left\Vert \left(\bm{g}\star\bm{\gamma}\right)\right\Vert _{\infty} \leq \gamma_{\max} \left\Vert \left(  \bm{D}_{\Lambda}^{\ast
}\bm{D}_{\Lambda}\right)  ^{-1}\right\Vert _{\infty,\infty}%
\end{align*}
where we used our notion after remark \ref{R1}. As
\[
\bm{D}_{\Lambda}\left(  \bm{C}_{\Lambda
}-\hat{\bm{\Gamma}}_{\Lambda}\right)  = \bm{D}_{\Lambda}\left(  \bm{D}_{\Lambda
}^{\ast}\bm{D}_{\Lambda}\right)  ^{-1}\left(\bm{g}\star\bm{\gamma}\right),
\]
for some $\bm{g}\in\partial \bm{l}_{\Lambda}\left(
\hat{\bm{\Gamma}}_{\Lambda}\right)  $ and $\bm{D}_{\Lambda}\left(
\bm{D}_{\Lambda}^{\ast}\bm{D}_{\Lambda}\right)  ^{-1}\left(\bm{g}\star\bm{\gamma}\right)=\left(  \bm{D}_{\Lambda
}^{\dagger}\right)  ^{\ast}\left(\bm{g}\star\bm{\gamma}\right)$ we have%
\begin{align*}
& \left\Vert \bm{D}_{\Lambda}\left(  \bm{C}_{\Lambda}-\hat{\bm{\Gamma}}_{\Lambda}\right)
\right\Vert _{2}  = \left\Vert \left(
\bm{D}_{\Lambda}^{\dagger}\right)  ^{\ast}\left(\bm{g}\star\bm{\gamma}\right)\right\Vert _{2}\leq \\
& \left\Vert \left(  \bm{D}_{\Lambda}^{\dagger}\right)  ^{\ast
}\right\Vert _{\infty,2}\left\Vert \left(\bm{g}\star\bm{\gamma}\right)\right\Vert _{\infty}  \leq\gamma_{\max} \left\Vert \left(  \bm{D}_{\Lambda}^{\dagger
}\right)  ^{\ast}\right\Vert _{2,1}%
\end{align*}
Where we used the notion after remark \ref{R1} and identity Eq.~\eqref{pq} for $\left(
\infty,2\right)  $ and $\left(  2,1\right)  $.
\end{proof}

Following the lines of \cite{tropp2006just}. Let $\Lambda$ be a subdomain and if $\omega
\in\Lambda$ then let $\bm{d}_{\omega}$ denote the corresponding atom in the
dictionary $\bm{D}$. Assume that $\widetilde{\bm{\Gamma}}_{\Lambda}$ is a minimizer of
Eq.~\eqref{GBP} over all the vectors with support $\Lambda$ and $\bm{g}$ is the
vector of subgradient vectors defined in lemmata \ref{Lemmata}, moreover $\hat{\bm{\Gamma}}
_{\Lambda}$ is the compression of $\widetilde{\bm{\Gamma}}_{\Lambda}$. If $\bm{\gamma}$ is the weight vector of Eq.~\eqref{GBP}, then we
define the $\bm{\gamma},\lambda$-\textit{weighted Strong Exact Recovery Coefficient}
\[
SERC_{\lambda}\left(  \Lambda,\hat{\bm{\Gamma}}_{\Lambda},\bm{\gamma}\right)  =\lambda\gamma_{\min}
-\max_{\omega\notin\Lambda}\left\langle \bm{D}_{\Lambda}^{\dagger}\bm{d}_{\omega
},\left(\bm{g}\star\bm{\gamma}\right)\right\rangle
\]
and the more friendly \textit{Exact Recovery Coefficient}%
\[
ERC_{\lambda}\left(  \Lambda\right)  =\lambda-\max_{\omega\notin\Lambda
}\left\Vert \bm{D}_{\Lambda}^{\dagger}\bm{d}_{\omega}\right\Vert _{1}%
\]
where in case of $\lambda=1$ we will write simply $SERC\left(  \Lambda
,\hat{\bm{\Gamma}}_{\Lambda},\bm{\gamma}\right)  \overset{def}{=}SERC_{1}\left(
\Lambda,\hat{\bm{\Gamma},\bm{\gamma}}_{\Lambda}\right)  $ and $ERC\left(  \Lambda\right)
\overset{def}{=}ERC_{1}\left(  \Lambda\right)  .$

As $\left\Vert \left(\bm{g}\star\bm{\gamma}\right)\right\Vert_{\infty}\leq \gamma_{\max}$, we have in the
$SERC_{\lambda}\left(  \Lambda,\hat{\bm{\Gamma}}_{\Lambda},\bm{\gamma}\right)  $ bound
that $\left\vert \left\langle \bm{D}_{\Lambda}^{\dagger}\bm{d}_{\omega},\left(\bm{g}\star\bm{\gamma}\right)\right\rangle
\right\vert \leq\left\Vert \bm{D}_{\Lambda}^{\dagger}\bm{d}_{\omega}\right\Vert
_{1}\left\Vert \left(\bm{g}\star\bm{\gamma}\right)\right\Vert _{\infty}\leq\gamma_{\max}\left\Vert \bm{D}_{\Lambda}^{\dagger
}\bm{d}_{\omega}\right\Vert _{1}$ which yields%
\begin{equation}
\gamma_{\max} ERC_{\lambda \frac{\gamma_{\min}}{\gamma_{\max}} } \left(  \Lambda\right)  \leq
SERC_{\lambda}\left(  \Lambda,\hat{\bm{\Gamma}}_{\Lambda},\bm{\gamma}\right)  .\label{ESERCERC}%
\end{equation}

For the next lemma we need a definition.

\begin{definition}\label{Dfullg}
\bigskip Let $\mathcal{G}%
_{i},~i\in\left\{  1,\dots,k\right\}  $ be the group partition in the problem \eqref{GBP} than a subdomain $\Lambda$ is group-full with respect to the $\ell_{2}$ part of $\bm{l}$ (or shortly group-full), if%
\[
\Lambda\cap\mathcal{G}_{i}=\mathcal{G}_{i}\text{ or }\emptyset\text{ if }\bm{l}_{\alpha_{i}}=\ell_{2}
\]
Thus every group is fully contained in $\Lambda$ or in its complement for which we choose the $\ell_{2}$ norm.
\end{definition}

Now we can prove the generalization of Lemma 6 in \cite{tropp2006just}. 

\begin{lemma}
[Correlation Condition]\label{LCorr} Assume that for the subdomain $\Lambda$
the matrix $\bm{D}_{\Lambda}$ has full column rank. Let $\hat{\bm{\Gamma}}_{\Lambda
}$ be the compression of a minimizer $\widetilde{\bm{\Gamma}}_{\Lambda}$ of
Eq.~\eqref{GBP} over the subdomain $\Lambda$ and assume that
$\Lambda$ is group-full. If among the norms of $\bm{l}$ we used elastic norms, let $\left\{\beta_{1},\dots,\beta_{r}\right\}$ be the set of the parameters used in the elastic norms and $\lambda\overset{def}{=}\min\left\{ 1,\beta_{1},\dots,\beta_{r}\right\}$. Now if
\[
\left\Vert \bm{D}^{\ast}\left(  \bm{X}-\bm{X}_{p,\Lambda}\right)  \right\Vert _{\infty}%
\leq SERC_{\lambda}\left(  \Lambda,\hat{\bm{\Gamma}}_{\Lambda},\bm{\gamma}\right)
\]
or
\[
\left\Vert \bm{D}^{\ast}\left(  \bm{X}-\bm{X}_{p,\Lambda}\right)  \right\Vert _{\infty}%
\leq\gamma_{\max} ERC_{\lambda \frac{\gamma_{\min}}{\gamma_{\max}} }\left(  \Lambda\right)
\]
then $\widetilde{\bm{\Gamma}}_{\Lambda}$ is the unique minimizer
for the global Eq.~\eqref{GBP} problem.
\end{lemma}

\begin{proof}
The proof is a modification of the proof of Lemma 6 in \cite{tropp2006just}.

From lemma \ref{Lemmata} we know that $\widetilde{\bm{\Gamma}}_{\Lambda}$ is the
unique minimizer of Eq.~\eqref{GBP} on the support $\Lambda$. \ Thus any variation of
$\widetilde{\bm{\Gamma}}_{\Lambda}$ on the support $\Lambda$ increases Eq.~\eqref{GBP}.
We will develop a condition which guarantees that the value of the objective
function Eq.~\eqref{GBP} increases when we change any other component of
$\widetilde{\bm{\Gamma}}_{\Lambda}$ which is not in $\Lambda$. Since Eq.~\eqref{GBP} is
convex, these two facts will imply that $\widetilde{\bm{\Gamma}}_{\Lambda}$ is the
global minimizer.

Let $\mathbf{e}_{\omega}\in\mathbb{R}^{M}$ be a basis vector corresponding to
the single element support $\omega\notin\Lambda$. We show that if $\delta\neq0$ then
\[
L\left(  \widetilde{\bm{\Gamma}}_{\Lambda}+\delta\cdot\mathbf{e}_{\omega}\right)  -L\left(
\widetilde{\bm{\Gamma}}_{\Lambda}\right)  >0.
\]
From lemma \ref{Lemmata} we know that $\widetilde{\bm{\Gamma}}_{\Lambda}$ is
the unique minimizer Eq.~\eqref{GBP} on the support $\Lambda$. The above two facts
yield that $\widetilde{\bm{\Gamma}}_{\Lambda}$ is a is a unique local optimum. The
convexity of Eq.~\eqref{GBP} gives that it is also a global unique optimum. 

By definitions%
\begin{align*}
&  L\left(  \widetilde{\bm{\Gamma}}_{\Lambda}+\delta\cdot\mathbf{e}_{\omega
}\right)  -L\left(  \widetilde{\bm{\Gamma}}_{\Lambda}\right)  =\\
&  \frac{1}{2}\left\{  \left\Vert \bm{X}-\bm{D}\left(  \widetilde{\bm{\Gamma}}_{\Lambda
}+\delta\cdot\mathbf{e}_{\omega}\right)  \right\Vert _{2}^{2}-\left\Vert
\bm{X}-\bm{D}\widetilde{\bm{\Gamma}}_{\Lambda}\right\Vert _{2}^{2}\right\}  + \\
& \left\langle \gamma,\bm{l}\left(  \widetilde{\bm{\Gamma}}_{\Lambda}+\delta\cdot\mathbf{e}_{\omega}\right)
-\bm{l}\left(  \widetilde{\bm{\Gamma}}_{\Lambda}\right)  \right\rangle  =\\
&  \frac{1}{2}\left\{  \left\Vert \bm{X}-\bm{D}\widetilde{\bm{\Gamma}}_{\Lambda}-\delta\cdot
\bm{d}_{\omega}\right\Vert _{2}^{2}-\left\Vert \bm{X}-\bm{D}\widetilde{\bm{\Gamma}}_{\Lambda
}\right\Vert _{2}^{2}\right\}  +\\
& \left\langle \gamma,\bm{l}\left(  \widetilde{\bm{\Gamma}}_{\Lambda}+\delta\cdot\mathbf{e}_{\omega}\right)
-\bm{l}\left(  \widetilde{\bm{\Gamma}}_{\Lambda}\right)  \right\rangle  =\\
&  \frac{1}{2}\left\{  -2\left\langle \bm{X}-\bm{D}\widetilde{\bm{\Gamma}}_{\Lambda}%
,\delta\cdot \bm{d}_{\omega}\right\rangle +\delta^{2}\left\Vert \bm{d}_{\omega
}\right\Vert _{2}^{2}\right\}  +\\
& \left\langle \gamma,\bm{l}\left(  \widetilde{\bm{\Gamma}}_{\Lambda}+\delta\cdot\mathbf{e}_{\omega}\right)
-\bm{l}\left(  \widetilde{\bm{\Gamma}}_{\Lambda}\right)  \right\rangle  =\\
&  ~-\left\langle \bm{X}-\bm{X}_{p,\Lambda}+\bm{X}_{p,\Lambda}-\bm{D}\widetilde{\bm{\Gamma}}_{\Lambda
},\delta\cdot \bm{d}_{\omega}\right\rangle +\frac{1}{2}\delta^{2}\left\Vert
\bm{d}_{\omega}\right\Vert _{2}^{2}+\\
& \left\langle \gamma,\bm{l}\left(  \widetilde{\bm{\Gamma}}_{\Lambda}+\delta\cdot\mathbf{e}_{\omega}\right)
-\bm{l}\left(  \widetilde{\bm{\Gamma}}_{\Lambda}\right)  \right\rangle  =
\end{align*}
where $\bm{d}_{\omega}=\bm{D}\mathbf{e}_{\omega}$ is the atom of $\bm{D}$ corresponding to
the index $\omega$. Now using that $\bm{X}_{p,\Lambda}=\bm{D}_{\Lambda}\bm{C}_{\Lambda}$ and
$\bm{D}\widetilde{\bm{\Gamma}}_{\Lambda}=\bm{D}_{\Lambda}\hat{\bm{\Gamma}}_{\Lambda}$we have
that%
\begin{align*}
&  L\left(  \widetilde{\bm{\Gamma}}_{\Lambda}+\delta\cdot\mathbf{e}_{\omega
}\right)  -L\left(  \widetilde{\bm{\Gamma}}_{\Lambda}\right)  = \frac{1}{2}\delta^{2}\left\Vert \bm{d}_{\omega}\right\Vert _{2}^{2}
- \\
& \delta\left\{  \left\langle \bm{X}-\bm{X}_{p,\Lambda},\bm{d}_{\omega}\right\rangle
+ \left\langle \bm{D}_{\Lambda}\bm{C}_{\Lambda}-\bm{D}\widetilde{\bm{\Gamma}}_{\Lambda},\bm{d}_{\omega
}\right\rangle \right\} \\
& + \left\langle \gamma,\bm{l}\left(  \widetilde{\bm{\Gamma}}_{\Lambda}+\delta\cdot\mathbf{e}_{\omega}\right)
-\bm{l}\left(  \widetilde{\bm{\Gamma}}_{\Lambda}\right)  \right\rangle  =\\
&  \frac{1}{2}\delta^{2}-\delta\left\{  \left\langle \bm{X}-\bm{X}_{p,\Lambda}%
,\bm{d}_{\omega}\right\rangle +\left\langle \bm{D}_{\Lambda}\left(  \bm{C}_{\Lambda
}-\hat{\bm{\Gamma}}_{\Lambda}\right)  ,\bm{d}_{\omega}\right\rangle \right\}
+\\
& \left\langle \gamma,\bm{l}\left(  \widetilde{\bm{\Gamma}}_{\Lambda}+\delta\cdot\mathbf{e}_{\omega}\right)
-\bm{l}\left(  \widetilde{\bm{\Gamma}}_{\Lambda}\right)  \right\rangle 
\end{align*}
where we used that $\supp{}{\left(  \widetilde{\bm{\Gamma}}_{\Lambda}\right)}
\subseteq\Lambda$ and that $\bm{D}$ is $\ell_{2}$ normalized. Let's estimate the part
$\bm{l}\left(  \widetilde{\bm{\Gamma}}_{\Lambda}+\delta\cdot\mathbf{e}_{\omega}\right)
-\bm{l}\left(  \widetilde{\bm{\Gamma}}_{\Lambda}\right)  $.

If for a group $\mathcal{G}_{i}$ the norm $\bm{l}_{\alpha_{i}}$ is the $\ell_{1}$ norm,  then $\bm{l}_{1}\left(  \widetilde{\bm{\Gamma}}_{\Lambda\cap\mathcal{G}_i}+\delta\cdot\mathbf{e}_{\omega\cap\mathcal{G}_i}\right)
-\bm{l}_{1}\left(  \widetilde{\bm{\Gamma}}_{\Lambda\cap\mathcal{G}_i}\right)$ equals $0$ if $\omega\notin\mathcal{G}_{i}$ or $\vert\delta\vert$ if $\omega\in\mathcal{G}_{i}$ since $\omega\notin\Lambda$. As $\Lambda$ is group-full we get the same if $\bm{l}_{\alpha_{i}}$ is the $\ell_{2}$ norm since the case $\omega\notin\mathcal{G}_{i}$ is trivial and $\omega\in\mathcal{G}_{i}$ with $\omega\notin\Lambda$ yields $\Lambda\cap\mathcal{G}_{i}=\emptyset$. 

Now if $\bm{l}_{\alpha_{i}}$ is the $\ell_{\beta_j,1,2}$ norm then if $\omega\notin\mathcal{G}_i$ then $\bm{l}_{\beta_j,1,2}\left(  \widetilde{\bm{\Gamma}}_{\Lambda\cap\mathcal{G}_i}+\delta\cdot\mathbf{e}_{\omega\cap\mathcal{G}_i}\right)
-\bm{l}_{\beta_j,1,2}\left(  \widetilde{\bm{\Gamma}}_{\Lambda\cap\mathcal{G}_i}\right)=0$. But if $\omega\in\mathcal{G}_i$, then $\bm{l}_{\beta_j,1,2}\left(  \widetilde{\bm{\Gamma}}_{\Lambda\cap\mathcal{G}_i}+\delta\cdot\mathbf{e}_{\omega\cap\mathcal{G}_i}\right)
-\bm{l}_{\beta_j,1,2}\left(  \widetilde{\bm{\Gamma}}_{\Lambda\cap\mathcal{G}_i}\right)$ equal to $\beta_j\left\Vert \delta\cdot
\mathbf{e}_{\omega}\right\Vert _{1}+\left(  1-\beta_j\right)  \left(  \left\Vert
\widetilde{\bm{\Gamma}}_{\Lambda\cap\mathcal{G}_i}+\delta\cdot\mathbf{e}_{\omega}\right\Vert
_{2}-\left\Vert \widetilde{\bm{\Gamma}}_{\Lambda\cap\mathcal{G}_i}\right\Vert _{2}\right)
\geq\beta_j\left\Vert \delta\cdot\mathbf{e}_{\omega}\right\Vert _{1}$.

Thus we have altogether that $\bm{l}\left(  \widetilde{\bm{\Gamma}}_{\Lambda}+\delta\cdot\mathbf{e}_{\omega}\right)
-\bm{l}\left(  \widetilde{\bm{\Gamma}}_{\Lambda}\right)$ is a vector containing $0$s and a single positive element which is at least $\lambda\left\vert\delta\right\vert$, where $\lambda\overset{def}{=}\min\left\{ 1,\beta_{1},\dots,\beta_{r}\right\}$. Thus

\[
\left\langle \gamma,\bm{l}\left(  \widetilde{\bm{\Gamma}}_{\Lambda}+\delta\cdot\mathbf{e}_{\omega}\right)
-\bm{l}\left(  \widetilde{\bm{\Gamma}}_{\Lambda}\right)  \right\rangle  \geq \gamma_{\min}\lambda \left\vert \delta\right\vert,
\]

For the part $\left\langle \bm{D}_{\Lambda}\left(  \bm{C}_{\Lambda}-\hat{\bm{\Gamma}
}_{\Lambda}\right)  ,\bm{d}_{\omega}\right\rangle $ we use lemma \ref{Lemmata} to
replace $\bm{C}_{\Lambda}-\hat{\bm{\Gamma}}_{\Lambda}=\left(  \bm{D}_{\Lambda
}^{\ast}\bm{D}_{\Lambda}\right)  ^{-1} \left(\bm{g}\star\bm{\gamma}\right)$, for some $\bm{g}\in\partial \bm{l}_{\Lambda}\left(
\hat{\bm{\Gamma}}_{\Lambda}\right)$ thus
\begin{align*}
&  \left\langle \bm{D}_{\Lambda}\left(  \bm{C}_{\Lambda}-\hat{\bm{\Gamma}}_{\Lambda
}\right)  ,\bm{d}_{\omega}\right\rangle =\left\langle \bm{D}_{\Lambda}\left(
\bm{D}_{\Lambda}^{\ast}\bm{D}_{\Lambda}\right)  ^{-1} \left(\bm{g}\star\bm{\gamma}\right),\bm{d}_{\omega}\right\rangle =\\
&  \left\langle \left(  \bm{D}_{\Lambda}^{\dagger}\right)  ^{\ast}%
 \left(\bm{g}\star\bm{\gamma}\right),\bm{d}_{\omega}\right\rangle =\left\langle  \left(\bm{g}\star\bm{\gamma}\right),\bm{D}_{\Lambda}^{\dagger}%
\bm{d}_{\omega}\right\rangle
\end{align*}

The above yield that%
\begin{align*}
&  L\left(  \hat{\bm{\Gamma}}_{\Lambda}+\delta\cdot\mathbf{e}\right)  -L\left(
\hat{\bm{\Gamma}}_{\Lambda}\right)  \geq\gamma_{\min}\lambda\left\vert \delta
\right\vert +\frac{1}{2}\delta^{2}- \\
& \delta\left\{  \left\langle
\left(\bm{g}\star\bm{\gamma}\right),\bm{D}_{\Lambda}^{\dagger}\bm{d}_{\omega}\right\rangle +\left\langle \bm{X}-\bm{X}_{p,\Lambda
},\bm{d}_{\omega}\right\rangle \right\}  \geq\\
&  \gamma_{\min}\lambda\left\vert \delta\right\vert +\frac{1}{2}\delta^{2}-\left\vert
\delta\right\vert \left\{  \left\vert \left\langle \left(\bm{g}\star\bm{\gamma}\right),\bm{D}_{\Lambda
}^{\dagger}\bm{d}_{\omega}\right\rangle \right\vert +\left\vert \left\langle
\bm{X}-\bm{X}_{p,\Lambda},\bm{d}_{\omega}\right\rangle \right\vert \right\}.
\end{align*}
Dropping the second order
term $\frac{1}{2}\delta^{2}>0$ to further reduce the right hand side we get
\begin{align*}
& L\left(  \widetilde{\bm{\Gamma}}_{\Lambda}+\delta\cdot\mathbf{e}\right)  -L\left(
\widetilde{\bm{\Gamma}}_{\Lambda}\right)  >\\
& \left\vert \delta\right\vert \left\{
\gamma_{\min}\lambda-\left\vert \left\langle \left(\bm{g}\star\bm{\gamma}\right),\bm{D}_{\Lambda}^{\dagger}\bm{d}_{\omega
}\right\rangle \right\vert -\left\vert \left\langle \bm{X}-\bm{X}_{p,\Lambda},\bm{d}_{\omega
}\right\rangle \right\vert \right\}  .
\end{align*}
Thus it is enough to prove that the right hand side is non-negative which is
equivalent to
\begin{equation}
\left\vert \left\langle \bm{X}-\bm{X}_{p,\Lambda},\bm{d}_{\omega}\right\rangle \right\vert
\leq\left\{  \gamma_{\min}\lambda-\left\vert \left\langle \left(\bm{g}\star\bm{\gamma}\right),\bm{D}_{\Lambda}^{\dagger
}\bm{d}_{\omega}\right\rangle \right\vert \right\}  \label{E1}%
\end{equation}
for every $\omega\notin\Lambda$. Thus if we can prove%
\begin{align*}
& \max_{\omega\notin\Lambda}\left\vert \left\langle \bm{X}-\bm{X}_{p,\Lambda},\bm{d}_{\omega
}\right\rangle \right\vert \leq\\
&  \gamma_{\min}\lambda-\max_{\omega
\notin\Lambda}\left\vert \left\langle \left(\bm{g}\star\bm{\gamma}\right),\bm{D}_{\Lambda}^{\dagger}\bm{d}_{\omega
}\right\rangle \right\vert 
\end{align*}
then Eq.~\eqref{E1} holds, thus $L\left(  \widetilde{\bm{\Gamma}}_{\Lambda}+\delta
\cdot\mathbf{e}_{\omega}\right)  -L\left(  \widetilde{\bm{\Gamma}}_{\Lambda}\right)  >0$.
\ As $\bm{X}-\bm{X}_{p,\Lambda}\perp \bm{d}_{\theta},~\forall\theta\in\Lambda$ we can take
the $\max$ on the left hand side also for these terms to get $\max
_{\omega\notin\Lambda}\left\vert \left\langle \bm{X}-\bm{X}_{p,\Lambda},\bm{d}_{\omega
}\right\rangle \right\vert =\max_{\omega}\left\vert \left\langle
\bm{X}-\bm{X}_{p,\Lambda},\bm{d}_{\omega}\right\rangle \right\vert =\left\Vert \bm{D}^{\ast
}\left(  \bm{X}-\bm{X}_{p,\Lambda}\right)  \right\Vert _{\infty}.$ Therefore%
\begin{align*}
& \left\Vert \bm{D}^{\ast}\left(  \bm{X}-\bm{X}_{p,\Lambda}\right)  \right\Vert _{\infty}%
\leq  \gamma_{\min}\lambda-\max_{\omega\notin\Lambda}\left\vert \left\langle
\left(\bm{g}\star\bm{\gamma}\right),\bm{D}_{\Lambda}^{\dagger}\bm{d}_{\omega}\right\rangle \right\vert
=\\
& SERC_{\lambda}\left(  \Lambda,\hat{\bm{\Gamma}}_{\Lambda},\bm{\gamma}\right)
\end{align*}
is a sufficient condition for every perturbation away $\widetilde{\bm{\Gamma}
}_{\Lambda}$ to increase the objective function Eq.~\eqref{GBP} . Since Eq.~\eqref{GBP}
\ is convex, it follows that $\widetilde{\bm{\Gamma}}_{\Lambda}$ is the global
minimizer of Eq.~\eqref{GBP}. \ Using the inequality Eq.~\eqref{ESERCERC} we conclude
the proof.
\end{proof}

\begin{corollary}
Assume that for the subdomain $\Lambda$ the matrix $\bm{D}_{\Lambda}$ has full
column rank and for the signal $\bm{X}$%
\[
\left\Vert \bm{X}-\bm{X}_{p,\Lambda}\right\Vert _{2}\leq \gamma_{\max} ERC_{\lambda \frac{\gamma_{\min}}{\gamma_{\max}} }\left(  \Lambda\right)
\]
where $\lambda$ is as in the previous lemma for the \eqref{GBP} problem, moreover, assume that $\Lambda$ is group-full. If $\hat{\bm{\Gamma}}$ is a
minimizer of the global program Eq.~\eqref{GBP} then $\supp{}{\left(  \hat{\bm{\Gamma}
}\right)}  \subseteq\Lambda$ and $\hat{\bm{\Gamma}}$ is the unique minimizer.
\end{corollary}

\begin{proof}
As $\left\Vert \bm{D}^{\ast}\left(  \bm{X}-\bm{X}_{p,\Lambda}\right)  \right\Vert _{\infty
}=\max_{\omega}\left\vert \left\langle \bm{d}_{\omega},\bm{X}-\bm{X}_{p,\Lambda}\right\rangle
\right\vert \leq\max_{\omega}\left\Vert \bm{d}_{\omega}\right\Vert _{2}\left\Vert
\bm{X}-\bm{X}_{p,\Lambda}\right\Vert _{2}$ and $\bm{D}$ is an $\ell_{2}$ normalized dictionary
we have $\left\Vert \bm{D}^{\ast}\left(  \bm{X}-\bm{X}_{p,\Lambda}\right)  \right\Vert
_{\infty}\leq\left\Vert \bm{X}-\bm{X}_{p,\Lambda}\right\Vert _{2}$. Thus the condition
of lemma \ref{LCorr} is fulfilled, yielding that the minimizer of Eq.~\eqref{GBP}
over $\Lambda$ is the global unique minimizer.
\end{proof}

The above results grouped together give the extension of Theorem 8 in \cite{tropp2006just}.

\begin{theorem}
\label{T1} Let $\bm{\Gamma}_{GBP}$ be a solution of Eq.~\eqref{GBP} and assume that for the group-full $\Lambda$ the columns of $\bm{D}_{\Lambda}$ are linearly
independent, moreover that $ERC_{\lambda \frac{\gamma_{\min}}{\gamma_{\max}}
}\left(  \Lambda\right)  >0$ where $\lambda\overset{def}{=}\min\left\{ 1,\beta_{1},\dots,\beta_{r}\right\}$ as before (i.e. if among the norms of $\bm{l}$ we used elastic norms, let $\left\{\beta_{1},\dots,\beta_{r}\right\}$ be the set of the parameters used in the elastic norms). 

If for the
signal $\bm{X}$ we have%
\[
\left\Vert \bm{D}^{\ast}\left(  \bm{X}-\bm{X}_{p,\Lambda}\right)  \right\Vert _{\infty}%
\leq  \gamma_{\max} ERC_{\lambda \frac{\gamma_{\min}}{\gamma_{\max}} }\left(  \Lambda\right)
\]
then
\begin{enumerate}
\item[1)]$\supp{}\bm{\Gamma}_{GBP}\subset\Lambda,$%
\item[2)] If $\bm{X}_{p,\Lambda}=\bm{D}_{\Lambda}\bm{C}_{\Lambda}$ where $\bm{C}_{\Lambda}=\bm{D}_{\Lambda}^{\dagger}\bm{X}_{p,\Lambda}$ then $\left\Vert
\bm{\Gamma}_{GBP}-\bm{C}_{\Lambda}\right\Vert _{\infty}\leq\gamma_{\max}\left\Vert \left(
\bm{D}_{\Lambda}^{\ast}\bm{D}_{\Lambda}\right)  ^{-1}\right\Vert _{\infty}$,
\item[3)] 
$\left\{i~\Big|~\left\vert \left(\bm{C}_{\Lambda}\right)_{i}\right\vert
>\gamma_{\max}\left\Vert \left(  \bm{D}_{\Lambda}%
^{\ast}\bm{D}_{\Lambda}\right)  ^{-1}\right\Vert _{\infty}\right\}\subseteq\supp{}{\bm{\Gamma}_{GBP}}$,

\item[4)] Moreover the minimizer $\bm{\Gamma}_{GBP}$ is unique.
\end{enumerate}
\end{theorem}

\begin{proof}
By our assumption lemma \ref{LCorr} yields that the unique global minimizer of
Eq.~\eqref{GBP} is the same as the minimizer of Eq.~\eqref{GBP} over the vectors with
support $\Lambda$ yielding 1) and 4). We can use corollary \ref{C1} to get 2)
which follows 3).
\end{proof}

\subsection{Mutual coherence, stripe norm to estimate (ERC)}\label{SSmutual}

Now we proceed with some of the results of Papyan et all \cite{papyan2016working}
and try to extend these for the \eqref{GBP} problem. For
this we will use that usually $\bm{D}$ has a special form which can enhance the
constants used in the theorems. If $\bm{D}$ is a matrix corresponding to a
convolutional neural network with classical kernels such as in \cite{papyan2017working},
then a column $\bm{d}_{i}$ of $\bm{D}$ is orthogonal to most of the columns in $\bm{D}$. Let us define some useful constants and the stripe norm.

\begin{definition}
\label{D1}Let $k_{\bm{d}_{i}}$ denote the maximal number of columns in $\bm{D}$ which
are not orthogonal to the column $\bm{d}_{i}$ and $k_{\bm{D}}\overset{def}{=}\max
_{i}k_{i}$ which we call the maximal stripe of $\bm{D}$. Let $k_{\Lambda}$ denote
the number of indexes in $\Lambda$. 
Moreover, if $1_{\Lambda}\in\mathbb{R}^{M}$ is the vector which has $1s$ at
the indices in $\Lambda$ and $0s$ otherwise, then let $k_{\Lambda
,\bm{D}}\overset{def}{=}\left\Vert 1_{\Lambda}\right\Vert _{0,st}$ which measures
the maximal theoretical number of columns in $\bm{D}_{\Lambda}$ which are possibly
not orthogonal to $\bm{d}_{i}$ for any $i.$
\end{definition}

In the papers \cite{papyan2017working}, \cite{papyan2016working} the norm $\left\Vert \bm{\Gamma}\right\Vert
_{0,\infty}$ is used. By the assumption on $\bm{D}$ these papers made (see at
Definition 1 in \cite{papyan2017working}) we have $\left\Vert \bm{\Gamma}\right\Vert _{0,st}%
\leq\left\Vert \bm{\Gamma}\right\Vert _{0,\infty}$. The difference comes from the
fact that in the $\left(  2n-1\right)  m$ long stripe in $\bm{D}$ there can be
orthogonal atoms and every atom outside of the stripe is orthogonal to $\bm{d}_{i}%
$, where the stripe corresponds to $\bm{d}_{i}$ which has at most an $n$-length
non-zero part. Also the little bit ambiguous name \textit{stripe} is used here
for these historical reasons.

Note that $\left\Vert \bm{\Gamma}_{\Lambda_{i}\left(  \bm{D}\right)  }\right\Vert
_{0}=\left\vert \supp{}{\bm{\Gamma}}\cap\Lambda_{i}\left(  \bm{D}\right)\right\vert $ where the right hand side is the number of elements in the intersection thus
$\left\vert \supp{}{\bm{\Gamma}}\cap\Lambda_{i}\left(  \bm{D}\right)\right\vert
\leq\left\Vert \bm{\Gamma}\right\Vert _{0,st}$ for every index $i$.

Note that $\bm{D}$ was normed, i.e. the columns are $\ell_{2}$ unit vectors.
In applications the goal is to use a dictionary $\bm{D}$ with low mutual coherence
i.e. $\mu\left(  \bm{D}\right)  $ is as low as possible thus $\bm{D}^{\ast}\bm{D}$ is a
diagonally dominant matrix which we will assume hereon.

Recall that
\[
SERC_{\lambda}\left(  \Lambda,\hat{\bm{\Gamma}}_{\Lambda},\bm{\gamma}\right)  =\lambda\gamma_{\min} -\max_{\omega\notin\Lambda}\left\langle \bm{D}_{\Lambda}^{\dagger}\bm{d}_{\omega
},\left(\bm{g}\star\bm{\gamma}\right)\right\rangle
\]
and 
\[
ERC_{\lambda}\left(  \Lambda\right)  =\lambda-\max_{\omega\notin\Lambda
}\left\Vert \bm{D}_{\Lambda}^{\dagger}\bm{d}_{\omega}\right\Vert _{1}
\]
We try to give a condition such that $ERC_{\lambda}\left(
\Lambda\right)  >0$ holds.

If $\Lambda$ is the support and $\omega\notin\Lambda$, then
\begin{align*}
& \left\Vert \bm{D}_{\Lambda}^{\dagger}\bm{d}_{\omega}\right\Vert _{1}=\left\Vert \left(
\bm{D}_{\Lambda}^{\ast}\bm{D}_{\Lambda}\right)  ^{-1}\bm{D}_{\Lambda}^{\ast}\bm{d}_{\omega
}\right\Vert _{1}\leq\\
& \left\Vert \left(  \bm{D}_{\Lambda}^{\ast}\bm{D}_{\Lambda}\right)
^{-1}\right\Vert _{1}\left\Vert \bm{D}_{\Lambda}^{\ast}\bm{d}_{\omega}\right\Vert _{1}.%
\end{align*}
As $\bm{D}^{\ast}\bm{D}$ is a diagonally dominant matrix we can use the
Ahlberg-Nilson-Varah bound
\begin{align*}
& \left\Vert \left(  \bm{D}_{\Lambda}^{\ast}\bm{D}_{\Lambda}\right)  ^{-1}\right\Vert
_{1}\leq\frac{1}{1-\max_{i\in\Lambda}\sum_{j\in\Lambda,j\neq i}\left\vert
\left\langle \bm{d}_{i},\bm{d}_{j}\right\rangle \right\vert }\leq\\
& \frac{1}{1-\max
_{i}\left\Vert \bm{D}_{\Lambda-\left\{  i\right\}  }^{\ast}\bm{d}_{i}\right\Vert _{1}%
},\label{E2}%
\end{align*}
Let $co~S$ denote the set containing the indices not contained in $S$. Then
for $\omega\notin\Lambda,$ $\left\Vert \bm{D}_{\Lambda}^{\ast}\bm{d}_{\omega}\right\Vert
_{1}\leq\left\Vert \bm{D}_{co\left\{  \omega\right\}  }^{\ast}\bm{d}_{\omega}\right\Vert
_{1}$. Thus for $\omega\notin\Lambda$%
\begin{equation}
\left\Vert \bm{D}_{\Lambda}^{\dagger}\bm{d}_{\omega}\right\Vert _{1}\leq\frac{\left\Vert
\bm{D}_{co\left\{  \omega\right\}  }^{\ast}\bm{d}_{\omega}\right\Vert _{1}}{1-\max
_{i}\left\Vert \bm{D}_{co\left\{  i\right\}  }^{\ast}\bm{d}_{i}\right\Vert _{1}%
},\label{E3}%
\end{equation}
Since $\sum_{j\neq i}\left\vert \left\langle \bm{d}_{i},\bm{d}_{j}\right\rangle
\right\vert \leq\left(  k_{\bm{D}}-1\right)  \mu\left(  \bm{D}\right)  $ we have the
more nicer form%
\begin{equation}
\left\Vert \bm{D}_{\Lambda}^{\dagger}\bm{d}_{\omega}\right\Vert _{1}\leq\frac{\left(
k_{\bm{D}}-1\right)  \mu\left(  \bm{D}\right)  }{1-\left(  k_{\bm{D}}-1\right)  \mu\left(
\bm{D}\right)  }.\label{B1}%
\end{equation}
This is typically not strict if $k_{\bm{D}}$ is big but for a $\bm{d}_{i}$ there are
only a few atoms where $\left\vert \left\langle \bm{d}_{i},\bm{d}_{j}\right\rangle
\right\vert \thickapprox\mu\left(  \bm{D}\right)  $ and for the rest $\left\vert
\left\langle \bm{d}_{i},\bm{d}_{j}\right\rangle \right\vert \ll\mu\left(  \bm{D}\right)  $
then $\sum_{j\neq i}\left\vert \left\langle \bm{d}_{i},\bm{d}_{j}\right\rangle
\right\vert <\left(  k_{\bm{D}}-1\right)  \mu\left(  \bm{D}\right)  $. Also for
$\omega\notin\Lambda,$ $\left\Vert \bm{D}_{\Lambda}^{\ast}\bm{d}_{\omega}\right\Vert
_{1}\leq\left\Vert \bm{D}_{co\left\{  \omega\right\}  }^{\ast}\bm{d}_{\omega}\right\Vert
_{1}\leq\left(  k_{\bm{D}}-1\right)  \mu\left(  \bm{D}\right)  $ is not strict typically
for the same reason.

For an $\omega\notin\Lambda,$ we can also use the estimates $\left\Vert
\bm{D}_{\Lambda}^{\ast}\bm{d}_{\omega}\right\Vert _{1}\leq k_{\Lambda}\mu\left(
\bm{D}\right)  $ and $\left\Vert \bm{D}_{\Lambda-\left\{  i\right\}  }^{\ast}%
\bm{d}_{i}\right\Vert _{1}\leq\left(  k_{\Lambda}-1\right)  \mu\left(  \bm{D}\right)  $
to get%
\[
\left\Vert \bm{D}_{\Lambda}^{\dagger}\bm{d}_{\omega}\right\Vert _{1}\leq\frac
{k_{\Lambda}\mu\left(  \bm{D}\right)  }{1-\left(  k_{\Lambda}-1\right)  \mu\left(
\bm{D}\right)  }.
\]

An other bound can be achieved by the estimates $\left\Vert \bm{D}_{\Lambda}^{\ast
}\bm{d}_{\omega}\right\Vert _{1}\leq k_{\Lambda,\bm{D}}\mu\left(  \bm{D}\right)  $ and
$\left\Vert \bm{D}_{\Lambda-\left\{  i\right\}  }^{\ast}\bm{d}_{i}\right\Vert _{1}%
\leq\left(  k_{\Lambda,\bm{D}}-1\right)  \mu\left(  \bm{D}\right)  $ yielding a similar
bound to the above%
\[
\left\Vert \bm{D}_{\Lambda}^{\dagger}\bm{d}_{\omega}\right\Vert _{1}\leq\frac
{k_{\Lambda,\bm{D}}\mu\left(  \bm{D}\right)  }{1-\left(  k_{\Lambda,\bm{D}}-1\right)
\mu\left(  \bm{D}\right)  }.
\]

Finally for the last bound let $\bm{\Gamma}\in\mathbb{R}^{M}$ be an arbitrary
vector with $\supp{}{\bm{\Gamma}}=\Lambda$ then $\left\Vert \bm{D}_{\Lambda}^{\ast
}\bm{d}_{\omega}\right\Vert _{1}\leq\left\vert \Lambda\cap\Lambda_{\omega
}\left(\bm{D}\right)\right\vert \mu\left(  \bm{D}\right)  \leq\left\Vert \bm{\Gamma}\right\Vert _{0,st}%
\mu\left(  \bm{D}\right)  $ where $\left\vert \Lambda\cap\Lambda_{\omega
}\left(\bm{D}\right)\right\vert =\left\vert \supp{}{\bm{\Gamma}}\cap\Lambda_{\omega
}\left(\bm{D}\right)\right\vert \leq\left\Vert \bm{\Gamma}\right\Vert _{0,st}$ which is explained
after definition \ref{D1}. Moreover%
\begin{align}\label{E5}
& \left\Vert \left(  \bm{D}_{\Lambda}^{\ast}\bm{D}_{\Lambda}\right)  ^{-1}\right\Vert _{1}
\leq\frac{1}{1-\max_{j\in\Lambda}\sum_{i\in\Lambda,j\neq i}\left\vert
\left\langle \bm{d}_{i},\bm{d}_{j}\right\rangle \right\vert }\leq
\\
& \frac{1}{1-\max_{j\in\Lambda}\left(  \left\vert \Lambda\cap\Lambda
_{j}\left(\bm{D}\right)\right\vert -1\right)  \mu\left(  \bm{D}\right)  } = \nonumber\\
& \frac{1}{1-\left(
\left\Vert \bm{\Gamma}\right\Vert _{0,st}-1\right)  \mu\left(  \bm{D}\right)} \nonumber
\end{align}

So we have%
\begin{equation}
\left\Vert \bm{D}_{\Lambda}^{\dagger}\bm{d}_{\omega}\right\Vert _{1}\leq\frac{\left\Vert
\bm{\Gamma}\right\Vert _{0,st}\mu\left(  \bm{D}\right)  }{1-\left(  \left\Vert
\bm{\Gamma}\right\Vert _{0,st}-1\right)  \mu\left(  \bm{D}\right)  }. \label{E4}%
\end{equation}

With the help of the above estimates we can give conditions to assure that the ERC bound is positive.

%

\begin{proposition}
\label{P1}If $k_{\bm{D}}<1+\frac{\lambda}{\left(  1+\lambda\right)  \mu\left(
\bm{D}\right)  }$ \ \ or $k_{\Lambda}<\frac{\lambda}{1+\lambda}\left(  1+\frac
{1}{\mu\left(  \bm{D}\right)  }\right)  $ \ \ or $k_{\Lambda,\bm{D}}<\frac{\lambda
}{1+\lambda}\left(  1+\frac{1}{\mu\left(  \bm{D}\right)  }\right)  $ then
$ERC_{\lambda}\left(  \Lambda\right)  >0$. Moreover, if $\bm{\Gamma}\in
\mathbb{R}^{M}$ with $\supp{}{\bm{\Gamma}}=\Lambda$ then if $\left\Vert
\bm{\Gamma}\right\Vert _{0,st}<\frac{\lambda}{1+\lambda}\left(  1+\frac{1}%
{\mu\left(  \bm{D}\right)  }\right)  $ then $ERC_{\lambda}\left(  \Lambda\right)
>0$. If $\bm{\Gamma}\in\mathbb{R}^{M}$ with $\supp{}{\bm{\Gamma}}=\Lambda$ then
$\left\Vert \bm{\chi}_{\bm{\Gamma},\mathcal{G}}\right\Vert _{0,st}<\frac{\lambda
}{1+\lambda}\left(  1+\frac{1}{\mu\left(  \bm{D}\right)  }\right)  $ yields
$ERC_{\lambda}\left(  \Lambda_{full}\right)  >0$
\end{proposition}

\begin{proof}
Solving $\frac{\left(  k_{\bm{D}}-1\right)  \mu\left(  \bm{D}\right)  }{1-\left(
k_{\bm{D}}-1\right)  \mu\left(  \bm{D}\right)  }<\lambda$ we get $k_{\bm{D}}<1+\frac{\lambda
}{\left(  1+\lambda\right)  \mu\left(  \bm{D}\right)  }$ which by Eq.~\eqref{B1} yields
$\left\Vert \bm{D}_{\Lambda}^{\dagger}\bm{d}_{\omega}\right\Vert _{1}<\lambda$%
,~$\forall\omega\notin\Lambda$ thus $ERC_{\lambda}\left(  \Lambda\right)
>0$.\newline Solving $\frac{k_{\Lambda}\mu\left(  \bm{D}\right)  }{1-\left(
k_{\Lambda}-1\right)  \mu\left(  \bm{D}\right)  }<\lambda$ we get $k_{\Lambda
}<\frac{\lambda}{1+\lambda}\left(  1+\frac{1}{\mu\left(  \bm{D}\right)  }\right)  $
and the same is true in the cases $k_{\Lambda,\bm{D}}$ and $\left\Vert
\bm{\Gamma}\right\Vert _{0,st}$. Moreover replacing $\bm{\Gamma}$ with $\bm{\chi}
_{\bm{\Gamma},\mathcal{G}}$ we get the last part.
\end{proof}

Note that if we don't use the $\ell_{\beta,1,2}$ ellastic norm, then $\lambda=1$ and we get back the usual $\frac{1}{2}\left(  1+\frac{1}{\mu\left(  \bm{D}\right)  }\right)  $ bound in most of the above cases.

\begin{lemma}
Let $\bm{Y}=\bm{X}+\bm{E}$ where $\bm{X}=\bm{D}\bm{\Gamma}$ is a signal contaminated with noise $\bm{E}$ and
$\supp{}{\bm{\Gamma}}\subseteq\Lambda$. If $\bm{Y}_{p,\Lambda}$ is the projection of $\bm{Y}$ to
the subspace spanned by $\bm{D}_{\Lambda}$ then%
\[
\left\Vert \bm{D}^{\ast}\left(  \bm{Y}-\bm{Y}_{p,\Lambda}\right)  \right\Vert _{\infty}%
\leq\left\Vert \bm{E}\right\Vert _{2}\sqrt{\left(  k_{\bm{D}}-1\right)  \mu\left(
\bm{D}\right)  }%
\]

\end{lemma}

\begin{proof}
The first part of the proof is the same as the proof of Lemma 1 in \cite{papyan2016working}.
If $\bm{P}_{\Lambda}$ is the projection to the linear subspace spanned by the
columns of $\bm{D}_{\Lambda}$ we have that if $\bm{Y}_{p,\Lambda}=\bm{P}_{\Lambda}\bm{Y}$ then
\[
\bm{Y}-\bm{Y}_{p,\Lambda}\perp \bm{d}_{\omega},~\forall\omega\in\Lambda
\]
Thus $\left\Vert \bm{D}^{\ast}\left(  \bm{Y}-\bm{Y}_{p,\Lambda}\right)  \right\Vert _{\infty
}=\left\Vert \bm{D}_{co\Lambda}^{\ast}\left(  \bm{Y}-\bm{Y}_{p,\Lambda}\right)  \right\Vert
_{\infty}$ where $co\Lambda$ is the complement of $\Lambda$. Now as $\supp{}{\bm{\Gamma}}\subseteq\Lambda$ yields $\bm{X}=\bm{D}_{\Lambda}\bm{\Gamma}_{\Lambda}$ we have
\begin{align*}
& \left\Vert \bm{D}_{co\Lambda}^{\ast}\left(  \bm{Y}-\bm{Y}_{p,\Lambda}\right)  \right\Vert
_{\infty}=\left\Vert \bm{D}_{co\Lambda}^{\ast}\left(  \bm{I}-\bm{P}_{\Lambda}\right)
\bm{Y}\right\Vert _{\infty}=\\
& \left\Vert \bm{D}_{co\Lambda}^{\ast}\left(  \bm{I}-\bm{P}_{\Lambda
}\right)  \left(  \bm{D}_{\Lambda}\bm{\Gamma}_{\Lambda}+\bm{E}\right)  \right\Vert _{\infty}%
\end{align*}
Since $\bm{D}_{\Lambda}\bm{\Gamma}_{\Lambda}$ is in the linear subspace spanned by the
columns of $\bm{D}_{\Lambda}$ we have $\left(  \bm{I}-\bm{P}_{\Lambda}\right)  \bm{D}_{\Lambda
}\bm{\Gamma}_{\Lambda}=0$. So
\begin{align*}
& \left\Vert \bm{D}_{co\Lambda}^{\ast}\left(  \bm{Y}-\bm{Y}_{p,\Lambda}\right)  \right\Vert
_{\infty}  =\left\Vert \bm{D}_{co\Lambda}^{\ast}\left(  \bm{I}-\bm{P}_{\Lambda}\right)
\bm{E}\right\Vert _{\infty}\leq\\
& \left\Vert \bm{D}_{co\Lambda}^{\ast}\left(  \bm{I}-\bm{P}_{\Lambda}\right)  \bm{E}\right\Vert
_{2}  =\\
& \sqrt{\left\langle \bm{D}_{co\Lambda}^{\ast}\left(  \bm{I}-\bm{P}_{\Lambda}\right)
\bm{E},\bm{D}_{co\Lambda}^{\ast}\left(  \bm{I}-\bm{P}_{\Lambda}\right)  \bm{E}\right\rangle }=\\
& \sqrt{\left\langle \left(  \bm{I}-\bm{P}_{\Lambda}\right)  \bm{E},\bm{D}_{co\Lambda}\bm{D}_{co\Lambda
}^{\ast}\left(  \bm{I}-\bm{P}_{\Lambda}\right)  \bm{E}\right\rangle }  \leq\\
& \sqrt{\left\Vert
\left(  \bm{I}-\bm{P}_{\Lambda}\right)  \bm{E}\right\Vert _{2}\left\Vert \bm{D}_{co\Lambda
}\bm{D}_{co\Lambda}^{\ast}\left(  \bm{I}-\bm{P}_{\Lambda}\right)  \bm{E}\right\Vert _{2}}\leq\\
& \sqrt{\left\Vert \left(  \bm{I}-\bm{P}_{\Lambda}\right)  \bm{E}\right\Vert _{2}\left\Vert
\bm{D}_{co\Lambda}\bm{D}_{co\Lambda}^{\ast}\right\Vert _{2}\left\Vert \left(
\bm{I}-\bm{P}_{\Lambda}\right)  \bm{E}\right\Vert _{2}} =\\
& \left\Vert \left(
\bm{I}-\bm{P}_{\Lambda}\right)  \bm{E}\right\Vert _{2}\sqrt{\max_{j}\lambda_{j}\left(
\bm{D}_{co\Lambda}\bm{D}_{co\Lambda}^{\ast}\right)  }%
\end{align*}
where $\lambda_{j}\left(  \bm{D}_{co\Lambda}\bm{D}_{co\Lambda}^{\ast}\right)  $-s are
the eigenvalues of $\bm{D}_{co\Lambda}\bm{D}_{co\Lambda}^{\ast}$. Since the matrices
$\bm{D}_{co\Lambda}\bm{D}_{co\Lambda}^{\ast}$ and $\bm{D}_{co\Lambda}^{\ast}\bm{D}_{co\Lambda}$
have the same non zero eigenvalues, we have that $\max_{j}\lambda_{j}\left(
\bm{D}_{co\Lambda}\bm{D}_{co\Lambda}^{\ast}\right)  =$ $\max_{j}\lambda_{j}\left(
\bm{D}_{co\Lambda}^{\ast}\bm{D}_{co\Lambda}\right)  $. Using the trick of Lemma 1 in
\cite{papyan2017working}, by the Gerschgorin theorem if $\theta\in co\Lambda$ we have that
the eigenvalues of $\bm{D}_{co\Lambda}^{\ast}\bm{D}_{co\Lambda}$ lie in the Gerschgorin
circles yielding $\left\vert \lambda_{j}\left(  \bm{D}_{co\Lambda}^{\ast
}\bm{D}_{co\Lambda}\right)  -1\right\vert \leq$ $\max_{i\in co\Lambda}\sum_{j\in
co\Lambda,j\neq i}\left\vert \left\langle \bm{d}_{i},\bm{d}_{j}\right\rangle \right\vert
\leq\left(  k_{\bm{D}}-1\right)  \mu\left(  \bm{D}\right)  .$

As $\left(  \bm{I}-\bm{P}_{\Lambda}\right)  \bm{E}$ is the orthogonal component of $\bm{E}$
perpendicular to the subspace spanned by the columns of $\bm{D}_{\Lambda}$ we have
that $\left\Vert \left(  \bm{I}-\bm{P}_{\Lambda}\right)  \bm{E}\right\Vert _{2}%
\leq\left\Vert \bm{E}\right\Vert _{2}.~$Thus%
\[
\left\Vert \bm{D}^{\ast}\left(  \bm{Y}-\bm{Y}_{p,\Lambda}\right)  \right\Vert _{\infty}%
\leq\left\Vert \bm{E}\right\Vert _{2}\sqrt{\left(  k_{\bm{D}}-1\right)  \mu\left(
\bm{D}\right)  }%
\]

\end{proof}

An other approach is to use the properties of the support of $\bm{\Gamma}$, where
$\bm{\Gamma}$ is the perfect solution $\bm{X}=\bm{D}\bm{\Gamma}$ and the local amplitude of the error $\left\Vert \bm{E}\right\Vert _{L,\bm{D}}$ which was defined at the end of section \ref{SecBandN}. %
For a nice convolutional dictionary as in \cite{papyan2016working}, \cite{papyan2017working} the local amplitude
can be calculated easily.

We recall Lemma 1 of \cite{papyan2016working} which can be proved for $\left\Vert
\bm{\Gamma}\right\Vert _{0,st}$.

\begin{lemma}
\label{L2}Let $\bm{Y}=\bm{X}+\bm{E}$ where $\bm{X}=\bm{D}\bm{\Gamma}$ is a signal contaminated with noise
$\bm{E}$ and $\supp{}{\bm{\Gamma}}=\Lambda$. If $\bm{Y}_{p,\Lambda}$ is the projection of
$\bm{Y}$ to the subspace spanned by $\bm{D}_{\Lambda}$ and
\[
\left\Vert \bm{\Gamma}\right\Vert _{0,st}\leq\frac{\widehat{\lambda}}%
{\widehat{\lambda}+1}\left(  1+\frac{1}{\mu\left(  \bm{D}\right)  }\right)
\]
then we have
\[
\left\Vert \bm{D}^{\ast}\left(  \bm{Y}-\bm{Y}_{p,\Lambda}\right)  \right\Vert _{\infty}%
\leq\left(  1+\widehat{\lambda}\right)  \left\Vert \bm{E}\right\Vert _{L}%
\]

\end{lemma}

\begin{proof}
The proof is the same as in the original lemma except in the estimates we have
to use $\left\Vert \bm{\Gamma}\right\Vert _{0,st}$ instead of $\left\Vert
\bm{\Gamma}\right\Vert _{0,\infty}$. As in the previous proof
\begin{align*}
\left\Vert \bm{D}^{\ast}\left(  \bm{Y}-\bm{Y}_{p,\Lambda}\right)  \right\Vert _{\infty} &
=\left\Vert \bm{D}_{co\Lambda}^{\ast}\left(  \bm{I}-\bm{P}_{\Lambda}\right)  \bm{E}\right\Vert
_{\infty}=\\
\max_{\omega\in co\Lambda}\left\vert \bm{d}_{\omega}^{\ast}\left(  \bm{I}-\bm{P}_{\Lambda
}\right)  \bm{E}\right\vert  &  \leq\max_{\omega\in co\Lambda}\left\vert \bm{d}_{\omega
}^{\ast}\bm{E}\right\vert +\left\vert \bm{d}_{\omega}^{\ast}\bm{P}_{\Lambda}\bm{E}\right\vert .
\end{align*}
We begin with the estimation of the first part $\max_{\omega\in co\Lambda}\left\vert \bm{d}_{\omega}^{\ast
}\bm{E}\right\vert =  \max_{\omega\in co\Lambda}\left\vert \left\langle \bm{d}_{\omega
},\bm{E}_{\supp{}{\bm{d}_{\omega}}}\right\rangle \right\vert \leq  \max_{\omega\in
co\Lambda}\left\Vert \bm{d}_{\omega}\right\Vert _{2}\left\Vert \bm{E}_{\supp{}{
\bm{d}_{\omega}}}\right\Vert _{2}=\left\Vert \bm{E}\right\Vert _{L}$ as $\bm{D}$ is
normalized.\newline For the second one $\max_{\omega\in co\Lambda}\left\vert
\bm{d}_{\omega}^{\ast}\bm{P}_{\Lambda}\bm{E}\right\vert =\max_{\omega\in co\Lambda
}\left\vert \bm{d}_{\omega}^{\ast}\bm{D}_{\Lambda}\left(  \bm{D}_{\Lambda}^{\ast}\bm{D}_{\Lambda
}\right)  ^{-1}\bm{D}_{\Lambda}^{\ast}\bm{E}\right\vert \leq\max_{\omega\in
co\Lambda}\left\Vert \bm{d}_{\omega}^{\ast}\bm{D}_{\Lambda}\left(  \bm{D}_{\Lambda}^{\ast
}\bm{D}_{\Lambda}\right)  ^{-1}\right\Vert _{1}\left\Vert \bm{D}_{\Lambda}^{\ast
}\bm{E}\right\Vert _{\infty}$. As before $\left\Vert \bm{D}_{\Lambda}^{\ast}\bm{E}\right\Vert
_{\infty}=\max_{\theta\in\Lambda}\left\vert \bm{d}_{\omega}^{\ast}\bm{E}\right\vert
\leq\left\Vert \bm{E}\right\Vert _{L}$ and $\max_{\omega\in co\Lambda}\left\Vert
\bm{d}_{\omega}^{\ast}\bm{D}_{\Lambda}\left(  \bm{D}_{\Lambda}^{\ast}\bm{D}_{\Lambda}\right)
^{-1}\right\Vert _{1}\leq\max_{\omega\in co\Lambda}\left\Vert \bm{d}_{\omega
}^{\ast}\bm{D}_{\Lambda}\right\Vert _{1}\left\Vert \left(  \bm{D}_{\Lambda}^{\ast
}\bm{D}_{\Lambda}\right)  ^{-1}\right\Vert _{1}.$ For $\omega\notin\Lambda
,~\left\Vert \bm{d}_{\omega}^{\ast}\bm{D}_{\Lambda}\right\Vert _{1}\leq$ $\left\Vert
\bm{\Gamma}\right\Vert _{0,st}\mu\left(  \bm{D}\right)  $ and for $\left\Vert \left(
\bm{D}_{\Lambda}^{\ast}\bm{D}_{\Lambda}\right)  ^{-1}\right\Vert _{1}\leq$ $\frac
{1}{1-\left(  \left\Vert \bm{\Gamma}\right\Vert _{0,st}-1\right)  \mu\left(
\bm{D}\right)  }$ holds as we saw in Eq.~\eqref{E4}. So we have%
\[
\max_{\omega\in co\Lambda}\left\vert \bm{d}_{\omega}^{\ast}\bm{P}_{\Lambda}\bm{E}\right\vert
\leq\frac{\left\Vert \bm{\Gamma}\right\Vert _{0,st}\mu\left(  \bm{D}\right)  }{1-\left(
\left\Vert \bm{\Gamma}\right\Vert _{0,st}-1\right)  \mu\left(  \bm{D}\right)
}\left\Vert \bm{E}\right\Vert _{L}.
\]
Since the condition $\frac{\left\Vert \bm{\Gamma}\right\Vert _{0,st}\mu\left(
\bm{D}\right)  }{1-\left(  \left\Vert \bm{\Gamma}\right\Vert _{0,st}-1\right)
\mu\left(  \bm{D}\right)  }<\widehat{\lambda}$ is equivalent to $\left\Vert
\bm{\Gamma}\right\Vert _{0,st}<\frac{\widehat{\lambda}}{\widehat{\lambda}+1}\left(
1+\frac{1}{\mu\left(  \bm{D}\right)  }\right)  $ by the proof of proposition
\ref{P1}, we conclude that $\left\Vert \bm{D}^{\ast}\left(  \bm{Y}-\bm{Y}_{p,\Lambda}\right)
\right\Vert _{\infty}\leq\left(  1+\widehat{\lambda}\right)  \left\Vert
\bm{E}\right\Vert _{L}$.
\end{proof}

Note that in the above lemma we can use $  \bm{\chi}_{\bm{\Gamma},\mathcal{G}}$ instead of $\bm{\Gamma}$ but then also $\Lambda$ must be changed to 
\[
\Lambda_{\ast}\overset{def}{=}\supp{}{\bm{\chi}_{\bm{\Gamma},\mathcal{G}}}.
\]
It's important to remember that if we do not use the $\ell_{2}$ norm in the $\eqref{GBP}$ problem, then $\Lambda_{\ast}=\supp{}{\bm{\Gamma}}$. Also note that $\widehat{\lambda}\geq0$, i.e. we do not need a restriction like
$\widehat{\lambda}\in\left[  0,1\right]  .$

Now we are able to prove theorem \ref{T2}.

\begin{proof}
[Proof of theorem \ref{T2}]The proof goes along the lines of the proof of
Theorem 6 in \cite{papyan2017working} with a slight modification. We want to use theorem
\ref{T1} for $\bm{Y}$. By our assumption and
proposition \ref{P1} we have that $\gamma_{\max} ERC_{\lambda \frac{\gamma_{\min}}{\gamma_{\max}} }\left(  \Lambda_{\ast}\right)
>0$ which is the first condition in theorem \ref{T1}. For the second one we
need $\left\Vert \bm{D}^{\ast}\left(  \bm{Y}-\bm{Y}_{p,\Lambda_{\ast}}\right)  \right\Vert
_{\infty}\leq\gamma_{\max} ERC_{\lambda \frac{\gamma_{\min}}{\gamma_{\max}} }\left(  \Lambda_{\ast}\right) .$ Let us use $\theta\overset{def}{=}\frac{\lambda\gamma_{\min}}{\gamma_{\max}}$ to be more compact. To use lemma \ref{L2} we need to calculate $\widehat{\lambda}$. If we solve
$c\frac{\theta}{1+\theta}=\frac{\widehat{\lambda}.}{1+\widehat{\lambda}.}$
we get $\widehat{\lambda}=\frac{c\theta}{1+\theta-c\theta}.$ So for this
$\widehat{\lambda}$ lemma \ref{L2} gives $\left\Vert \bm{D}^{\ast}\left(
\bm{Y}-\bm{Y}_{p,\Lambda_{\ast}}\right)  \right\Vert _{\infty}\leq\left(
1+\widehat{\lambda}\right)  \left\Vert \bm{E}\right\Vert _{L}$ thus if $\left(
1+\widehat{\lambda}\right)  \left\Vert \bm{E}\right\Vert _{L}\leq\gamma_{\max}
ERC_{\theta}\left(  \Lambda_{\ast}\right)  $ then the second condition holds
as well. Recall that by Eq.~\eqref{E4} we have $ERC_{\theta}\left(
\Lambda_{\ast}\right)  =\theta-\max_{\omega\notin\Lambda_{\ast}}\left\Vert
\bm{D}_{\Lambda_{\ast}}^{\dagger}\bm{d}_{\omega}\right\Vert _{1}\geq\theta
-\frac{\left\Vert \bm{\chi}_{\bm{\Gamma},\mathcal{G}}\right\Vert _{0,st}\mu\left(  \bm{D}\right)
}{1-\left(  \left\Vert \bm{\chi}_{\bm{\Gamma},\mathcal{G}}\right\Vert _{0,st}-1\right)  \mu\left(
\bm{D}\right)  }.$ Using our assumption%
\begin{align*}
& \theta-\frac{\left\Vert \bm{\chi}_{\bm{\Gamma},\mathcal{G}}\right\Vert _{0,st}\mu\left(  \bm{D}\right)
}{1-\left(  \left\Vert \bm{\chi}_{\bm{\Gamma},\mathcal{G}}\right\Vert _{0,st}-1\right)  \mu\left(
\bm{D}\right)  }\geq\\
& \theta-\frac{c\frac{\theta}{1+\theta}\left(  1+\frac{1}%
{\mu\left(  \bm{D}\right)  }\right)  \mu\left(  \bm{D}\right)  }{1-\left(
c\frac{\theta}{1+\theta}\left(  1+\frac{1}{\mu\left(  \bm{D}\right)  }\right)
-1\right)  \mu\left(  \bm{D}\right)  }.
\end{align*}
Thus $\left(  1+\widehat{\lambda}\right)  \left\Vert \bm{E}\right\Vert _{L}%
\leq\gamma_{\max} ERC_{\theta}\left(  \Lambda_{\ast}\right)  $ is satisfied if
\begin{align*}
& \left(  1+\frac{c\theta}{1+\theta-c\theta}\right)  \frac{\left\Vert
\bm{E}\right\Vert _{L}}{\gamma_{\max}}\leq\\
& \theta-\frac{c\frac{\theta}{1+\theta}\left(
1+\frac{1}{\mu\left(  \bm{D}\right)  }\right)  \mu\left(  \bm{D}\right)  }{1-\left(
c\frac{\theta}{1+\theta}\left(  1+\frac{1}{\mu\left(  \bm{D}\right)  }\right)
-1\right)  \mu\left(  \bm{D}\right)  }%
\end{align*}
which we solve for $\gamma$ to get
\[
\frac{1}{\theta\left(  1-c\right)  }\left\Vert \bm{E}\right\Vert _{L}\leq\gamma_{\max}
\]
which is equivalent to $\frac{1}{\lambda\left(  1-c\right)  }\left\Vert \bm{E}\right\Vert
_{L}\leq\gamma_{\min}$. Since this bound holds theorem \ref{T1} yields that the minimizer $\bm{\Gamma}
_{BP}$ of Eq.~\eqref{GBP} is the unique minimizer, $\supp{}{\bm{\Gamma}
_{BP}}\subseteq\Lambda_{\ast}.$ These prove 1) and 2). Theorem \ref{T1} also
yields $\left\Vert \bm{\Gamma}_{BP}-\bm{C}_{\Lambda_{\ast}}\right\Vert _{\infty}%
\leq\gamma_{\max}\left\Vert \left(  \bm{D}_{\Lambda_{\ast}}^{\ast}\bm{D}_{\Lambda_{\ast}%
}\right)  ^{-1}\right\Vert _{\infty,\infty}$ and $\supp{}{\bm{\Gamma}_{BP}}$
contains every index $\omega\in\Lambda_{\ast}$ for which $\left\vert
\bm{C}_{\Lambda\ast}\right\vert >\gamma_{\max}\left\Vert \left(  \bm{D}_{\Lambda_{\ast}}^{\ast
}\bm{D}_{\Lambda_{\ast}}\right)  ^{-1}\right\Vert _{\infty,\infty}$ where
$\bm{Y}_{p,\Lambda_{\ast}}=\bm{D}_{\Lambda_{\ast}}\bm{C}_{\Lambda_{\ast}}$.

To estimate $\left\Vert \bm{\Gamma}_{BP}-\bm{\Gamma}\right\Vert _{\infty}$we need an
estimate on $\left\Vert \bm{C}_{\Lambda_{\ast}}-\bm{\Gamma}\right\Vert _{\infty}$ and
use the triangle inequality. Since $\supp{}{\bm{\Gamma}}\subseteq\Lambda_{\ast
}$ and $\bm{X}=\bm{D}\bm{\Gamma}$ we have $\bm{\Gamma}=\left(  \bm{D}_{\Lambda_{\ast}}^{\ast}%
\bm{D}_{\Lambda_{\ast}}\right)  ^{-1}\bm{D}_{\Lambda_{\ast}}^{\ast}\bm{X}$ moreover
$\bm{Y}_{p,\Lambda_{\ast}}=\bm{D}_{\Lambda_{\ast}}\bm{C}_{\Lambda_{\ast}}$ gives $\left(
\bm{D}_{\Lambda_{\ast}}^{\ast}\bm{D}_{\Lambda_{\ast}}\right)  ^{-1}\bm{D}_{\Lambda_{\ast}%
}^{\ast}\bm{Y}_{p,\Lambda_{\ast}}=\bm{C}_{\Lambda_{\ast}}$. Since $\bm{Y}-\bm{Y}_{p,\Lambda_{\ast
}}$ is orthogonal to the linear space spanned by the columns of $\bm{D}_{\Lambda
_{\ast}}$ we have $\bm{D}_{\Lambda_{\ast}}^{\ast}\left(  \bm{Y}-\bm{Y}_{p,\Lambda_{\ast}%
}\right)  $ which gives $\left(  \bm{D}_{\Lambda_{\ast}}^{\ast}\bm{D}_{\Lambda_{\ast}%
}\right)  ^{-1}\bm{D}_{\Lambda_{\ast}}^{\ast}\bm{Y}=\left(  \bm{D}_{\Lambda_{\ast}}^{\ast
}\bm{D}_{\Lambda_{\ast}}\right)  ^{-1}\bm{D}_{\Lambda_{\ast}}^{\ast}\left(
\bm{Y}-\bm{Y}_{p,\Lambda_{\ast}}+\bm{Y}_{p,\Lambda_{\ast}}\right)  =\left(  \bm{D}_{\Lambda_{\ast
}}^{\ast}\bm{D}_{\Lambda_{\ast}}\right)  ^{-1}\bm{D}_{\Lambda_{\ast}}^{\ast}%
\bm{Y}_{p,\Lambda_{\ast}}$. Now we can proceed:
\begin{align*}
& \left\Vert \bm{C}_{\Lambda_{\ast}}-\bm{\Gamma}\right\Vert _{\infty}  =\left\Vert
\left(  \bm{D}_{\Lambda_{\ast}}^{\ast}\bm{D}_{\Lambda_{\ast}}\right)  ^{-1}%
\bm{D}_{\Lambda_{\ast}}^{\ast}\left(  \bm{Y}-\bm{X}\right)  \right\Vert _{\infty}\leq\\
& \left\Vert \left(  \bm{D}_{\Lambda_{\ast}}^{\ast}\bm{D}_{\Lambda_{\ast}}\right)
^{-1}\right\Vert _{\infty}\left\Vert \bm{D}_{\Lambda_{\ast}}^{\ast}\bm{E}\right\Vert
_{\infty} \leq\left\Vert \left(  \bm{D}_{\Lambda_{\ast}}^{\ast}\bm{D}_{\Lambda_{\ast
}}\right)  ^{-1}\right\Vert _{\infty}\left\Vert \bm{E}\right\Vert _{L}%
\end{align*}
where $\left\Vert \bm{D}_{\Lambda_{\ast}}^{\ast}\bm{E}\right\Vert _{\infty}%
\leq\left\Vert \bm{E}\right\Vert _{L}$ was proved in the proof of lemma \ref{L2}.
So we have%
\begin{align*}
& \left\Vert \bm{\Gamma}_{BP}-\bm{\Gamma}\right\Vert _{\infty}  \leq\left\Vert
\bm{\Gamma}_{BP}-\bm{C}_{\Lambda_{\ast}}\right\Vert _{\infty}+\left\Vert \bm{C}_{\Lambda
_{\ast}}-\bm{\Gamma}\right\Vert _{\infty}\leq\\
&  \gamma_{\max}\left\Vert \left(  \bm{D}_{\Lambda_{\ast}}^{\ast}\bm{D}_{\Lambda_{\ast}%
}\right)  ^{-1}\right\Vert _{\infty}+\left\Vert \left(  \bm{D}_{\Lambda_{\ast}%
}^{\ast}\bm{D}_{\Lambda_{\ast}}\right)  ^{-1}\right\Vert _{\infty}\left\Vert
\bm{E}\right\Vert _{L}%
\end{align*}
Now if we set $\gamma_{\max}=\frac{1}{\theta\left(  1-c\right)  }\left\Vert
\bm{E}\right\Vert _{L}$ then%
\[
\left\Vert \bm{\Gamma}_{BP}-\bm{\Gamma}\right\Vert _{\infty}\leq\frac{1+\theta-\theta
c}{\theta\left(  1-c\right)  }\left\Vert \left(  \bm{D}_{\Lambda_{\ast}}^{\ast
}\bm{D}_{\Lambda_{\ast}}\right)  ^{-1}\right\Vert _{\infty}\left\Vert \bm{E}\right\Vert
_{L}%
\]

As $\left\Vert \left(  \bm{D}_{\Lambda_{\ast}}^{\ast}\bm{D}_{\Lambda_{\ast}}\right)
^{-1}\right\Vert _{\infty}=\left\Vert \left(  \bm{D}_{\Lambda_{\ast}}^{\ast
}\bm{D}_{\Lambda_{\ast}}\right)  ^{-1}\right\Vert _{1}$ we can use equation
Eq.~\eqref{E5}%
\begin{align}
& \left\Vert \left(  \bm{D}_{\Lambda_{\ast}}^{\ast}\bm{D}_{\Lambda_{\ast}}\right)
^{-1}\right\Vert _{\infty}  \leq\frac{1}{1-\left(  \left\Vert \bm{\chi}_{\bm{\Gamma},\mathcal{G}}\right\Vert _{0,st}-1\right)  \mu\left(  \bm{D}\right)  }\leq\label{E6}\\
& \frac{1}{1-\left(  c\frac{\theta}{1+\theta}\left(  1+\frac{1}{\mu\left(
\bm{D}\right)  }\right)  -1\right)  \mu\left(  \bm{D}\right)  }  =\\
& \frac{1+\theta
}{\left(  1+\mu\left(  \bm{D}\right)  \right)  \left(  1+\theta-c\theta\right)
}.\nonumber
\end{align}
With the above inequality we get%
\[
\left\Vert \bm{\Gamma}_{BP}-\bm{\Gamma}\right\Vert _{\infty}\leq\frac{1+\theta}{\left(
1+\mu\left(  \bm{D}\right)  \right)  \theta\left(  1-c\right)  }\left\Vert
\bm{E}\right\Vert _{L}%
\]
which is 3) and which implies 4).
\end{proof}

\subsection{Multi Layered extension}\label{SSmulti}

In the original theorem \cite{papyan2016working} $\lambda=1$ as the $\ell_{1}$ norm was used
and $c=\frac{2}{3}$ was chosen. For these we get back $\left\Vert \bm{\chi}_{\bm{\Gamma},\mathcal{G}}\right\Vert _{0,st}<\frac{1}{3}\left(  1+\frac{1}{\mu\left(  \bm{D}\right)
}\right)  $ but get $3\left\Vert \bm{E}\right\Vert _{L}=\gamma$ instead of
$4\left\Vert \bm{E}\right\Vert _{L}$. The weaker constant in 3) and 4) is
$6\left\Vert \bm{E}\right\Vert _{L}$ instead of $7,5\left\Vert \bm{E}\right\Vert _{L}$.

Now we proceed with the extension of Theorem 12 in \cite{papyan2017convolutional} to the layered case and prove theorem \ref{T3} to show how the result extends to the layered problem, and to show how the error accumulation can occur in this case. 

Given $\bm{X}=\bm{D}\bm{\Gamma}$,
we may assume that $\bm{\Gamma}$ can be further decomposed in a way similar to $\bm{X}$:
\begin{align}
\bm{X} &=\bm{D}_{1}\bm{\Gamma}_{1},\nonumber\\
\bm{\Gamma}_{1} &=\bm{D}_{2}\bm{\Gamma}_{2},\nonumber\\
&\vdots\nonumber\\
\bm{\Gamma}_{K-1} &=\bm{D}_{K}\bm{\Gamma}_{K}.\label{La1}
\end{align}
The layered problem then tries to recover $\bm{\Gamma}_1,\dots,\bm{\Gamma}_K$.

\begin{definition}
The Layered Group Basis Pursuit (LGBP) first solves the Sparse Coding problem $\bm{X}=\bm{D}_{1}\bm{\Gamma}^{1}$ via Eq.~\eqref{GBP} with parameters
$\bm{\gamma}^{1}$, group-partitioning $\mathcal{G}^1\overset{def}{=}\left\{\mathcal{G}^1_{i}\right\}$ and norms $\bm{l}^1$ corresponding to the groups, obtaining $\hat{\bm{\Gamma}}_{1}$ as a solution.
Next, it solves another Sparse Coding problem $\hat{\bm{\Gamma}}_{1}=\bm{D}_{2}\bm{\Gamma}_{2}$ again by Eq.~\eqref{GBP}  with parameter
$\bm{\gamma}^{2}$, group-partitioning $\left\{\mathcal{G}^2_{i}\right\}$ and norms $\bm{l}^2$ corresponding to the groups, denoting the result by $\hat{\bm{\Gamma}}_{2}$ and so on. The final vector $\hat{\bm{\Gamma}}_{K}$ is the solution of (LGBP).
The vector $\bm{\gamma}^{LGBP}$ contains all the weights of $\bm{\gamma}^i$ in Eq.~\eqref{GBP} for each layer $i$. Moreover, on each layer we have a $\lambda_i$ as in Theorem \ref{T2} and let the vector $\bm{\lambda}^{LGBP}$ contain these constants. Let $\theta_{i}=\frac{\lambda_{i}\gamma_{\min}}{\gamma_{\max}}$ and the vector $\bm{\theta}^{LGBP}$ collect these. Let $\bm{\chi}_{\bm{\Gamma}_{i},\mathcal{G}^{i}}$ denote the 2-norm group characteristic vector for the $i$-th layer of the perfect solution and $\gamma^{i}_{\max},\,\gamma^{i}_{\max}$ the min and max value corresponding to $\bm{\gamma}^{i}$.
\end{definition}

Using Theorem \ref{T2} and the above definition we have the following:

\begin{theorem}
\label{T3}Suppose the clean signal $\bm{X}$ has a layered decomposition as in ~\eqref{La1} and $\bm{Y}=\bm{X}+\bm{E}$ is a noisy signal.
Let $\left\{
\hat{\bm{\Gamma}}_{i}\right\}  _{i=1}^{K}$ be the set of solutions of (LGBP) with parameters $\bm{\gamma}^{LGBP},~\bm{\lambda}^{LGBP}$.
Assume that the following hold:
\begin{enumerate}
\item[a)] $\left\Vert \bm{\chi}_{\bm{\Gamma}_{i},\mathcal{G}^{i}}\right\Vert _{0,st,\bm{D}_{i}}\leq c_{i}\frac{\theta_{i}}%
{1+\theta_{i}}\left(  1+\frac{1}{\mu\left(  \bm{D}_{i}\right)  }\right)  $,
\end{enumerate}
where $\left\Vert \bm{\chi}_{\bm{\Gamma}_{i},\mathcal{G}^{i}}\right\Vert _{0,st,\bm{D}_{i}}$ is calculated with respect to the dictionary
$\bm{D}_{i}$ and $c_{i}\in\left(0,1\right)$ is a suitable constant; and
\begin{enumerate}
\item[b)] $\gamma^{i}_{\min}=$ $\frac{1}{\lambda_{i}\left(
1-c_{i}\right)  }\epsilon_{i-1}$,
\end{enumerate} 
with $\epsilon_{0}=\left\Vert
\bm{E}\right\Vert _{L,\bm{D}_{1}}$, 
$\left\Vert \bm{\chi}_{\bm{\Gamma}_{K},\mathcal{G}^{K}}\right\Vert _{0,\bm{D}_{K+1}}$
$\overset{def}{=}\left\Vert \bm{\chi}_{\bm{\Gamma}_{K},\mathcal{G}^{K}}\right\Vert
_{0}$ and
\begin{equation}
\text{{\scriptsize
$\epsilon_{i}=
\epsilon_{0}
\prod_{j=1}^{i}\Bigl( \sqrt{\left\Vert \bm{\chi}_{\bm{\Gamma}_{j},\mathcal{G}^{j}}\right\Vert _{0,\bm{D}_{j+1}}}\frac{1+\theta_{j}}{\left(  1+\mu\left(
\bm{D}_{j}\right)  \right)  \theta_{j}\left(  1-c_{j}\right)  }\Bigr)$.
}}
\label{lgbp_err}
\end{equation}
Then
\begin{enumerate}
\item[1)] $\supp{}{\hat{\bm{\Gamma}}_{i}}\subseteq\supp{}{\bm{\chi}_{\bm{\Gamma}_{i},\mathcal{G}^{i}}}$, 
\item[2)] $\left\Vert \bm{\Gamma}
_{i}-\hat{\bm{\Gamma}}_{i}\right\Vert _{L,\bm{D}_{i}}\leq\epsilon_{i}$,
\item[3)]
{\small{$\left\{j~\Big|~| \left(  \bm{\Gamma}_{i}\right)_{j}|
>\frac{1+\theta_{i}}{\left(
1+\mu\left(  \bm{D}_{i}\right)  \right)  \theta_{i}\left(  1-c_{i}\right)
}\epsilon_{i-1}\right\}\subseteq\supp{}{\hat{\bm{\Gamma}}_{i}}$}},

\item[4)] the solution $\hat{\bm{\Gamma}}_{i}$ is the unique
solution of Eq.~\eqref{GBP} with the corresponding norm and parameter.
\end{enumerate}

\end{theorem}

\begin{proof}
The proof is the same as that of Theorem 12 in
\cite{papyan2017convolutional} but with the modified tools. For the problem $\bm{Y}=\bm{X}+\bm{E}$ with
$\bm{X}=\bm{D}_{1}\bm{\Gamma}_{1}$ the assumptions of theorem \ref{T2} hold as $\gamma^{1}_{\min}=\frac{1}{\lambda_{1}\left(  1-c_{1}\right)  }\left\Vert
\bm{E}\right\Vert _{L,\bm{D}_{1}}$ and $\left\Vert\bm{\chi}_{\bm{\Gamma}_{1},\mathcal{G}^{1}}\right\Vert _{0,st,\bm{D}_{1}}\leq c_{1}\frac{\theta_{1}}{1+\theta
_{1}}\left(  1+\frac{1}{\mu\left(  \bm{D}_{1}\right)  }\right)  $ thus the
first layer of (LGBP) yields that

A) $\supp{}{\hat{\bm{\Gamma}}_{1}}\subseteq \supp{}{\bm{\chi}_{\bm{\Gamma}_{1},\mathcal{G}^{1}}}$

B) the minimizer of Eq.~\eqref{GBP} with the corresponding norm and parameters is unique

C) $\left\Vert \hat{\bm{\Gamma}}_{1}-\bm{\Gamma}_{1}\right\Vert _{\infty}%
<\frac{1+\theta_{1}}{\left(  1+\mu\left(  \bm{D}_{1}\right)  \right)  \theta
_{1}\left(  1-c_{1}\right)  }\left\Vert \bm{E}\right\Vert _{L,\bm{D}_{1}}$

D) $\supp{}{\hat{\bm{\Gamma}}_{1}}$ contains every index $i$ for which $\left\vert
\left(  \bm{\Gamma}_{1}\right)  _{i}\right\vert >\frac{1+\theta_{1}}{\left(
1+\mu\left(  \bm{D}_{1}\right)  \right)  \theta_{1}\left(  1-c_{1}\right)
}\left\Vert \bm{E}\right\Vert _{L,\bm{D}_{1}}$

Now we estimate $\left\Vert \hat{\bm{\Gamma}}_{1}-\bm{\Gamma}_{1}\right\Vert
_{L,\bm{D}_{2}}$ with respect to $\bm{D}_{2}$ since this will yield the new error.
\begin{align*}
& \left\Vert \hat{\bm{\Gamma}}_{1}-\bm{\Gamma}_{1}\right\Vert _{L,\bm{D}_{2}}  \leq
\sqrt{\left\Vert \hat{\bm{\Gamma}}_{1}-\bm{\Gamma}_{1}\right\Vert _{0,\bm{D}_{2}}%
}\left\Vert \hat{\bm{\Gamma}}_{1}-\bm{\Gamma}_{1}\right\Vert _{\infty}\leq\\
&  \sqrt{\left\Vert \hat{\bm{\Gamma}}_{1}-\bm{\Gamma}_{1}\right\Vert _{0,\bm{D}_{2}}%
}\frac{1+\theta_{1}}{\left(  1+\mu\left(  \bm{D}_{1}\right)  \right)  \theta
_{1}\left(  1-c_{1}\right)  }\left\Vert \bm{E}\right\Vert _{L,\bm{D}_{1}}%
\end{align*}
Using A) we have that $\left\Vert \hat{\bm{\Gamma}}_{1}-\bm{\Gamma}_{1}\right\Vert
_{0,\bm{D}_{2}}\leq\left\Vert \bm{\chi}_{\bm{\Gamma}_{1},\mathcal{G}^{1}}\right\Vert
_{0,\bm{D}_{2}}$ and we get
\begin{align*}
& \left\Vert \hat{\bm{\Gamma}}_{1}-\bm{\Gamma}_{1}\right\Vert _{L,\bm{D}_{2}}\leq\\
& \sqrt{\left\Vert \bm{\chi}_{\bm{\Gamma}_{1},\mathcal{G}^{1}}\right\Vert _{0,\bm{D}_{2}}%
}\frac{1+\theta_{1}}{\left(  1+\mu\left(  \bm{D}_{1}\right)  \right)  \theta
_{1}\left(  1-c_{1}\right)  }\left\Vert \bm{E}\right\Vert _{L,\bm{D}_{1}}=\epsilon_{1}%
\end{align*}
thus we have A), B) we can rewrite the bound in C) and D) to $\frac
{\epsilon_{1}}{\sqrt{\left\Vert \bm{\chi}_{\bm{\Gamma}_{1},\mathcal{G}^{1}}\right\Vert
_{0,\bm{D}_{2}}}}$ and we proved

E) $\left\Vert \hat{\bm{\Gamma}}_{1}-\bm{\Gamma}_{1}\right\Vert _{L,\bm{D}_{2}}%
\leq\epsilon_{1}$

If $\bm{E}_{1}\overset{def}{=}\hat{\bm{\Gamma}}_{1}-\bm{\Gamma}_{1}$ then at the next
layer we have $\hat{\bm{\Gamma}}_{1}=\bm{D}_{2}\bm{\Gamma}_{2}+\bm{E}_{1}.$ Since E) holds and
$\left\Vert \bm{\chi}_{\bm{\Gamma}_{2},\mathcal{G}^{2}}\right\Vert _{0,st,\bm{D}_{2}}\leq
c_{2}\frac{\theta_{2}}{1+\theta_{2}}\left(  1+\frac{1}%
{\mu\left(  \bm{D}_{2}\right)  }\right)  $ was assumed we can one again use theorem
\ref{T2} and choosing $\gamma^{2}_{\min}=\frac{1}{\lambda_{2}\left(
1-c_{2}\right)  }\epsilon_{1}$and get that

A) $\supp{}{\hat{\bm{\Gamma}}_{2}}\subseteq \supp{}{\bm{\chi}_{\bm{\Gamma}_{2},\mathcal{G}^{2}}}$

B)The minimizer of Eq.~\eqref{GBP} with the corresponding norm and parameters is unique.

C) $\left\Vert \hat{\bm{\Gamma}}_{2}-\bm{\Gamma}_{2}\right\Vert _{\infty}%
<\frac{1+\theta_{2}}{\left(  1+\mu\left(  \bm{D}_{2}\right)  \right)  \theta
_{2}\left(  1-c_{2}\right)  }\epsilon_{1}$

D) $\supp{}{\hat{\bm{\Gamma}}_{2}}$ contains every index $i$ for which $\left\vert
\left(  \bm{\Gamma}_{2}\right)  _{i}\right\vert >\frac{1+\theta_{2}}{\left(
1+\mu\left(  \bm{D}_{2}\right)  \right)  \theta_{2}\left(  1-c_{2}\right)
}\epsilon_{1}$

As before we can prove \
\begin{align*}
& \left\Vert \hat{\bm{\Gamma}}_{2}-\bm{\Gamma}_{2}\right\Vert _{L,\bm{D}_{3}}\leq\\
& \sqrt{\left\Vert\bm{\chi}_{\bm{\Gamma}_{2},\mathcal{G}^{2}}\right\Vert _{0,\bm{D}_{3}}%
}\frac{1+\theta_{2}}{\left(  1+\mu\left(  \bm{D}_{2}\right)  \right)  \theta
_{2}\left(  1-c_{2}\right)  }\epsilon_{1}=\epsilon_{2}%
\end{align*}
and proceed with the proof. Thus in general $\epsilon_{i}=\left\Vert
\bm{E}\right\Vert _{L,\bm{D}_{1}}\Pi_{j=1}^{i}\left(  \sqrt{\left\Vert \bm{\chi}_{\bm{\Gamma}_{j},\mathcal{G}^{j}}\right\Vert _{0,\bm{D}_{j+1}}}\frac{1+\theta_{j}
}{\left(  1+\mu\left(  \bm{D}_{j}\right)  \right)  \theta_{j}\left(
1-c_{j}\right)  }\right)  $
\end{proof}

\bigskip As we mentioned at theorem \ref{T2} instead of $\frac{1+\theta
}{\left(  1+\mu\left(  \bm{D}\right)  \right)  \theta\left(  1-c\right)  }$ we
could use the bound $\frac{1+\theta}{\theta\left(  1-c\right)  }$ in this
case
\[
\epsilon_{i}=\left\Vert \bm{E}\right\Vert _{L,\bm{D}_{1}}\Pi_{j=1}^{i}\left(
\sqrt{\left\Vert \bm{\chi}_{\bm{\Gamma}_{j},\mathcal{G}^{j}}\right\Vert _{0,\bm{D}_{j+1}}%
}\frac{1+\theta_{j}}{\theta_{j}\left(  1-c_{j}\right)  }\right)
\]
and also in 3) the bound would be $\frac{1+\theta_{i}}{\theta_{i}\left(
1-c_{i}\right)  }\epsilon_{i-1}$. Moreover if only the $\ell_{1}$, $\ell_{2}$
norms were chosen with $c_{i}=\frac{2}{3}$ we would get%
\[
\epsilon_{i}=\left\Vert \bm{E}\right\Vert _{L,\bm{D}_{1}}6^{i}\Pi_{j=1}^{i}%
\sqrt{\left\Vert \bm{\chi}_{\bm{\Gamma}_{j},\mathcal{G}^{j}}\right\Vert _{0,\bm{D}_{j+1}}}%
\]
and in b) $\gamma^{i}_{\min}=$ $3\epsilon_{i-1}.$

\subsection{Rewriting Multi layer to a Single layer}\label{SSSingle}

Of course the single layer case is a special cases of this theorem, but the "reverse" is also true which we will see next.

As the layered methods suffers from error accumulation,
\cite{cazenavette2020architectural} proposed to rewrite the layered problem $\ell_{1}$ regularized problem into a single layer global \eqref{BP}-like minimization scheme as

\begin{equation}
{\tiny
\left[
\begin{array}
[c]{c}%
\bm{X}\\
0\\
\vdots\\
0
\end{array}
\right]  -\left[
\begin{array}
[c]{cccc}%
\bm{D}_{1} & 0 & \cdots & 0\\
-\bm{I} & \bm{D}_{2} & \ddots & \vdots\\
& \ddots & \ddots & 0\\
0 &  & -\bm{I} & \bm{D}_{K}%
\end{array}
\right]  \left[
\begin{array}
[c]{c}%
\bm{\Gamma}_{1}\\
\bm{\Gamma}_{2}\\
\vdots\\
\bm{\Gamma}_{K}%
\end{array}
\right]  =\left[
\begin{array}
[c]{c}%
0\\
\vdots\\
\vdots\\
0
\end{array}
\right].%
}%
\end{equation}
The problem is that it is not in the form of a proper \eqref{BP} as the matrix is not normalized and we do not have a global parameter $\gamma$ even if we chose the $\ell_1$ regularisation term on every layer with the same weight. So the classical stability results can not be applied to the rewritten minimization problem:

{\tiny
\begin{align*}
& \argmin{\left\{  \overline{\Gamma}_{j}\right\}  }{}\frac{1}{2}\left\Vert
\left[
\begin{array}
[c]{c}%
X\\
0\\
\vdots\\
0
\end{array}
\right]  -\left[
\begin{array}
[c]{cccc}%
D_{1} & 0 & \cdots & 0\\
-\bm{I} & D_{2} & \ddots & \vdots\\
& \ddots & \ddots & 0\\
0 &  & -\bm{I} & D_{K}%
\end{array}
\right]  \left[
\begin{array}
[c]{c}%
\bar{\Gamma}_{1}\\
\bar{\Gamma}_{2}\\
\vdots\\
\bar{\Gamma}_{K}%
\end{array}
\right]  \right\Vert _{2}^{2}+\nonumber\\
& \sum_{j=1}^{K}\gamma_{j}\left\Vert 
\overline{\Gamma}_{j}  \right\Vert _{1}\\
\end{align*}
}%
but using a renormalization trick it becomes a proper \eqref{GBP} problem:

{\tiny
\begin{align}
& \argmin{\left\{  \overline{\Gamma}_{j}\right\}  }{}\frac{1}{2}\left\Vert
\left[
\begin{array}
[c]{c}%
X\\
0\\
\vdots\\
0
\end{array}
\right]  -\left[
\begin{array}
[c]{cccc}%
\frac{1}{\sqrt{2}}D_{1} & 0 & \cdots & 0\\
-\frac{1}{\sqrt{2}}\bm{I} & \frac{1}{\sqrt{2}}D_{2} & \ddots & \vdots\\
& \ddots & \ddots & 0\\
0 &  & -\frac{1}{\sqrt{2}}\bm{I} & D_{K}%
\end{array}
\right]  \left[
\begin{array}
[c]{c}%
\overline{\Gamma}_{1}\sqrt{2}\\
\overline{\Gamma}_{2}\sqrt{2}\\
\vdots\\
\overline{\Gamma}_{K}%
\end{array}
\right]  \right\Vert _{2}^{2}\nonumber\\
& +\sum_{j=1}^{K-1}\frac{1}{\sqrt{2}}\gamma
_{j}\left\Vert \sqrt{2} \overline{\Gamma}_{j}  \right\Vert
_{1}+\gamma_{K}\left\Vert  \overline{\Gamma}_{K}  \right\Vert
_{1}\nonumber
\end{align}
}%

If on the index set of $\widetilde{\bm{\Gamma}}\overset{def}{=} \left(  \widetilde{\bm{\Gamma}}_{1},\dots
,\widetilde{\bm{\Gamma}}_{K}\right)  $ we take as partitions the groups of the blocks, since the modified matrix is normalized, the last minimization is a simple Eq.~\eqref{GBP}. Let $\bm{D}_{mod}$ denote the modified matrix $\bm{X}_{mod}\overset{def}{=}\left(\bm{X}^{T},\bm{0},\dots,\bm{0}\right)^{T}$,  $\bm{X}_{mod}=\bm{D}_{mod}\bm{\Gamma}_{mod}$ be a solution and $\bm{Y}=\bm{X}_{mod}+\bm{E}_{mod}$ be a noisy signal. If $\bm{D}_{mod},\,\bm{\Gamma}_{mod},\,\bm{E}_{mod}$ fulfill the conditions of Theorem \ref{T2} then we can apply the theorem to get a stability result. If $\widetilde{\bm{\Gamma}}_{GBP}$ is the solution of the rewritten-renormalized \eqref{GBP}, then the solution for the original rewritten problem is $\left(  \frac{1}{\sqrt{2}}\widetilde{\bm{\Gamma}}_{GBP,1},\dots
,\frac{1}{\sqrt{2}}\widetilde{\bm{\Gamma}}_{GBP,K-1},\widetilde{\bm{\Gamma}}_{GBP,K}\right)$. Moreover $\mu\left(\bm{D}_{mod}\right)=\max\{\frac{1}{2}\mu\left(\bm{D}_{1}\right),\linebreak[1]\dots,\frac{1}{2}\mu\left(\bm{D}_{K-1}\right),\mu\left(\bm{D}_{K}\right),\frac{1}{2}\left\Vert\bm{D}_{2}\right\Vert_{\max},\dots,\frac{1}{2}\left\Vert\bm{D}_{K}\right\Vert_{\max}\}$ indicates that $\bm{D}_{2},\dots,\bm{D}_{K}$ should have unit vectors with small coordinates but this works against the stripe norm. So the renormalization comes with a cost!

This renormalization trick works for a "full matrix" i.e. if:
{\tiny
\begin{align*}
& \argmin{ \overline{\Gamma}  }{}\frac{1}{2}\left\Vert
\left[
\begin{array}
[c]{c}%
\bm{X}_1\\
\bm{X}_2\\
\vdots\\
\bm{X}_K
\end{array}
\right]  -\left[
\begin{array}
[c]{cccc}%
\bm{D}_{1} & \bm{F}_{12} & \cdots & \bm{F}_{1,K}\\
 \bm{B}_{2,1}& \bm{D}_{2} & \ddots & \vdots\\
& \ddots & \ddots & \bm{F}_{K-1,K}\\
\bm{B}_{K,1} &  & \bm{B}_{K,K-1} & D_{K}%
\end{array}
\right]  \left[
\begin{array}
[c]{c}%
\bar{\Gamma}_{1}\\
\bar{\Gamma}_{2}\\
\vdots\\
\bar{\Gamma}_{K}%
\end{array}
\right]  \right\Vert _{2}^{2}+\nonumber\\
& \left< \bm{\gamma},\bm{l}\left( \bm{\Gamma} \right) \right>\\
\end{align*}
}%

where $ \overline{\Gamma} =\left( \overline{\Gamma}_{1},\dots\overline{\Gamma}_{K}\right) $ and the groups $\mathcal{G}_{i},~i\in\left\{  1,\dots,k\right\}$ can span across column blocks, but for a group $\mathcal{G}_{i}$ the corresponding atoms must have the same norm if the $\ell_{2}$ or the $\ell_{\beta,1,2}$ norm was used for that group. If the blocks $\bm{B}_{j+1,j},~j\in\left\{1,\dots,K-1\right\}$ equals $-\bm{I}$ and the blocks $\bm{B}_{p,q},~q+1\neq p,~\bm{F}_{r,s}$ are zero matrices, then the above reduces to the single layer version of the layered \eqref{GBP}, which has a different solution as the classical layered (BP) as we discussed above.

If the blocks $\bm{F}_{p,q}$ are zero blocks and we have no restriction to the rest, then it becomes a skip connection network.

\subsection{Linear classifiers}\label{SSlinclas}

So proceeding with the general (LGBP), if we turn to the linear classifiers as in \cite{romano2020adversarial} the above
theorem has several stability consequences. We can apply theorem \ref{T3} to the
margin of a single linear classifier which was defined as:

\begin{definition}\label{DefClass}
Assume that from Eq.~\eqref{ML} we get
the sparse vector $\bm{\Gamma}_{K}\in\mathbb{R}^{M_{K}}$ at the last layer and we
have linear classifiers $f_{i}\left(  \bm{\Gamma}_{K}\right)  =\bm{w}^{T}_{i}\bm{\Gamma}
_{K}+b_{i},~i\in\left\{  1,\dots,C\right\}  $ where $\bm{w}_{i}\in
\mathbb{R}^{M_{K}},b_{i}\in\mathbb{R}$. The signal $\bm{X}$ in Eq.~\eqref{ML} corresponds to class $j$ if $f_{j}\left(  \bm{\Gamma}_{K}\right)
>\max_{i\neq j}f_{i}\left(  \bm{\Gamma}_{K}\right)  $.
Let $\class{\left(  \bm{\Gamma}
_{K}\right)}$ be the function which returns this class.
Let $\mathcal{O}\left(  \bm{X}\right)  \overset{def}{=}f_{j}\left(  \bm{\Gamma}
_{K}\right)  -\max_{i\neq j}f_{i}\left(  \bm{\Gamma}_{K}\right)  $ measure the
distance from the next best class which is called the \textit{margin} for $\bm{X}$. For a single linear classifier $\mathcal{O}\left(
\bm{X}\right)  \overset{def}{=}\left\vert f\left(  \bm{\Gamma}_{K}\right)  \right\vert
.$
\end{definition}
So we have:
\begin{corollary}
If the assumptions of theorem \ref{T3} hold for $\bm{Y}=\bm{X}+\bm{E}$ and we have a single
linear classifier for which%
\begin{align*}
& \mathcal{O}\left(  \bm{X}\right)  >\\
& \left\Vert w\right\Vert _{2}\sqrt{\left\Vert
\bm{\chi}_{\bm{\Gamma}_{K},\mathcal{G}^{K}}\right\Vert _{0}}\frac{1+\theta_{K}
}{\left(  1+\mu\left(  \bm{D}_{K}\right)  \right)  \theta_{K}\left(
1-c_{K}\right)  }\epsilon_{K-1}%
\end{align*}
and $\hat{\bm{\Gamma}}_{K}$ is the solution of (LGBP) for $\bm{Y}$, then
$\class{\left(  \bm{\Gamma}_{K}\right)}  =\class{\left(  \hat{\bm{\Gamma}}_{K}\right)}$.
\end{corollary}

\begin{proof}
In the proof of theorem \ref{T3} we saw that $\left\Vert \hat{\bm{\Gamma}}_{K}
-\bm{\Gamma}_{K}\right\Vert _{\infty}<\frac{1+\theta_{K}}{\left(  1+\mu\left(
\bm{D}_{K}\right)  \right)  \theta_{K}\left(  1-c_{K}\right)  }\epsilon_{K-1}$ and
that $\supp{}{\hat{\bm{\Gamma}}_{K}}\subseteq \supp{}{\bm{\chi}_{\bm{\Gamma}_{K},\mathcal{G}^{K}}}$. As before
\begin{align*}
& \left\Vert \hat{\bm{\Gamma}}_{K}-\bm{\Gamma}_{K}\right\Vert _{2}  \leq
\sqrt{\left\Vert \bm{\chi}_{\bm{\Gamma}_{K},\mathcal{G}^{K}}\right\Vert _{0}}\left\Vert
\hat{\bm{\Gamma}}_{K}-\bm{\Gamma}_{K}\right\Vert _{\infty}\leq\\
&  \sqrt{\left\Vert \bm{\chi}_{\bm{\Gamma}_{K},\mathcal{G}^{K}}\right\Vert _{0}}%
\frac{1+\theta_{K}}{\left(  1+\mu\left(  \bm{D}_{K}\right)  \right)  \theta_
{K}\left(  1-c_{K}\right)  }\epsilon_{K-1}.
\end{align*}
Since $\left\vert f\left(  \bm{\Gamma}_{K}\right)  -f\left(  \hat{\bm{\Gamma}}_{K}
\right)  \right\vert =\left\vert  \bm{w}^{T}\left(  \bm{\Gamma}_{K}-\hat{\bm{\Gamma}}_{K}%
\right)  \right\vert \leq\left\Vert \bm{w}\right\Vert _{2}\left\Vert
\hat{\bm{\Gamma}}_{K}-\bm{\Gamma}_{K}\right\Vert _{2}\leq\left\Vert \bm{w}\right\Vert
_{2}\sqrt{\left\Vert \bm{\chi}_{\bm{\Gamma}_{K},\mathcal{G}^{K}}\right\Vert _{0}}%
\frac{1+\theta_{K}}{\left(  1+\mu\left(  \bm{D}_{K}\right)  \right)  \theta
_{K}\left(  1-c_{K}\right)  }\epsilon_{K-1}.$ By our assumption our decision
can not change.
\end{proof}
For the multi-class version let we have:
\begin{corollary}\label{Cor1}
Assume that theorem \ref{T3} holds for $\bm{Y}=\bm{X}+\bm{E}$ and $\phi=\max_{i\neq
j}\left\Vert \bm{w}_{i}-\bm{w}_{j}\right\Vert _2$ for the weights of the linear classifiers $f_{i}\left(  \bar{\bm{\Gamma}}\right)
=\bm{w}^{T}_{i}\bar{\bm{\Gamma}}+b_{i}$.
If
\[
\text{{\scriptsize
$\mathcal{O}\left(  \bm{X}\right)   > 
\phi   \sqrt{\left\Vert\bm{\chi}_{\bm{\Gamma}_{K},\mathcal{G}^{K}}
\right\Vert _{0}}\frac{1+\theta_{K}}{\left(
1+\mu\left(  \bm{D}_{K}\right)  \right)  \theta_{K}\left(  1-c_{K}\right)
}\epsilon_{K-1}%
$}}
\]
and $\hat{\bm{\Gamma}}_{K}$ is the solution of the (LGBP) for $\bm{Y}$, then
$\class{(  \bm{\Gamma}_{K})}  =\class{(  \hat{\bm{\Gamma}}_{K})}$.
\end{corollary}

\begin{proof}
As in the above proof if $\left\Vert \hat{\bm{\Gamma}}_{K}-\bm{\Gamma}
_{K}\right\Vert _{2}\leq\sqrt{\left\Vert\bm{\chi}_{\bm{\Gamma}_{K},\mathcal{G}^{K}}\right\Vert _{0}}\frac{1+\theta_{K}}{\left(  1+\mu\left(  \bm{D}_{K}\right)
\right)  \theta_{K}\left(  1-c_{K}\right)  }\epsilon_{K-1}$. If $\class{\left(
\bm{\Gamma}_{K}\right)}  =j$ then for any $i\neq j$ we have by the triangle
inequality
\begin{align*}
& \left\vert f_{j}\left(  \hat{\bm{\Gamma}}_{K}\right)  -f_{i}\left(
\hat{\bm{\Gamma}}_{K}\right)  \right\vert   =\\
& \left\vert f_{j}\left(
\hat{\bm{\Gamma}}_{K}\right)  -f_{j}\left(  \bm{\Gamma}\right)  +f_{j}\left(
\bm{\Gamma}\right)  -f_{i}\left(  \bm{\Gamma}\right)  +f_{i}\left(  \bm{\Gamma}\right)
-f_{i}\left(  \hat{\bm{\Gamma}}_{K}\right)  \right\vert \\
& = \left\vert f_{j}\left(  \bm{\Gamma}\right)  -f_{i}\left(  \bm{\Gamma}\right)
+\bm{w}^{T}_{j}\left(  \hat{\bm{\Gamma}}_{K}-\bm{\Gamma}_{K}\right)  +\bm{w}^{T}_{i}\left(  \bm{\Gamma}
_{K}-\hat{\bm{\Gamma}}_{K}\right)  \right\vert  \\
& = \left\vert f_{j}\left(
\bm{\Gamma}\right)  -f_{i}\left(  \bm{\Gamma}\right)  -\left(  \bm{w}_{j}-\bm{w}_{i}\right)^{T}
\left(  \bm{\Gamma}_{K}-\hat{\bm{\Gamma}}_{K}\right)  \right\vert \geq\\
& \left\vert f_{j}\left(  \bm{\Gamma}\right)  -f_{i}\left(  \bm{\Gamma}\right)
\right\vert -\left\vert \left(  \bm{w}_{j}-\bm{w}_{i}\right)^{T}  \left(  \bm{\Gamma}
_{K}-\hat{\bm{\Gamma}}_{K}\right)  \right\vert  \geq\\
& \mathcal{O}\left(\bm{X}\right)%
-\max_{i\neq j}\left\Vert \bm{w}_{j}-\bm{w}_{i}\right\Vert _{2}\left\Vert
\hat{\bm{\Gamma}}_{K}-\bm{\Gamma}_{K}\right\Vert _{2}\geq\mathcal{O}\left(\bm{X}\right)-\\
& \phi  \sqrt{\left\Vert \bm{\chi}_{\bm{\Gamma}_{K},\mathcal{G}^{K}}\right\Vert _{0}}\frac{1+\theta_{K}}{\left(  1+\mu\left(
\bm{D}_{K}\right)  \right)  \theta_{K}\left(  1-c_{K}\right)  }\epsilon_{K-1} 
\\
&>0.
\end{align*}
This yields that $\class{\left(  \hat{\bm{\Gamma}}_{K}\right)}  =j=\class{\left(
\bm{\Gamma}_{K}\right)}  $.
\end{proof}

\subsection{About positive coding}

At the definition of \eqref{GBP} we could use the positive coding condition to get:
\begin{equation}
\argmin{\bar{\bm{\Gamma}}\geq\bm{0}}\,{L\left(  \bar{\bm{\Gamma}}\right)} \overset{def}{=}\argmin{\bar{\bm{\Gamma}}}{\frac{1}{2}\left\Vert \bm{X}-\bm{D}\bar{\bm{\Gamma}}
\right\Vert _{2}^{2}+\left< \bm{\gamma},  \bm{l}\left( \bar{\bm{\Gamma}}\right) \right>},
\tag{PosGBP}\label{PosGBP}%
\end{equation}

where $\bar{\bm{\Gamma}}\geq\bm{0}$ means that each element of the vector $\bm{\Gamma}$ is nonnegative. 

Interestingly with some restrictions our results extend also to this case, because some of the conditions of our theorems can be very restrictive. To see this we prove the following weak stability theorem for positive coding, which is the direct extension of Theorem \ref{T1}.
\begin{theorem}
\label{Tpos1} Let $\bm{\Gamma}_{+GBP}$ be a solution of Eq.~\eqref{PosGBP} and assume that $\Lambda\overset{def}{=}\supp{}{\bm{\Gamma}_{+GBP}}$ is group-full and the columns of $\bm{D}_{\Lambda}$ are linearly
independent, moreover that $ERC_{\lambda \frac{\gamma_{\min}}{\gamma_{\max}}
}\left(  \Lambda\right)  >0$ where $\lambda\overset{def}{=}\min\left\{ 1,\beta_{1},\dots,\beta_{r}\right\}$ as before (i.e. if among the norms of $\bm{l}$ we used elastic norms, let $\left\{\beta_{1},\dots,\beta_{r}\right\}$ be the set of the parameters used in the elastic norms). 

If for the
signal $\bm{X}$ in Eq.~\eqref{PosGBP} we have%
\[
\left\Vert \bm{D}^{\ast}\left(  \bm{X}-\bm{X}_{p,\Lambda}\right)  \right\Vert _{\infty}%
\leq  \gamma_{\max} ERC_{\lambda \frac{\gamma_{\min}}{\gamma_{\max}} }\left(  \Lambda\right)
\]
then
\begin{enumerate}
\item[1)]$\supp{}\bm{\Gamma}_{+GBP}\subset\Lambda,$ (trivial)
\item[2)] If $\bm{X}_{p,\Lambda}=\bm{D}_{\Lambda}\bm{C}_{\Lambda}$ where $\bm{C}_{\Lambda}=\bm{D}_{\Lambda}^{\dagger}\bm{X}_{p,\Lambda}$ then $\left\Vert
\bm{\Gamma}_{+GBP}-\bm{C}_{\Lambda}\right\Vert _{\infty}\leq\gamma_{\max}\left\Vert \left(
\bm{D}_{\Lambda}^{\ast}\bm{D}_{\Lambda}\right)  ^{-1}\right\Vert _{\infty}$,
\item[3)] 
$\left\{i~\Big|~\left\vert \left(\bm{C}_{\Lambda}\right)_{i}\right\vert
>\gamma_{\max}\left\Vert \left(  \bm{D}_{\Lambda}%
^{\ast}\bm{D}_{\Lambda}\right)  ^{-1}\right\Vert _{\infty}\right\}\subseteq\supp{}{\bm{\Gamma}_{+GBP}}$, (trivial)
\item[4)] Moreover the minimizer $\bm{\Gamma}_{+GBP}$ is unique and equals de minimizer of \eqref{GBP}.
\end{enumerate}
\end{theorem}
\begin{proof}
We left the trivial conditions \textit{1)} ad \textit{3)} here only to see how we apply Theorem \ref{T1}. Let us restrict Eq.~\eqref{GBP} to the domain $\Lambda$. This problem is convex and $\bm{\Gamma}_{+GBP}$ must be a local minima of this restricted problem as all the coordinates of $\Gamma_{+GBP}$ are strictly positive on $\Lambda$ so every local variation of $\bm{\Gamma}_{+GBP}$ is also a positive code. By the convexity of the unconstrained \eqref{GBP} on $\Lambda$ it must be also \textit{the global minima of the unconstrained \eqref{GBP} problem on the domain $\Lambda$}. Now as $\Lambda$ is group-full and fulfills the condition of Theorem \ref{T1}, we get by Theorem \ref{T1} that the minimizer $\bm{\Gamma}_{GBP}$ of the global \eqref{GBP} is unique moreover that $\supp{}{\bm{\Gamma}_{GBP}}\subset\Lambda$. Thus $\bm{\Gamma}_{GBP}$ is \textit{the global minima of the unconstrained \eqref{GBP} problem on the domain $\Lambda$}. By the uniqueness, it must be $\bm{\Gamma}_{+GBP}$. I.e. the minimizer of the positive coding problem happens to be the same as the minimizer of the global unconstrained \eqref{GBP}. So the Theorem \ref{T1} holds for $\bm{\Gamma}_{GBP}=\bm{\Gamma}_{+GBP}$.
\end{proof}

Note that if $\sup{}{\bm{\Gamma}_{+GBP}}$ is not group-full and $\mathcal{G}_{i}$ is a group where the $\ell_{2}$ was used but it is neighter in $\sup{}{\bm{\Gamma}_{+GBP}}$ nor in it's complementer, then we could try and break this group apart. We can do this by taking $\sup{}{\bm{\Gamma}_{+GBP}}\cap \mathcal{G}_{i}$ and $co\sup{}{\bm{\Gamma}_{+GBP}}\cap \mathcal{G}_{i}$ as groups. The problem is that it will change the regularizer of \eqref{PosGBP} thus the minimum can change yielding that we have to break up an other group and we change the problem further. However if we use only the $\ell_{1}$ norm and/or the elastic norm, then we have this kind of weak stability.

We can wonder why the groups corresponding to the elastic norm are not required to be full in the support. This is because we pay the price with the parameter $\lambda$, where the $\ell_{1}$ part of the elastic norm yields the stability.

Since positive coding is natural (from the point of view of biology), the stability should be investigated further to find a real stability theorem for this case.

\section{Additional Experimental Details}
\label{AppB}
\subsection{Neural Network Architectures}\label{ss:nets}
\subsubsection{Synthetic Experiment}
\begin{description}
\item[Linear Transformer]
\begin{enumerate}
    \item[]
    \item Attention Layer $100$ units
    \item Dense Linear Layer $75$ units
    \item Batch Normalization Layer
    \item Rectified Linear Unit (ReLU) Activation
    \item Dense Linear Layer $1$ unit
\end{enumerate}
\item[Dense Network]
\begin{enumerate}
    \item[]
    \item Dense Linear Layer $75$ units
    \item Batch Normalization Layer
    \item Rectified Linear Unit (ReLU) Activation
    \item Dense Linear Layer $1$ unit
\end{enumerate}
\item[Dense Deep Network]
\begin{enumerate}
    \item[] This architecture was designed to approximately match the parameter count of the Linear Transformer.
    \item Dense Linear Layer $96$ units
    \item Rectified Linear Unit (ReLU) Activation
    \item Dense Linear Layer $92$ units
    \item Rectified Linear Unit (ReLU) Activation
    \item Dense Linear Layer $89$ units
    \item Rectified Linear Unit (ReLU) Activation
    \item Dense Linear Layer $85$ units
    \item Rectified Linear Unit (ReLU) Activation
    \item Dense Linear Layer $82$ units
    \item Rectified Linear Unit (ReLU) Activation
    \item Dense Linear Layer $78$ units
    \item Rectified Linear Unit (ReLU) Activation
    \item Dense Linear Layer $75$ units
    \item Batch Normalization Layer
    \item Rectified Linear Unit (ReLU) Activation
    \item Dense Linear Layer $1$ unit
\end{enumerate}
\end{description}

\subsubsection{MNIST Experiment}
\begin{description}
\item[Linear Transformer]
\begin{enumerate}
    \item[]
    \item Attention Layer $784$ units
    \item Batch Normalization Layer
    \item Dense Linear Layer $32$ units
    \item Rectified Linear Unit (ReLU) Activation
    \item Dense Linear Layer $10$ units
    \item Softmax Activation
\end{enumerate}
\item[Dense Network]
\begin{enumerate}
    \item[]
    \item Dense Linear Layer $32$ units
    \item Rectified Linear Unit (ReLU) Activation
    \item Dense Linear Layer $10$ units
    \item Softmax Activation
\end{enumerate}
\item[Dense Deep Network]
\begin{enumerate}
    \item[]
    \item Dense Linear Layer $676$ units
    \item Rectified Linear Unit (ReLU) Activation
    \item Dense Linear Layer $569$ units
    \item Rectified Linear Unit (ReLU) Activation
    \item Dense Linear Layer $461$ units
    \item Rectified Linear Unit (ReLU) Activation
    \item Dense Linear Layer $354$ units
    \item Rectified Linear Unit (ReLU) Activation
    \item Dense Linear Layer $246$ units
    \item Rectified Linear Unit (ReLU) Activation
    \item Dense Linear Layer $139$ units
    \item Rectified Linear Unit (ReLU) Activation
    \item Dense Linear Layer $32$ units
    \item Rectified Linear Unit (ReLU) Activation
    \item Dense Linear Layer $10$ unit
    \item Softmax Activation
\end{enumerate}
\end{description}

\newpage
\subsection{Additional Synthetic Data Statistics}\label{ss:Synth_attack}
\begin{table}[!h]
    \begin{tabularx}{\columnwidth}{Xccc}
        \toprule
        & \tableheadline{Inactive} & \tableheadline{Mean} & \tableheadline{Found} \\
        Method & \tableheadline{Groups} & \tableheadline{Grp. Acc.} & \tableheadline{Grp. Combs.} \\
        \midrule
        BP & $47.6\%$ & $58.3\%$ & $\phantom{0}0.0\%$\\
        GBP \& PGBP & $87.8\%$ & $98.5\%$ & $48.8\%$\\
        Transformer & $85.6\%$ & $96.5\%$ & $\phantom{0}9.5\%$ \\
        Dense Shallow & $83.2\%$ & $93.8\%$ & $\phantom{0}1.4\%$ \\
        Dense Deep & $74.3\%$ & $84.5\%$ & $\phantom{0}0.0\%$ \\
        \bottomrule
    \end{tabularx}
    \caption[Statistics for the attack-free case of the Synthetic dataset.]{Statistics for the attack-free case of the Synthetic dataset. $\mathcal{O}(\bm{X})\geq 0.1$ margin. Grp., Acc. and Combs. stand for Group, Accuracy and Combinations, respectively.
    All scores are much higher for Group Basis Pursuit (GBP) compared to Basis Pursuit (BP), since GBP has access to the true group structure, while BP picks individual units instead.
    Each input was generated by a specific group combination using 8 groups drawn randomly (see main text). The best result in finding the \emph{exact} group combinations is close to 50\%, implying that we are outside of the scope of our theorems. Accuracy, i.e., mean of the found active (True Positive) and inactive (True Negative) groups relative to all groups (Positive + Negative) is close to $100\%$ for GBP and Pooled Group Basis Pursuit (PGBP): they achieve the same score as only their downstream classifiers differ.
    Performance of the feedforward networks are worse since they are trained to approximate PGBP instead of the ground truth.
    Transformer outperforms Dense networks. Dense Shallow network gives better results than the Dense Deep network; the latter is overtrained.}
    \label{tab:sparsity}
\end{table}
\newpage
\onecolumn
\subsection{Additional Experimental Results}\label{ss:Synth_attack2}


\begin{figure}[!h]
\centering     
\includegraphics[width=70mm]{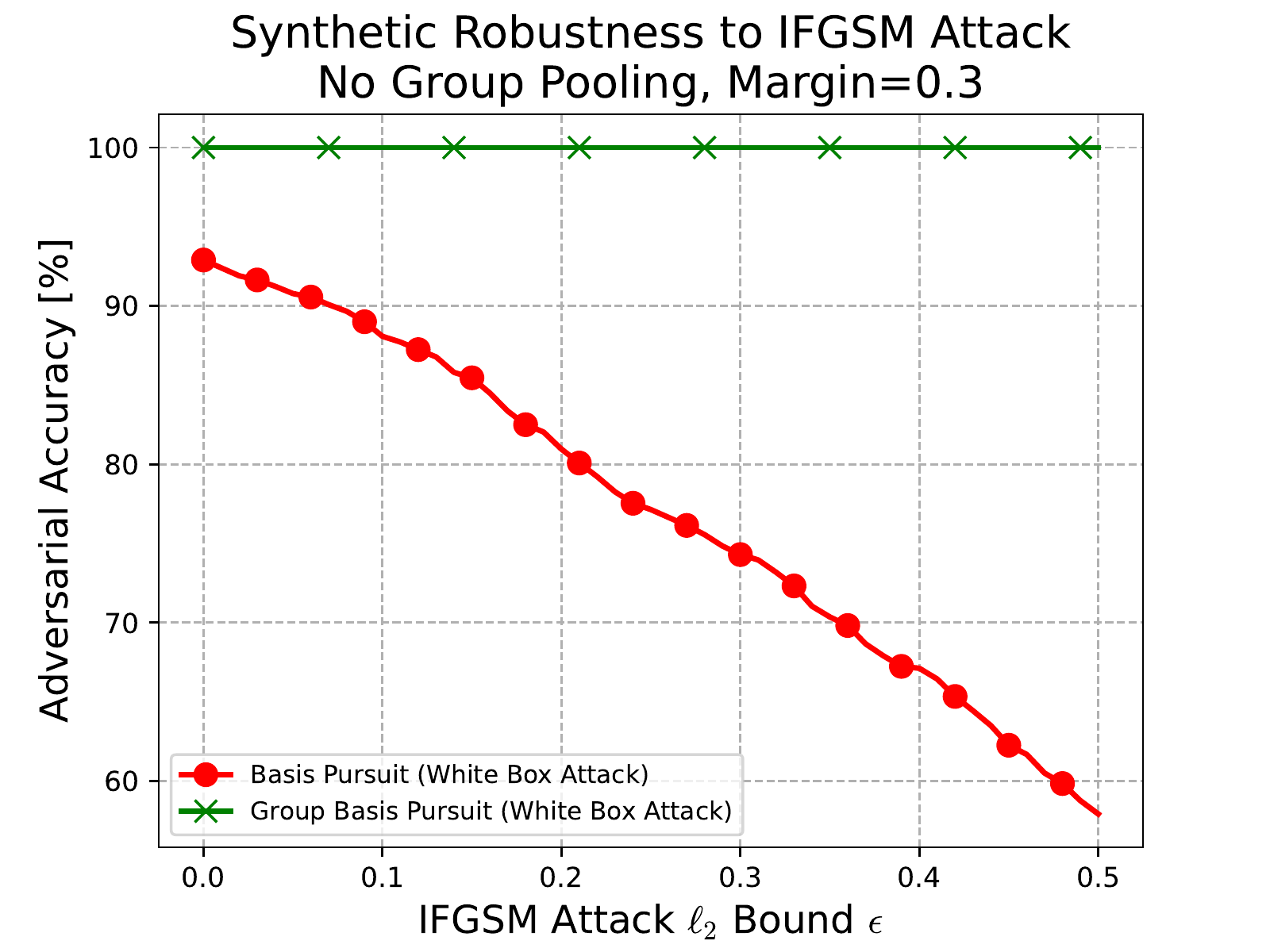}
\textbf{(a)}\hfill
\includegraphics[width=70mm]{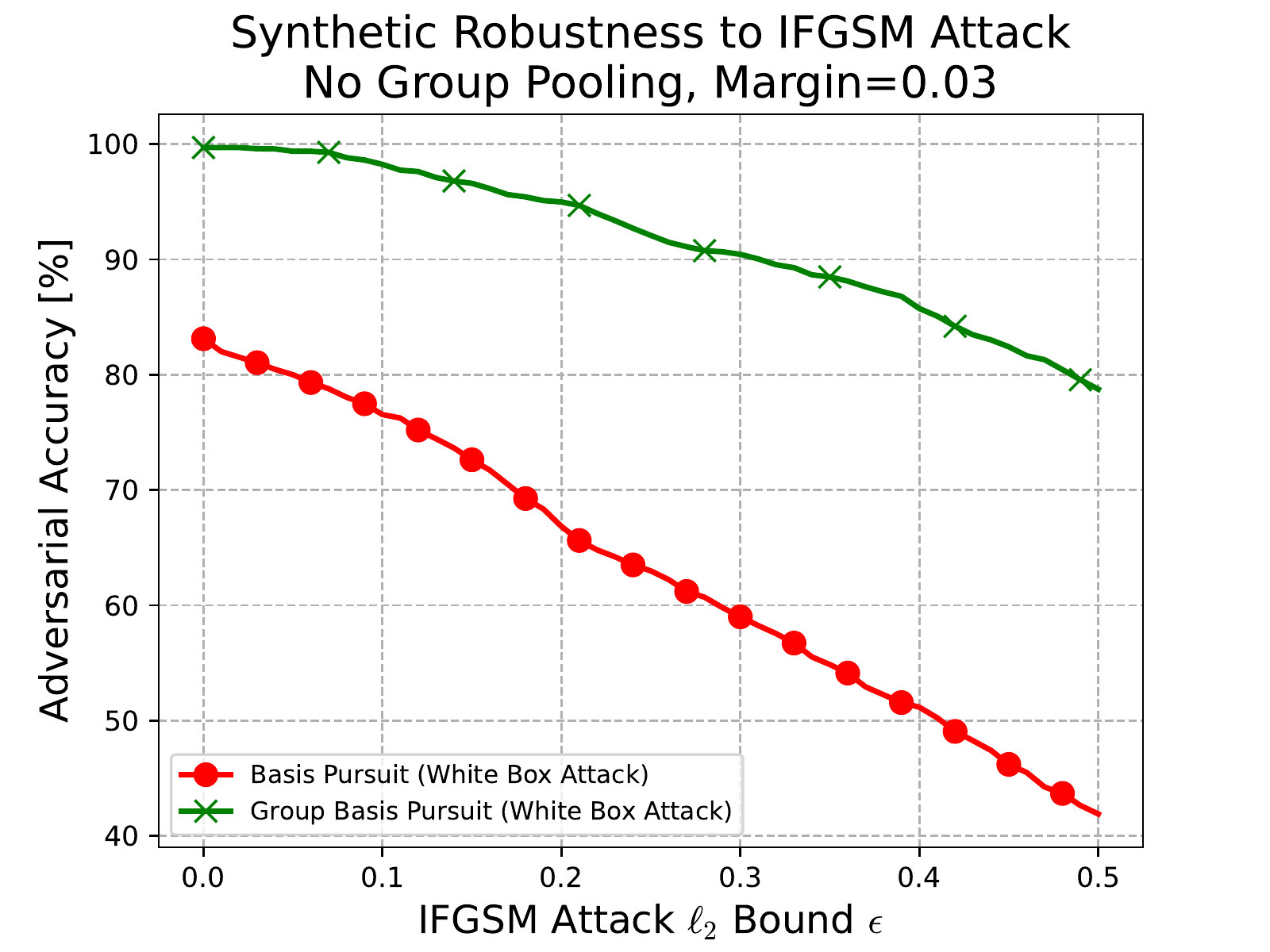}
\textbf{(b)}\\
\includegraphics[width=70mm]{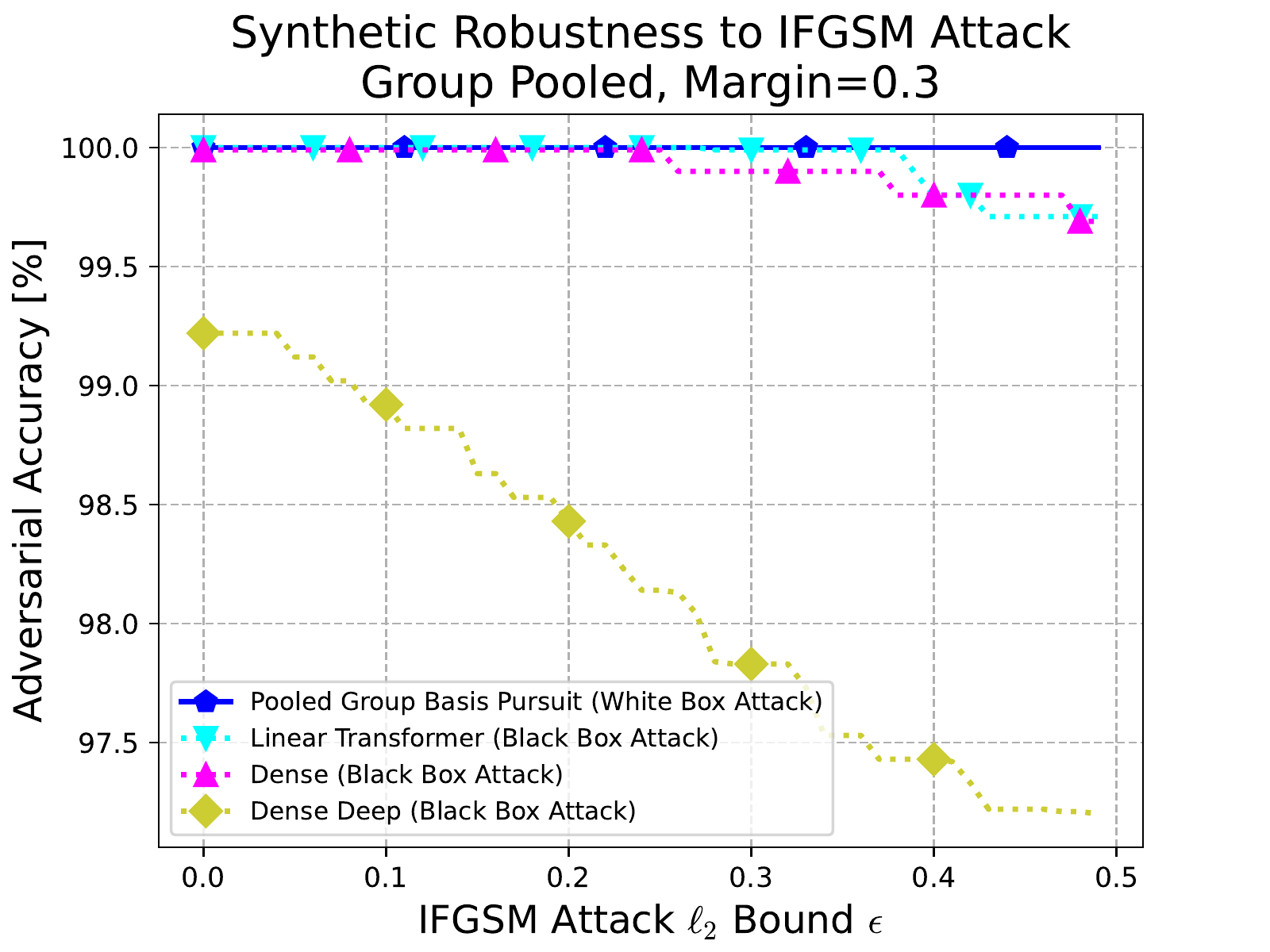}
\textbf{(c)}\hfill
\includegraphics[width=70mm]{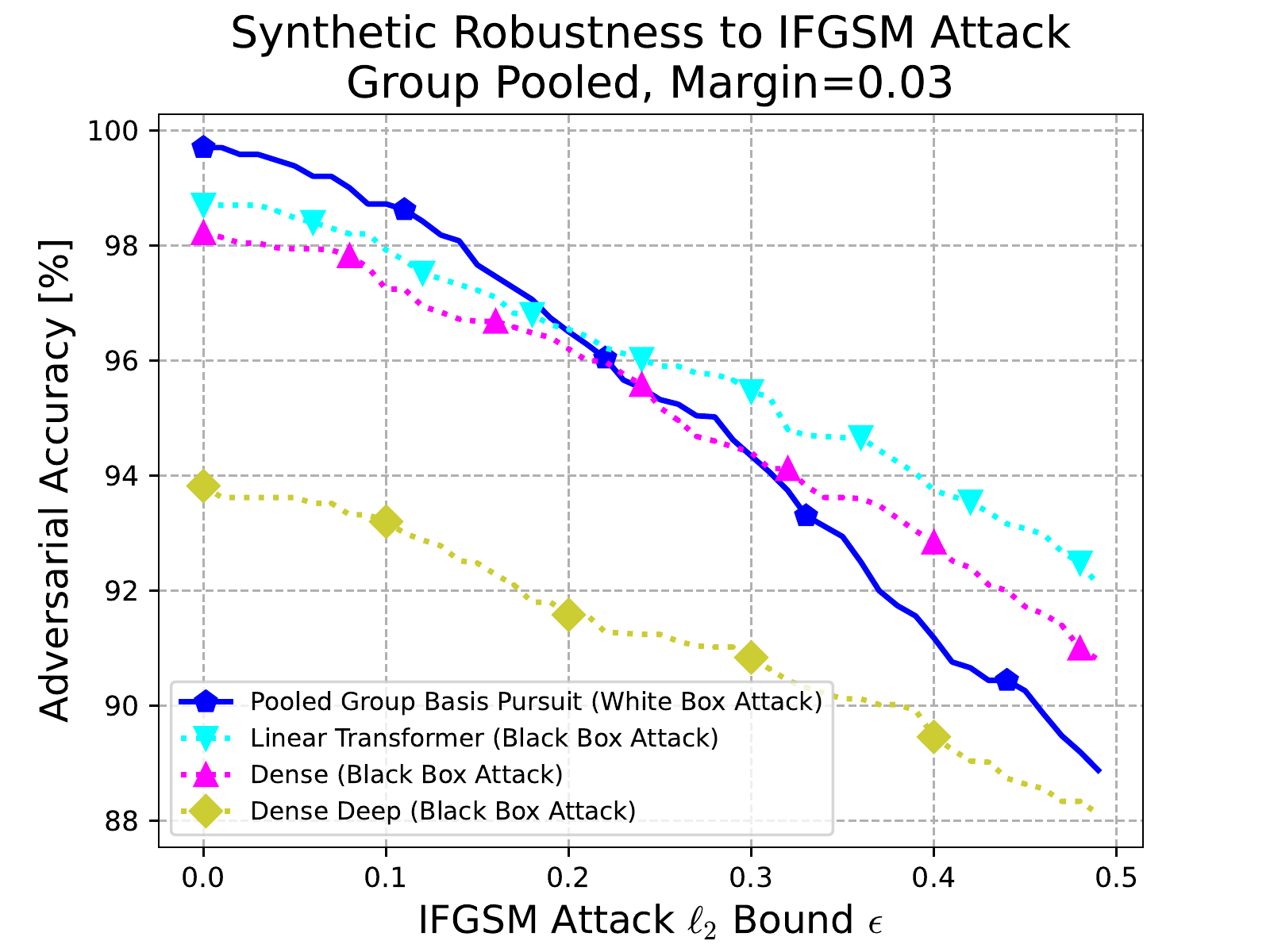}
\textbf{(d)}
\caption{Experimental results for adversarial robustness against Iterative Fast Gradient Sign Method (IFGSM) attack on the Synthetic dataset. 
White Box attack: Pooled Group Basis Pursuit (PGBP, red). Black Box attacks: Linear Transformer (LT, green), Dense (blue), Deep network (turquoise).
\textbf{(a):} $\mathcal{O}(\bm{X})\geq 0.3$ margin experiment: PGBP (red) achieves perfect scores for all $\epsilon$ values. Performance of LT (green) and the Dense (blue) networks trained on PGBP groups indices starts to deteriorate for large $\epsilon$ values although it is under Black Box attack. Deep network having parameter count similar to LT is overfitting.
\textbf{(b):} $\mathcal{O}(\bm{X})\geq 0.1$ margin experiment: PGBP (red) achieves perfect scores for smaller $\epsilon$ values, while it breaks down faster than LT (green) and the Dense (blue) networks for large $\epsilon$'s. Deep network is overfitting.
\textbf{(c):} $\mathcal{O}(\bm{X})\geq 0.03$ margin experiment: PGBP (red) can't achieve perfect scores for $\epsilon$ values, including the no attack case and breaks down faster than LT and the Dense network for large $\epsilon$'s. Deep network is overfitting.}
\label{fig:net-exps2}
\end{figure}

\end{document}